\makeatletter
\@ifundefined{ifarx}{%
	\newif\ifarx
	\arxfalse
}{}
\makeatother

\documentclass{article}
\usepackage{fullpage}
\usepackage{amsmath,amsfonts,amsthm}
\newtheorem{lemma}{Lemma}

\newtheorem{informal}{Informal Theorem}
\usepackage{thmtools}
\usepackage{thm-restate}
\newtheorem{theorem}{Theorem}
\usepackage{hyperref}

\usepackage{cleveref}

\newtheorem{definition}{Definition}

\author{
Yuval Dagan\\
MIT, EECS\\
\texttt{dagan@csail.mit.edu}
\and 
Constantinos Daskalakis \\
MIT, EECS\\
\texttt{costis@csail.mit.edu}
\and 
Anthimos Vardis Kandiros\\
MIT, EECS\\
\texttt{kandiros@mit.edu}
}

\usepackage[natbib=true]{biblatex} %Imports biblatex package
\usepackage[utf8]{inputenc}
\usepackage{thmtools} 
\usepackage{thm-restate}
\usepackage{xcolor}
\DeclareMathOperator{\Var}{Var}
\DeclareMathOperator{\Cov}{Cov}
\DeclareMathOperator*{\argmax}{argmax}
\DeclareMathOperator*{\E}{\mathbb{E}}
\renewcommand{\v}{\overline}
\newcommand{\f}{\tilde}
\newcommand{\x}{\v x}
\renewcommand{\l}{\left}
\renewcommand{\r}{\right}
\newcommand{\lp}{\left}
\newcommand{\rp}{\right}
\newcommand{\R}{\mathbf{R}}

\newif\ifcomments
\commentstrue

\DeclareMathOperator{\KL}{KL}

\newtheorem{claim}{Claim}
\newtheorem{corollary}{Corollary}
\newtheorem{proposition}{Proposition}
\newtheorem{remark}{Remark}

\DeclareMathOperator{\tr}{trace}
\DeclareMathOperator{\poly}{poly}

\title{EM's Convergence in Gaussian Latent Tree Models}

\begin{document}

\maketitle
\begin{abstract}
    We study the optimization landscape of the log-likelihood function and the convergence of the Expectation-Maximization (EM) algorithm in latent Gaussian tree models, i.e.~tree-structured Gaussian graphical models whose leaf nodes are observable and non-leaf nodes are unobservable. We show that the unique  non-trivial stationary point of the population log-likelihood is its global maximum, and establish that the expectation-maximization algorithm is guaranteed to converge to it in the single latent variable case. Our results for the landscape of the log-likelihood function in general latent tree models provide support for the extensive practical use of maximum likelihood based-methods in this setting. Our results for the EM algorithm extend an emerging line of work on obtaining global convergence guarantees for this celebrated algorithm. We show our results for the non-trivial stationary points of the log-likelihood by arguing that a certain system of polynomial equations obtained from the EM updates has a unique non-trivial solution. The global convergence of the EM algorithm follows by arguing that all trivial fixed points are higher-order saddle points. 
\end{abstract}

\section{Introduction} \label{sec:intro}

Estimating latent variable models is a widely-studied task in Statistics and Machine Learning. It is also a  daunting one, computationally and statistically. Even if the underlying, fully observable distribution is an exponential family  and therefore has  a concave log-likelihood function, marginalizing out the latent variables makes the log-likelihood non-concave, in most cases. In the same exponential-family setting,  under mild conditions, the population (i.e.~infinite sample) log-likelihood  of the complete model has a unique maximum at  the true model parameters, yet even in this setting very little is understood about  the landscape of the partially observable model's log-likelihood or its stationary points.

A  widely-applicable method for estimating latent variable models is the Expectation-Maximization (EM) algorithm of~\cite{DempsterLR77}. Given a parametric family of distributions $\{p_{\boldsymbol{\theta}}(\boldsymbol{X},\boldsymbol{Y})\}_{\boldsymbol{\theta} \in \Theta}$, where variables~$\boldsymbol{X}$ are observable and variables~$\boldsymbol{Y}$ are unobservable, and given independent observations $\boldsymbol{x}_1, \boldsymbol{x}_2,\ldots$ from some model in this family, the EM algorithm starts with some initialization $\boldsymbol{\theta}^{(0)}$  and iteratively performs a sequence of interleaved ``E-steps'' and ``M-steps,'' a consecutive pair of which are called an ``EM update.'' Specifically, for all $t \ge 0$, the algorithm updates the current vector of parameters $\boldsymbol{\theta}^{(t)}$ by performing the following: 
%\yuval{Should we use the same notation as in the proof?}
\begin{itemize}
    \item {\bf (E-step)}  for each sample $i$, compute a posterior  belief about the values of the unobservable variables  by setting, for all $\boldsymbol{y}$,  $Q_i^{(t)}(\boldsymbol{y})=p_{\boldsymbol{\theta}^{(t)}}(\boldsymbol{Y}=\boldsymbol{y}|\boldsymbol{X}=\boldsymbol{x}_i)$;
    
    \item {\bf (M-step)} update the parameters to $\boldsymbol{\theta}^{(t+1)} \in \arg \max_{\boldsymbol{\theta}}\sum_i \int_{\boldsymbol y} Q_i^{(t)}(\boldsymbol{y}) \log { p_{\boldsymbol{\theta}}(\boldsymbol{X}=\boldsymbol{x}_i,\boldsymbol{Y}=\boldsymbol{y}) \over Q_i^{(t)}(\boldsymbol{y})} d\boldsymbol{y}.$ 
    %\yuval{Do we want to have an integral instead?}
\end{itemize}
%\yuval{Is it too technical too early?}
Notice that, by Jensen's inequality, the function maximized in the M-step of the algorithm lower bounds the log-likelihood of the samples and, by the choice made in the E-step of the algorithm, this lower bound equals the log-likelihood of the samples at $\boldsymbol{\theta}=\boldsymbol{\theta}^{(t)}$. Thus, whenever the EM update results in $\boldsymbol{\theta}^{(t+1)}\neq \boldsymbol{\theta}^{(t)}$, the likelihood of the samples increases. Moreover, when the class of models $\{p_{\boldsymbol\theta}\}_{\boldsymbol{\theta}\in \Theta}$ is an exponential family, the M-step becomes  a concave maximization problem, making the algorithm quite attractive in this setting. 

Despite its wide use and study, with north of 66k citations, relatively little is known about EM's behavior. Conditions have been identified under which the EM iterates converge to or cluster at stationary points of the log-likelihood --- see e.g.~\cite{Wu83,Tseng04,ChretienH08}, or exhibit \emph{local} convergence to the maximum of the likelihood --- see e.g.~\cite{RednerW84,BalakrishnanWW14,zhao2020statistical,kwon2020algorithm}. Conditions under which EM exhibits \emph{global} convergence  to the maximum of the likelihood are rare~\citep{Wu83} with a surge of recent results inching towards establishing global convergence guarantees in more and more settings --- from balanced mixtures of two Gaussians~\citep{xu2016global,daskalakis2017ten} to balanced mixtures of two truncated Gaussians~\citep{nagarajan2020analysis}, balanced mixtures of two Laplace distributions~\citep{barazandeh2018behavior}, unbalanced mixtures of two Gaussians~\citep{xu2018benefits}, binary variable naive Bayes models~\citep{daskalakis2018bootstrapping}, and mixtures of linear regression models~\citep{kwon2019global,klusowski2019estimating,kwon2020converges,kwon2021minimax} --- and towards understanding the role of overparametrization in EM's global convergence~\citep{xu2018benefits,dwivedi2020singularity}.

%\yuval{It seems like a lot has done. What do we innovate over the work that has already done before? Is our statement stronger (in some sense) than mixtures of Gaussian?}

To the best of our knowledge, recent works on the global convergence of EM are for single-latent-variable models. Extending this recent line of work, our paper studies the convergence of EM in  Latent Gaussian Tree Models (LGTMs), i.e.~tree-structured Gaussian Graphical Models whose leaf variables are observable and non-leaf variables are unobservable. Latent tree models in general, and LGTMs in particular have found wide use in scientific and applied domains due to their combined expressiveness and tractability of inference; see e.g.~\cite{mourad2013survey,zwiernik2018latent} for  recent surveys. Some notable applications of LGTMs are in phylogenetics, where they have been used to model the evolution of continuous traits~\citep{felsenstein1973maximum,hiscott2016efficient,truell2021maximum}, in network tomography, to model network delays~\citep{castro2004network,eriksson2010toward,bhamidi2010network}, and in  linguistics, for modelling the evolution of languages using acoustic data~\citep{ringe2002indo,shiers2017gaussian}. 

Given observations from a latent tree model, a long line of research has studied whether the structure of the model can be recovered and, if the structure is known, whether the parameters of the model can be recovered. Most techniques with provable guarantees are based on defining and estimating tree metrics from the samples; see e.g.~\cite{erdos1999few,felsenstein2004inferring,daskalakis2006optimal,roch2006short,roch2010toward,roch2017phase} and their references. On the practical front, however, some of the most popular packages are based on maximum likelihood estimation; see e.g.~\cite{yang1997paml,stamatakis2006raxml}. Even when the latent structure is known, however, the landscape of the likelihood function is not well understood~\citep{felsenstein1973maximum,truell2021maximum}. This is true even in the population limit, i.e.~when infinitely many samples are available, even when the tree is trivial, i.e.~there is a single latent node, and even when the latent tree model is a LGTM. For this paradigmatic and widely used family of models we study the following question:

\medskip \noindent \begin{minipage}{\textwidth} {\em {\bf Main Question:}  %\yuval{should we define the model formally?} 
Given a LGTM model $p_{\boldsymbol{\theta^*}}(\boldsymbol{X},\boldsymbol{Y})$ on a tree $T$ whose leaves $\boldsymbol{X}$ are observable and internal nodes $\boldsymbol{Y}$ are unobservable, can we characterize the  stationary points of the population log likelihood $\ell_{\boldsymbol{\theta}^*}(\boldsymbol{\theta})\equiv\mathbb{E}_{\boldsymbol{x} \sim p_{\boldsymbol{\theta^*}}(\boldsymbol{X})}[\log p_{\boldsymbol{\theta}}(\boldsymbol{X}=\boldsymbol{x})]$? Does it have spurious stationary points $\boldsymbol{\theta} \neq \boldsymbol{\theta^*}$ and under what conditions does  EM  converge to $\boldsymbol{\theta^*}$?
}
\end{minipage}

\medskip We study the afore-described questions in the setting where all the nodes of $T$ are single-dimensional Gaussian variables and assume, without loss of generality, that they have zero mean and that the setting is ferromagnetic, i.e.~assume that for every pair of variables their correlation lies in $(0,1)$. Our first main result is the following (the formal version is Theorem~\ref{t:stationary}, combined with Lemma~\ref{l:em_equiv}):

%Moreover, we parametrize the LGTM in terms of its precision matrix, whose structure reflects the structure of $T$. In this case, the family of models $\{p_{\boldsymbol{\theta}}\}_{\boldsymbol{\theta} \in \Theta}$ that we are optimizing over is an exponential family, and this makes the M-step of EM efficient. 

%\yuval{Perhaps describe in more details}
\begin{informal} \label{inf thm:convergence trees}
In the setting of our main question, and that of the preceding paragraph, suppose that $\boldsymbol{\theta}$ is a stationary point of the population log-likelihood $\ell_{\boldsymbol{\theta}^*}(\cdot)$ that is non-trivial, i.e.~in model $p_{\boldsymbol{\theta}}(\boldsymbol{X},\boldsymbol{Y})$ there is no edge of the tree whose endpoints have correlation in $\{0,1\}$. Then $\boldsymbol{\theta} = \boldsymbol{\theta}^*$. 
%or there exists an edge $e$ in $T$ such that the correlation of the endpoints of $e$ under $\boldsymbol{\theta}$ is $0$.
\end{informal}
\noindent Our result guarantees that, if gradient-descent, EM, or similar method converges to a stationary point of the population log-likelihood that is non-trivial, then this point indexes the true model. While there are other criteria that can be used, our result implies that the stationarity of the log-likelihood can be used as an alternative, post-hoc criterion to argue that the correct model has been identified. In particular, our result substantiates the extensive use of maximum-likelihood-based methods in practice~\citep{yang1997paml,stamatakis2006raxml}, and the experimental evidence that EM succeeds with high probability in this setting~\citep{wang2006severity}. Next, we study whether we can guarantee that EM converges to a non-trivial stationary point in our setting. We show this for the case where there is a single latent node in the model (the formal statement is Theorem~\ref{thm:one-formal}):

\begin{informal} \label{inf thm:convergence single latent}
In the setting Informal Theorem~\ref{inf thm:convergence trees}, suppose additionally that there is a single latent node. In this case, EM is guaranteed to converge to $\boldsymbol{\theta}^*$. If $n$ is the number of leaves and $\epsilon$ is the desired accuracy for all the parameters (in absolute value), then the sample complexity is $O(\mathrm{poly}(n) / \epsilon^2)$ and the number of iterations is $O(\mathrm{poly}(n)\log(1/\epsilon))$.
%or there exists an edge $e$ in $T$ such that the correlation of the endpoints of $e$ under $\boldsymbol{\theta}$ is $0$.
\end{informal}

\paragraph{Proof Ideas.}
%To prove Informal Theorem 1, we start from the case of one latent node. 
Starting with Informal Theorem~\ref{inf thm:convergence single latent}, it is known that EM converges to some $\boldsymbol\theta$ that is a fixed point of the EM update. Hence, the proof follows by making the following  arguments: (1) ${\boldsymbol\theta}^*$ is the only non-trivial fixed point of the EM update; if we parametrize the model via correlations on its edges this the same as saying that ${\boldsymbol\theta}^*$ is the only fixed point of the EM update in the  interior of the parameter space $\Theta$; and (2) While there are fixed points at the boundary of $\Theta$,  EM does not converge to any of those. To show (1), we analyze the EM update and show that its fixed points are solutions to a system of degree-2 polynomial equations. By using a simple special form of the Jacobian Conjecture we argue that these have a unique non-trivial solution. To show (2), we analyze the behavior of the EM around the fixed points at the boundary, by computing the derivatives of the update rule around these points. This requires careful analysis since these fixed points are higher-order saddle points. 
For the case of general trees (Informal Theorem~\ref{inf thm:convergence trees}), we show that the fixed points are solutions of higher-degree polynomial equations. 
We establish a novel reduction of these algebraic equations to the second-degree polynomial system from the single-latent-model case, and use this reduction to certify that there is a unique fixed point at the interior of the parameter space, which is ${\boldsymbol\theta}^*$.

\section{Tree models with one latent node}\label{sec:onelatent}
We start by elaborating on the simpler setting with one latent node. For simplicity, we start by anayzing the \emph{population EM}, which amounts to running the EM algorithm on the whole population (rather than a finite sample). So simplify even more, we analyze asymptotic convergence, namely, in the limit as the number of iterations goes to infinity. Later, in Section~\ref{sec:results-finite}, we describe the finite-sample and finite-iterate result.

\subsection{Definitions and properties of the model}\label{s:one_latent_defs}
We consider the family of multivariate Gaussian distributions over zero-mean variables, with latent node $y$ and observed nodes $x_1,\dots,x_n$, that have the property
\[
\Pr[x_1,\dots,x_n,y]
= \Pr[y] \prod_{i=1}^n \Pr[x_i\mid y]\enspace.
\]
Each such distribution is uniquely defined by the following parameters:
\[
\sigma_y^2 := \Var(y); \ 
\sigma_{x_i}^2 := \Var[x_i]; \
\rho_i := \frac{\Cov(x_i,y)}{\sqrt{\Var(x_i)\Var(y)}},\ \text{for } i=1,\dots,n \enspace,
\]
We note that generally $\rho_i \in [-1,1]$, yet, we analyze the ferromagnetic setting where $\rho_i \in [0,1]$. Denote by $\mathcal{P}$ the set of all such distributions with $\rho_i \in [0,1]$ and $\sigma_y,\sigma_{x_i} > 0$.
We use the following convenient properties of distributions in $\mathcal{P}$:
\begin{claim}\label{cla:lambda}
For any $i\ne j$, $\Cov(x_i,x_j) = \sigma_{x_i}\sigma_{x_j}\rho_i\rho_j$. Further,  $\E[x_i \mid y] = \frac{\sigma_{x_i}\rho_{i}}{\sigma_y}y$ and
\begin{equation}\label{eq:conditional}
\E[y\mid x_1\cdots x_n] = \sigma_y \sum_{i=1}^n \lambda_i \frac{x_i}{\sigma_i}, \quad \text{where }
\lambda_i = \frac{\rho_i/(1-\rho_i^2)}{1 + \sum_{j=1}^n \rho_j^2/(1-\rho_j^2)}\enspace.
\end{equation}
\end{claim}
% This implies that $\mathcal{P}$ can be written formally as
% \[
% \mathcal{P} = \l\{ \mu \sim N(0, \Sigma) \colon \Sigma = \begin{pmatrix}
% \sigma_y & \sigma_1 & \cdots & \sigma_n
% \end{pmatrix}
% \begin{pmatrix}
% 1 & \rho_1 & \rho_2 & \cdots &  \rho_{n-1} & \rho_n\\
% \rho_1 & 1 & 0 & \cdots & 0 & 0 \\
% \cdots &  &  & \cdots &  & \\
% \rho_n & 0 & 0& \cdots & 0 & 1
% \end{pmatrix}
% \begin{pmatrix}
% \sigma_y \\
% \sigma_1 \\
% \cdots \\
% \sigma_n
% \end{pmatrix},\
% \sigma_i, \rho_i \in [0,1]
% \r\}
% \]
Given any distribution $\mu \in \mathcal{P}$, denote its marginals and conditionals as $\mu_{x_1\cdots x_n}$, $\mu_{y\mid x_1 \cdots x_n}$ etc. Lastly, denote $\v\rho = (\rho_1,\dots,\rho_n)$ and $\x = (x_1,\dots,x_n)$.
\subsection{The expectation-maximization algorithm - over the population}\label{s:one_latent_em}

We analyze the EM algorithm. We start by analyzing the \emph{population EM}: this assumes that each iteration of the EM is executed over the \emph{whole population}, rather than over a finite sample. This greatly simplifies the analysis (see Section~\ref{sec:results-finite} for the results on finite sample). The population EM can be described as follows: we set $\mu^0 \in \mathcal{P}$ arbitrarily. Then, at any $t > 0$, define
\[
\mu^{t+1} = \argmax_{\mu \in \mathcal P} \E_{x_1\cdots x_n \sim \mu^*_{x_1\cdots x_n}} \E_{y \sim, \mu^t_{y \mid x_1 \cdots x_n}}[\log \Pr_\mu(x_1,\dots,x_n, y)],
\]
where $\Pr_\mu$ denotes the density with respect to $\mu$. Denote by $\sigma_{\cdot}^t, \rho_i^t$ the parameters corresponding to $\mu^t$ and by $\lambda_i^t$ the coefficients from \eqref{eq:conditional}. Similarly, $\sigma_{\cdot}^*,\rho_i^*$ and $\lambda_i^*$ correspond to $\mu^*$. We would like to understand how these parameters update in each iteration of the EM algorithm. For this purpose, we have the following lemma:
\begin{lemma}\label{lem:covariance-conserve}
Let $\mu^{t,*}$ denote the joint distribution over $x_1\cdots x_n,y$ such that
\[
\Pr_{\mu^{t,*}}[x_1,\cdots, x_n,y] 
= \Pr_{\mu^*}[x_1,\cdots, x_n]\Pr_{\mu^t}[y\mid x_1,\dots, x_n].
\]
Then, for any $i$, we have that
\[
\E_{\mu^{t+1}}[x_i y] = \E_{\mu^{t,*}}[x_i y], \ \Var_{\mu^{t,*}}[x_i] = \Var_{\mu^{t+1}}[x_i],\ 
\Var_{\mu^{t,*}}[y] = \Var_{\mu^{t+1}}[y]\enspace.
\]
\end{lemma}
\begin{proof}
Notice that $\mu^{t+1}$ is the MLE over $\mathcal P$, given samples drawn from $\mu^{t,*}$. Hence, for any $\mu \in \mathcal P$,
\begin{align}
\E_{\x,y\sim \mu^{t,*}}\log \Pr_{\mu}(\x,y)
&= \E_{\x,y\sim \mu^{t,*}}\log\Pr_{\mu}(y) \prod_i \Pr_{\mu}(x_i\mid y) \notag\\
&= \E_{\x,y\sim \mu^{t,*}}\log\Pr_{\mu}(y) + \sum_i \E_{\x,y\sim \mu^{t,*}}\log\Pr_{\mu}(x_i\mid y). \label{eq:mu-decomp}
\end{align}
Recall that each $\mu\in \mathcal{P}$ can be decomposed as $\Pr_{\mu}[x,y] = \Pr_{\mu_y}[y]\prod_i \Pr_{\mu_{x_i\mid y}}[x_i\mid y]$, and each term in this decomposition can be chosen to maximize its corresponding summand from \eqref{eq:mu-decomp}. By Gibbs inequality\footnote{Gibb's inequality states that for any distribution $P$, $\argmax_{Q} \E_{x\sim P}\log \Pr_Q[x] = P$, where $Q$ is taken over all the probability measures. This inequality can be similarly applied on conditional distributions.}, the maximizing choice is obtained by selecting $\mu_y\sim \mu^{t,*}_y$ and $\mu_{x_i\mid y} \sim \mu^{t,*}_{x_i\mid y}$. This choice satisfies $\mu_{x_i y} \sim \mu_{y}\mu_{x_i\mid y} \sim \mu^{t,*}_y \mu^{t,*}_{x_i\mid y} \sim \mu^{t,*}_{x_i,y}$. Hence, the pairwise marginals between $x_i$ and $y$ are conserved, which concludes the proof.
\end{proof}

As a corollary, we obtain the following update rules for the parameters $\sigma^t_y,\sigma^t_{i}$ and $\rho^t_i$, using the analogous parameters of $\mu^{t,*}$, that are calculated using formulas for the conditional Gaussian measure (proof appears in Section~\ref{sec:app-proofs}).
\begin{lemma}\label{lem:update}
For any $i\ne j$, denote by $\Delta^t_{ij} = \rho^*_i\rho^*_j - \rho_i\rho_j$. For any $t > 0$ and any $i$, $\sigma^t_{i} = \sigma^*_{i}$, 
\[
\E_{\mu^{t+1}}[x_i y] =
\sigma^t_{i}\sigma^t_y \l(\lambda_i^t + \sum_{j\ne i} \rho^*_i\rho^*_j \lambda_j^t \r)
= \sigma^t_{i}\sigma^t_y \l(\rho_i^t + \sum_{j\ne i} \Delta_{ij}^t \lambda_j^t \r)
\enspace,
\]
\begin{equation*} 
\E_{\mu^{t+1}}[y^2]
= (\sigma_y^{t+1})^2
= (\sigma^t_y)^2\lp(\sum_{i=1}^n (\lambda^t_i)^2 + \sum_{i\ne j \in \{1,\dots,n\}}
\lambda^t_i \lambda^t_j \rho^*_i \rho^*_j\rp)
= (\sigma_y^t)^2\l( 
1 + \sum_{j\ne k} \Delta_{ij}^t \lambda_j^t \lambda_k^t
\r)
\end{equation*}
\begin{equation}\label{eq:update-rho}
\rho_i^{t+1} = 
\frac{\lambda_i^t + \sum_{j\ne i} \lambda_j^t \rho^*_i\rho^*_j}{\sqrt{\sum_{i=1}^n (\lambda^t_i)^2 + \sum_{i\ne j \in \{1,\dots,n\}}
\lambda^t_i \lambda^t_j \rho^*_i \rho^*_j}}
= \frac{\rho_i^t + \sum_{j\ne i} \Delta_{ij} \lambda_j^t}{\sqrt{1 + \sum_{j\ne k} \Delta_{ij} \lambda_j^t \lambda_k^t}}
\enspace.
\end{equation}
\end{lemma}

\subsection{Asymptotic Convergence of the population EM}

We would like to argue that the iterates of the EM converge to $\mu^*$, which is characterized by the parameters $\sigma_y^*,\sigma_{x_i}^*$ and $\rho_i^*$ for $i=1,\dots,n$.
Yet, notice that one only observes samples from the marginal $\mu^*_{x_1\cdots x_n}$, which is a function of only $\sigma_{x_i}^*$ and $\rho_i^*$ but not of $\sigma_y^*$. Hence, we cannot expect to learn $\sigma_y^*$. With regard to the other parameters, we prove:
\begin{theorem}\label{thm:one-formal}
Assume that the correlations in the underlying distribution $\mu^*$ satisfy $\rho_i^* \in (0,1)$ for all $i$ and that the first iterate of the population EM satisfies $\rho^0_i \in (0,1)$ for all $i$. Then, the iterates $\rho^t_i$ and $\sigma^t_i$ of the population EM converge to the parameters of the underlying distribution:
\[
\lim_{t\to \infty} \rho^t_i = \rho^*_i, \quad \text{and }  \sigma^t_i = \sigma^*_i \text{ for all } i = 1,\dots, n \text{ and } t \ge 1 \enspace.
\]
\end{theorem}
We will stay under the assumptions that $\rho^*_i,\rho^0_i \in (0,1)$ for all $i$ throughout the proof.
Recall that Lemma~\ref{lem:update} argues that for any $t \ge 1$, $\sigma_{x_i}^t = \sigma_{x_i}^*$. Hence, it remains to argue about the convergence of $\rho_i^t$ to $\rho_i^*$.
%Since the update rule of $\rho_i^t$, as computed in \eqref{eq:update-rho}, does not depend neither on $\sigma_{x_i}$ nor on $\sigma_y$, we will make the convenient assumption that $\sigma^*_{i} = \sigma^*_y = 1$. Similarly, when analyzing iteration $t$ of the EM, we will often assume that $\sigma_y^t = 1$.
We use the following definition:
\begin{definition}
A point $\v\rho = (\rho_1,\dots,\rho_n)$ is a stationary point of the EM if $\v\rho^t = \v\rho$ implies that $\v\rho^{t+1}=\v\rho$. 
\end{definition}
Denote the set of all stationary points by $\mathcal{S}$. The following is proved in Section~\ref{sec:app-proofs}:
\begin{lemma} \label{lem:convergence-to-stat}
The iterates of the population EM converge to some stationary point $\v\rho \in \mathcal{S}$. Further, for any $\v\rho \in \mathcal{S}$, if $\v\rho^t = \v\rho$ then for any $i$, 
$\mu^{t+1} = \mu^t$
\end{lemma}

Since the EM converges only to a stationary point, it is useful to characterize the set of stationary points, as stated below:
\begin{lemma}\label{lem:stationaries}
The set $\mathcal{S}$ of stationary points of the population EM equals
\[
\mathcal{S} = \{ (\rho^*_1,\dots,\rho^*_n), (0,\dots,0)\} \cup 
\{ \overline{\rho}^{(i)} \colon i=1,\dots,n \}, \quad \text{where }
\overline{\rho}^{(i)}_j = \begin{cases}
1 & j = i \\
\rho^*_i \rho^*_j & j \ne i
\end{cases}\enspace.
\]
\end{lemma}

We will rule out the possibility of convergence to any point that is not the true parameter.

\begin{lemma}\label{lem:converge_int}
The correlation parameters of the population EM converge to a point in $(0,1)$:
\[
0 < \lim_{t \to \infty} \rho^t_i < 1 \quad \text{for all } i=1,\dots, n\enspace.
\]
\end{lemma}

Combining the two lemmas above, we obtain that $\rho_i^t$ must converge to $\rho^*_i$, which concludes Theorem~\ref{thm:one-formal}. In Section~\ref{sec:interior} we will prove that there are no stationary points where $\rho_i \in (0,1)$ for all $i$. In Section~\ref{sec:zero} we will prove that the only stationary point with some $\rho_i=0$ is $(0,\dots,0)$, and that the algorithm does not converge to this point. In Section~\ref{sec:boundary} we will prove that the only stationary points where some $\rho_i=1$ are $\overline{\rho}^{(i)}$ and that the algorithm does not converge to these points.
In Section~\ref{sec:app_final} we combine these results to prove Lemmas~\ref{lem:stationaries} and ~\ref{lem:converge_int}.
In Section~\ref{sec:sketch-boundary} we sketch the proofs from Section~\ref{sec:zero} and Section~\ref{sec:boundary}.

\subsection{No stationary points at the interior} \label{sec:interior}

In this section, we will prove that there are no stationary points of the EM with $\rho_i \in (0,1)$ for all $i$, as summarized in the following lemma:

\begin{lemma}\label{lem:int}
Let $\v\rho = (\rho_1,\dots,\rho_n)\in (0,1)^n$. If $\v\rho\ne \v\rho^*$ then $\v\rho$ is not a stationary point.
\end{lemma}

Below, fix some stationary point $\v\rho^t = (\rho_1,\dots,\rho_n) \in (0,1)^n$ and we will prove that $\v\rho=\v\rho^*$.
By Lemma~\ref{lem:convergence-to-stat},
$\E_{\mu^{t+1}}[x_iy] = \E_{\mu^t}[x_iy]$ for all $i$, and by Lemma~\ref{lem:update} this translates to
\[
\sigma_{x_t}^t \sigma_y^t \rho_i^t = \sigma^t_{i}\sigma^t_y \l(\rho_i^t + \sum_{j\ne i} \Delta_{ij}^t \lambda_j^t \r)
~\Longrightarrow~
\sum_{j\ne i} \Delta^t_{ij} \lambda_j^t = 0
~\Longrightarrow~
\sum_{j\ne i} \rho^*_i\rho^*_j\lambda^t_j = \sum_{j\ne i} \rho^t_i\rho^t_j\lambda^t_j \enspace.
\]
Multiplying by $\lambda^t_i$ and substitute $\rho^t_i \lambda^t_i = u^t_i$ and $\rho^*_j \lambda^t_j = u^*_j$, we obtain
\begin{equation} \label{eq:stationary-eq}
\forall i=1,\dots,n \colon 
\quad
\sum_{j\ne i} \rho^*_i\rho^*_j\lambda^t_i\lambda^t_j = \sum_{j\ne i} \rho^t_i\rho^t_j\lambda^t_i\lambda^t_j
~\Longrightarrow~
\sum_{j \ne i} u^t_iu^t_j
= \sum_{j \ne i} u^*_iu^*_j\enspace.
\end{equation}
By the assumption that $\rho^t_i,\rho^*_i\in (0,1)$ and by definition of $\lambda_i$ in \eqref{eq:conditional}, we have that $\lambda^t_i>0$, hence $u^t_i,u^*_i > 0$.
We will prove that this set of equations imply that $\v u^t = \v u^*$, as summarized below:
\begin{lemma}\label{lem:equation-system}
    Let $\v u, \v v \in (0,\infty)^n$, and assume that for all $i=1,\dots,n$, $\sum_{j \ne i} u_i u_j = \sum_{j\ne i} v_i v_j$. Then $\v u = \v v$.
\end{lemma}

Lemma~\ref{lem:equation-system} will imply that $\lambda^t_i\rho^t_i = \lambda^t_i\rho^*_i$
which implies that $\rho^t_i=\rho^*_i$ as required to conclude the proof of Lemma~\ref{lem:int}.

In the remainder of this section, we prove Lemma~\ref{lem:equation-system}. For this purpose, we have the following lemma, which is a special case of the Jacobian Conjecture on the uniqueness of solutions of polynomial systems. Its special case for degree-2 polynomials was proven by \cite{wang1980jacobian}. We state and prove a corollary of this statement:
\begin{lemma} \label{lem:Jacobian}
Let $p_1,\dots,p_n \colon \mathbb{R}^n \to \mathbb{R}$ a collection of quadratic polynomials. If there are two distinct vectors, $\v u$ and $\v u'$ such that $p_j(\v u) = p_j(\v u')$ for all $j$, then, the Jacobian matrix of $\v p$ computed at $(\v u+\v u')/2$, $J^{\v p}((\v u + \v u')/2)$, is singular, where
\[
J^{\v p}_{i,j}(\v v)
= \frac{dp_i(\v v)}{dv_j}.
\]
\end{lemma}
\begin{proof}
Look at the path $\gamma \colon [0,1]\to \mathbb{R}^n$ defined by $\gamma(s) = s \v u + (1-s) \v u'$. Then, $p_i(\gamma(s))$ is a quadratic polynomial in $s$, that satisfies $p_i(\gamma(0)) = p_i(\gamma(1))$. Therefore, $1/2$ is a stationary point, hence
\begin{align*}
0 = \frac{d}{ds} p_i(\gamma(s))\bigg|_{s=1/2}
= \sum_j \frac{d p_i}{dv_j}\bigg|_{\v v = \gamma(1/2)} \frac{d\gamma_j}{ds}\bigg|_{s=1/2}
= \sum_j \frac{d p_i}{dv_j}\bigg|_{\v v = \gamma(1/2)} (u'_j - u_j).
\end{align*}
Combining these equalities in a matrix form for all $i=1,\dots,n$, we derive that $0=J^{\v p}(\v\gamma(1/2)) (\v u'-\v u)
= J^{\v p}((\v u + \v u')/2) (\v u'-\v u)$
hence $J^{\v p}((\v u + \v u')/2)$ is singular.
\end{proof}
To complete the proof of Lemma~\ref{lem:equation-system}, let us substitute $p_i(\v u) = \sum_{j\ne i} u_i u_j$. The following lemma proves that the Jacobian matrix of this system is non-singular at any point $\v u$ with positive entries, which suffices to conclude the proof:
% and that Jacobian is
% \[
% J_{i,j}^{\v p}(\v u) = \begin{cases}
% \sum_{k \ne i} u_k & i = j \\
% u_i & i \ne j
% \end{cases}\enspace.
% \]
% We would like to substitute $\v u = \v u^t$ and $\v u' = \v u^*$, and notice that all the entries of both vectors are positive, hence, all the entries of $(\v u + \v u')/2$ are positive, which implies that it is sufficient to look at the Jacobian in points where all $u_i$ are positive. For that, we have the following lemma:
\begin{lemma}\label{lem:non-singular}
Let $u_1,\dots,u_n>0$. Then the following matrix $J$ is non singular:
\begin{equation}\label{eq:nonsingular-mat}
J_{ij} = \begin{cases}
\sum_{k \ne i} u_k & i = j\\
u_i & i \ne j
\end{cases}.
\end{equation}
\end{lemma}

Notice that this matrix can be written as $D + \v v \v w^\top$, where $D$ is diagonal and $\v v,\v w \in \mathbb{R}^n$. Such a matrix is singular if and only if $1 + \v v^\top A^{-1} \v w = 0$ \cite{sherman1950adjustment}. Proving that this matrix is non-singular from this formula is not immediate. For completeness, we present in 
Section~\ref{sec:app-proofs} a proof that does not rely on this formula.
This concludes Lemma~\ref{lem:equation-system} and Lemma~\ref{lem:int} follows.

\subsection{The limiting point is bounded away from the boundary} \label{sec:sketch-boundary}

In this section, we sketch the proof that $0 < \lim_{t\to \infty} \rho^t_i < 1$, whose full version appears in Section~\ref{sec:zero} and Section~\ref{sec:boundary}. 
First, by inspecting the update rule from Lemma~\ref{lem:update}, it is easy to verify that the only stationary points with some $\rho_i \in \{0,1\}$ 
are $\v 0=(0,\dots,0)$ and $\v\rho^{(i)}$ (defined formally in Lemma~\ref{lem:stationaries}). It remains to prove that the EM does not converge to these stationary points. To argue about $\v0$, we write the Taylor series of 
$\rho^{t+1}_i$ around $0$ and obtain \[
\rho^{t+1}_i \ge \rho^t_i + \sum_{j \ne i} \rho^t_j \rho^*_i\rho^*_j - O\l(\sum_{j,k} \rho^t_j\rho^t_k\r) \ge \rho^t_i + \Omega\l(\max_{j\ne i} \rho^t_j \r) - O\l(\max_{j\in\{1,\dots,n\}} (\rho^t_j)^2\r)\enspace.
\]

This can be used to show that whenever $\v\rho^t$ approaches $\v0$, in the subsequent iteration it will repel from $\v0$ and particularly, $\sum_i \rho^{t+1}_i \ge \sum_i \rho^t_i$. Hence, the algorithm cannot converge to $\v0$. 

Next, we explain why EM does not converge to $\v\rho^{(i)}$. Assume w.l.o.g that $i=1$ and recall that $\v\rho^{(1)} = (1,\rho^*_1\rho^*_2,\dots,\rho^*_1\rho^*_n)$. The proof is more involved than the proof at $\v0$, due to (1) one is required to derive a second-order Taylor series (compared to a first-order computed around $\v0$), and (2) one would like to argue that
if $\v\rho^t \approx \v\rho^{(1)}$ then $\|\v\rho^{t+1} - \v\rho^{(1)}\| \ge \|\v\rho^{t} - \v\rho^{(1)}\|$ (in some norm). Yet, this is true only for some values of $\v\rho^t$: in particular, this holds whenever $\max_i |\rho^t_i - \rho^{(1)}_i| \le O(|\rho^t_1-\rho^{(1)}_1|)$. By analyzing the update step of the EM, it can be shown that even if $\rho^t$ does not satisfy this condition, which might happen with a bad initialization ($t=0$), $\rho^{t+1}$ will always satisfy it and the proof follows.

\subsection{Finite sample and finite iteration} \label{sec:results-finite}

In this section we assume that $m$ draws from the marginal distribution over the leaves are given, where sample $i$ is denoted by $(x^{(i)}_1,\dots,x^{(i)}_n)$, for $i \in \{1,\dots,m\}$. Then, the EM iterate is defined by 
\[
\mu^{t+1} = \argmax_{\mu \in \mathcal P} \sum_{i=1}^m \E_{y \sim \mu^t_{y \mid x_1^{(i)} \cdots x_n^{(i)}}}\lp[\log \Pr_\mu\lp(x_1^{(i)},\dots,x_n^{(i)}, y\rp)\rp].
\]
We have the following theorem, which bounds the sample complexity and the convergence time of sample-based EM (the proof appears in Section~\ref{app:finite_convergence}):
\begin{theorem}\label{thm:general-finite}
Let $\alpha,\beta>0$ and assume that $\theta^0_i,\theta^*_i \ge \alpha$ for all $i$ and $\theta^0_i,\theta^*_i \le 1-\beta$ for all $i$. Then, there exist constants $C_1(\alpha,\beta),C_3(\alpha,\beta)$ that depend only on $\alpha$ and $\beta$ and universal constants $C_2,C_4$ such that the following holds: let $\epsilon,\delta>0$ and assume that the number of samples $m$ is at least $m \ge (C_1(\alpha,\beta) n^{C_2}+\log(1/\delta))/\epsilon^2$ and let $T\ge C_3(\alpha,\beta)n^{C_4} \log(1/\epsilon)$. Then, with probability $1-\delta$ over the sample, for any $t \ge T$, $|\rho^t_i - \rho^*_i| \le \epsilon$ and $|\sigma^t_i - \sigma^*_i| \le \epsilon$ for all $i\in \{1,\dots,n\}$.
\end{theorem}

\section{General Tree models}
We now extend the analysis for more complicated tree topologies.
\subsection{Definitions and properties of the model} \label{sec:def-tree}
We consider a multivariate Gaussian  latent-tree distribution, that is characterized by a tree $G = (V,E)$ with $n$ leaves $x_1,\dots,x_n$ (of degree 1) and $m$ internal nodes $y_1,\dots,y_n$ (of degree at least $2$). %When we want to refer to some node in the tree without wanting to specify whether it is a leaf or an internal node, we will use the symbol $z$. 
The vertices $u \in V$ corresponds to random variables $\{z_u \colon u \in V\}$
%and we sometimes refer to the random variables themselves as the nodes.
with the joint distribution
\begin{equation}\label{eq:product_intro}
\Pr[z_1,\ldots,z_{n+m}] = \prod_{(i,j) \in E} \Pr[z_i,z_j] 
\end{equation}
where $z_{x_i}$ are observed by the algorithm and $z_{y_i}$ are latent variables.
The variables are assumed zero-mean and the distribution is uniquely characterized by the variance parameters $\sigma_{u}^2 := \Var[Z_u]$ for $u \in V$ and the correlation parameters $\rho_{uv} := \Cov[Z_uZ_v]/\sqrt{\Var(Z_u)\Var(Z_v)}$ for each edge $(u,v) \in E$. As with one latent, we assume $\rho_{uv} \in [0,1]$. The correlation between any two nodes $u,v\in E$ can be expressed as the product of correlations along the path connecting them,
%For two arbitrary nodes $z_i,z_j$ in the tree, define $P(z_i,z_j)$ to be the set of edges in the unique path connecting $z_i$ to $z_j$.
%It can be shown that in a Gaussian Graphical Model that satisfies the product decomposition ~\ref{eq:product} 
$$
\rho_{z_iz_j} := \frac{\E[z_iz_j]}{\sigma_{z_i}\sigma_{z_j}} = \prod_{(z_u,z_v)\in P(z_i,z_j)}
\rho_{z_uz_v}, \quad \text{where $P(z_i,z_j)$ is the path connecting $i$ and $j$}.
$$
%Hence, these parameters are enough to specify the joint distribution over all $x$ and $y$. 
%We denote such distribution by $\mu_{G,\v\rho, \v\sigma}$ where $G$ is ommitted if it is clear from context. \yuval{do we need subscript $G$?}
Similarly to the argument regarding topologies of one latent node, the variances of the latent nodes $\sigma_{y_j}^2$ cannot be estimated (see Remark~\ref{r:var}). %\yuval{do we need what's next?} and one can assume that $\sigma_v = 1$ for all $v \in V$ (up to a scaling of the random variables).

An equivalent way to characterize the distribution is through the exponential family parametrization $(J=-\Sigma^{-1},h)$, where 
$
\Pr[z_1\cdots z_{n+m}]\propto  \exp\l(\sum_{ij} J_{ij}z_iz_j/2 + \sum_i h_iz_i\r).
$
Since the distribution factorizes as \eqref{eq:product_intro}, the only non-zero entries $J_{ij}$ correspond to edges $(i,j) \in E$ and for diagonal elements $i = j$. While in the original model $h_i=0$ since the random variables are zero mean, in the conditional distribution of $y$ given $x$ this is no longer true and
$
h_y = - J_{yx}x
$.

Lastly, we remark that the parameters are information-theoretically identifiable only when all the latent nodes have degree at least $3$, yet, any graph can be modified to to satisfy this property, while retaining the distribution over observables (see Remark  ~\ref{r:degree}).
% \begin{remark}
% The variances $\sigma_{y_v}^2$ on the internal nodes $y_v$ of the tree cannot be estimated. This is due to the fact that samples from two distributions $\mu_{\v\rho,\v\sigma}, \mu_{\v\rho,\v\sigma'}$ that have the same correlations $\v\rho$ and the same variances on the leaves, differ only by a scaling of the unobserved nodes. This scaling is unidentifiable from the marginal distribution of the leaves.
% \end{remark}
% \begin{remark}
%      One has to assume that each internal node has degree at least $3$. Indeed, if $v$ is an internal node with neighbors $u_1,u_2$, then we can vary the values $\rho_{u_1v},\rho_{vu_2}$ as long as $\rho_{u_1v}\rho_{vu_2}$ remains constant, without affecting the distribution over the leaves. In particular, such a change cannot be detected by the algorithm. In order to handle topologies where such a node $v$ exists, one consider the graph where $v$ is removed and its two neighbors are connected by an edge and this keeps the same distribution over the leaves. 
% \end{remark}
\subsection{EM and likelihood for general trees}
In this section, we analyze the EM iteration on latent tree models (see Section~\ref{s:one_latent_em} for an elaborate exposition).
Given an initial point $\mu^0 \in \mathcal{P}_G$, its iteration $\mu^t$ is defined as:
\begin{equation}\label{eq:em_dup}
\mu^{t+1} = \max_{\mu \in \mathcal P_G} \E_{x_1\cdots x_n \sim \mu^*_{x_1\cdots x_n}} \E_{y_1,\cdots,y_m \sim \mu^t_{y \mid x_1 \cdots x_n}}[\log \Pr_\mu(x_1,\dots,x_n, y_1,\dots, y_m)],
\end{equation}
where $\Pr_\mu$ denotes the density with respect to $\mu$. Denote by $\sigma_{\cdot}^t, \rho_i^t, J^t$ the parameters corresponding to $\mu^t$. Similarly, $\sigma_{\cdot}^*,\rho_i^*$ and $\lambda^*$ correspond to $\mu^*$. 
Using the same arguments as in Lemma~\ref{lem:covariance-conserve}, the following can be shown:
\begin{lemma}\label{l:update_rule_intro}
Let $\mu^{t,*}$ denote the joint distribution over $x_1\cdots x_n,y_1,\ldots,y_m$ such that
\[
\Pr_{\mu^{t,*}}[x_1,\cdots, x_n,y_1,\ldots,y_m] = \Pr_{\mu^*}[x_1,\cdots, x_n]\Pr_{\mu^t}[y_1,\ldots,y_m\mid x_1, \ldots,x_n].
\]
Then, for any $u\in  V$ and any $(v,w)\in E$,
\[
\Cov_{\mu^{t+1}}(z_v,z_w) = \Cov_{\mu^{t,*}}(z_v,z_w), \ 
\Var_{\mu^{t,*}}[z_u] = \Var_{\mu^{t+1}}[z_u]
\enspace.
\]
\end{lemma}

Denote by $\mathcal{S}$ the set of fixpoints of the EM, namely, the points $\rho$ such that $\v\rho^t = \v\rho$ implies $\v\rho^{t+1} = \v\rho$. Analogously to Lemma~\ref{lem:convergence-to-stat}, these correspond to the set of fixpoints $\mu$ of \eqref{eq:em_dup}. We can equivalently consider the EM iteration under the natural parametrization $J,h$, as discussed in Section~\ref{sec:def-tree}. Hence, a fixed point $\f\rho$ corresponds to some $\f J$, meaning that $\mathcal{S}$ remains unchanged if we change parametrization. The reason we choose the parameter $J$ is that the likelihood has a more convenient form as an exponential family. The importance of $\mathcal{S}$ is further exemplified by the following lemma, which states that the notions of EM fixpoints and stationary points of the log-likelihood are equivalent. The proof is folklore.

\begin{lemma}\label{l:em_equiv}
Let $\mu^* \in P_G$ be such that $\rho^*_{ij} \in (0,1)$ for all
$(i,j) \in E$ and define $L(J) := \E_{x\sim \mu^*} \log \Pr_{\mu(J)} (x)$. Then, for any $\f J \in \mathbb{R}^{(n+m)\times(n+m)}_+$ we have that 
$\nabla L(\f J) = 0$ if and only if $\f J$ is
a stationary point of the update rule \eqref{eq:em_dup}.

\end{lemma}

%\yuval{see how I defined the fixpoints above, in terms of $\rho$ only. We should use the same definition for fixpoints of the EM}
%Analogously to Lemma~\ref{lem:convergence-to-stat}, the fixpoints are determined only by the parameters $\rho^t$, hence, in the sequel we will focus, for convenience, on fixpoints where all nodes have variance $1$. %the determining quantity for the distribution in each iteration are the correlations $\rho^t$.  Hence, any fixpoint $\f \mu$ of \eqref{eq:em_dup} that is given by some parameters $\f\sigma,\f J$ can be converted into a fixpoint with the same likelihood value but with all internal nodes $y$ having variance $1$. Therefore, in the sequel we will cofus on fixpoints where all nodes have variance $1$.

\subsection{Uniqueness of EM fixpoints for general trees}
In this section, we prove that the only fixpoint of EM with non-degenerate edge weights is the true model $\mu^*$. A detailed proof of all the claims in this Section is given in Section~\ref{app:general}.

\begin{restatable}{theorem}{stationary}
\label{t:stationary}
Let $G = (V,E)$ be a tree and $\mu^* \in P_G$ be a distribution with $\rho^*_{ij} \in (0,1)$ for all
$(i,j) \in E$. 
Suppose $\f\rho$ is a stationary point of the EM update rule \eqref{eq:em}  with $\f\rho_{ij} \in (0,1)$ for all $(i,j) \in E$. Then
$
\f\rho_{ij} = \rho^*_{ij}
$
for all $(i,j) \in E$. 
\end{restatable}
We will denote by $\mu^{*,t}$ the distribution defined as
$
\mu^{*,t}(x,y) := \mu^*(x) \f \mu(y|x)
$.
We begin by exploring a simple implication of the fixpoint conditions, according to the rules~\eqref{eq:em_dup}. The proof follows by noticing that the conditional distribution of $y$ given $x$ is the same in $\f\mu$ and $\mu^{*,t}$. 

\begin{restatable}{lemma}{moments}
\label{l:moments}
Let $\mu^*,\f\mu$ be the distributions defined in Theorem ~\ref{t:stationary}. Then, for any internal nodes $y_1,y_2$ that are connected by and edge in $G$
\begin{align*}
\E_{\mu^*} \lp[\E_{\f\mu}[y_1|x]\E_{\f\mu}[y_2|x]\rp] &= \E_{\f\mu} \lp[\E_{\f\mu}[y_1|x]\E_{\f\mu}[y_2|x]\rp]\\
\E_{\mu^*} \lp[\E_{\f\mu}[y_1|x]^2\rp] = \E_{\f\mu} \lp[\E_{\f\mu}[y_1|x]^2\rp] &,\
\E_{\mu^*} \lp[\E_{\f\mu}[y_2|x]^2\rp] = \E_{\f\mu} \lp[\E_{\f\mu}[y_2|x]^2\rp]
\end{align*}
\end{restatable}
Hence, all we have to do is compute these conditional expectations in a suitable way, so as to reveal some structure. Let $N$ be the set of neighbors of $y_1$ in $G$ and assume w.l.o.g. that they are all non-leaf nodes (the case of leaf neighbors is even easier).
The first step will be to marginalize all the other non-leaf nodes except $y_1$ and $y_2$. Denote by $Y^c$ the set of non-leaf nodes that are not neighbors of $y_1$.
Then, let $h'_{y_1},h'_{y_2}$ be the information parameters of $y_1,y_2$ after the marginalization of $Y^c$. Since no neighbor of $y_1$ was marginalized, we will have $h'_{y_1} =0$.
Suppose we remove the edge $(y_1,y_2)$ from the tree. Then the set of leaves is partitioned into two subsets. Call $S_{y_2}$ the subset that is connected to $y_2$ after the removal.
We show that $h'_{y_2}$ is a linear combination of the values of the leaves in $S_{y_2}$. To do that, we utilize the marginalization formulas for Gaussian distributions, as well as the fact that $\f J_{Y^cY^c}^{-1}$ can be thought of as a covariance matrix of some Gaussian tree distribution and hence satisfies the multiplication over paths property.

\begin{restatable}{lemma}{linearcomb}
Let $y_2$ be a non-leaf neighbor of $y_1$ and $S_{y_2}$ be the corresponding set of leaves for the partition that $y_2$ belongs to.
Then,  the quantity
$
h'_{y_2} = h_{y_2} - \f J_{y_2 Y^c} \f J_{Y^cY^c}^{-1} h_{Y^c}
$
is a linear combination of the leaves in $S_{y_2}$. 
\end{restatable}

A similar result holds for any neighbor $y_i$ of $y_1$.
Our strategy is to build a system of equations similar to the one in Lemma~\ref{lem:equation-system}. But in the proof of one latent, the variables of the system corresponded to the covariances of the latent node with individual leaves. Here, for neighbor $y_2$ we will define some variable that depends on a linear combination of the paths leading to leaves in $S_{y_2}$. 
In this direction, let's
define the vector $H \in \mathbb{R}^s$, which has one entry for each node in $N$. We define $H_i = h'_{y_i}$.  

So far we have marginalized on everyone except 
nodes in $N$.
To gain specific information about the interaction of $y_1,y_2$, we now marginalize also over the neighbors of $y_1$, except $y_2$. Denote by $N_2$ that set of neighbors.
Let $h''_{y_1},h''_{y_2}, \f J''$ be the resulting parametrizations. From standard properties of the Gaussian, we can compute  these as follows

$$
h''_{y_1} = h_{y_1} -\f J_{y_1 N_2} (\f J'_{N_2N_2})^{-1} h'_{N_2} = \sum_{y_j \in N, y_j \neq y_2} \frac{\f J_{y_1 y_j}}{\f J'_{y_jy_j}} h'_{y_j} ,\ h''_{y_2} = h'_{y_2} - \f J_{y_2  N_2} (\f J'_{N_2N_2})^{-1} h'_{N_2} = h'_{y_2}
$$
This follows by the fact that the matrix $\f J'_{N_2N_2}$ is diagonal, since the neighbors of $y_1$ are not connected to each other, and $J_{y_2N_2} = 0$. 
To write this more compactly, we introduce the vector $r \in \mathbb{R}^s$, where $r_i = \f J_{y_1 y_i}/ \f J'_{y_iy_i}$ if $i \in N$. Notice that $r_i \neq 0$ always. Hence, this relation becomes 
\begin{equation}\label{eq:external_intro}
h''_{y_1} = \sum_{i \neq y_2} r_i H_i
\end{equation}

We are now close to arriving at the desired form of the polynomial equations. Since we have marginalized over all nodes except $y_1,y_2$, we can now determine the conditional expectations by the formula

$$
\begin{pmatrix}
\E_{\f\mu}[y_1|x]\\
\E_{\f\mu}[y_2|x]
\end{pmatrix} = 
\f\Sigma_{y_1y_2|x} 
\begin{pmatrix}
h''_{y_1}\\
h''_{y_2}
\end{pmatrix}
$$
that connects the two parametrizations in a Gaussian.
Here, $\f\Sigma_{y_1y_2|x}$ is the
$2\times 2$ covariance matrix of the conditional distribution of $y_1,y_2$ given $x$. 
The reason we have used $\f\Sigma_{y_1y_2|x}$ is that the covariance matrix does not change when we marginalize some nodes.
Suppose
$$
\f\Sigma_{y_1y_2|x} = \begin{pmatrix}
c_1 &c_2\\
c_3 &c_4
\end{pmatrix}
$$
with $c_1c_4 - c_2c_3 \neq 0$. The reason the variances are not necessarily $1$ is that we are now in the conditional model. Then, the first fixpoing condition of Lemma ~\ref{l:moments} translates to the following:

\begin{align*}
&\E_{x \sim \mu^*} [(c_1h''_{y_1} + c_2 h''_{y_2})(c_3 h''_{y_1} +  c_4h_{y_2})] = \E_{x \sim \f\mu} [(c_1h''_{y_1} + c_2 h''_{y_2})(c_3 h''_{y_1} +  c_4h_{y_2})] 
\end{align*}
We observe that this relation can be viewed as a linear equation in terms of the variables
$
\E_{x\sim \mu^*}[h_1^2]  - \E_{x \sim \f\mu}[h_1^2],\ \E_{x\sim \mu^*}[h_2^2]
- \E_{x \sim \f\mu}[h_2^2], \ \E_{x\sim \mu^*}[h_1h_2]  - \E_{x \sim \f\mu}[h_1h_2]
$.
The coefficients of these variables are functions of $c_1,c_2,c_3,c_4$.
We can obtain two other such equations from the other two conditions of Lemma ~\ref{l:moments}. 
It turns out that this $3\times 3$ system has a unique solution if and only if $\f \Sigma_{y_1y_2|x}$ is invertible. This is the content of the following Lemma.

\begin{lemma}
Let $\f\Sigma_{y_1y_2|x}$ be invertible. Then the conditions of Lemma ~\ref{l:moments} imply that
$$
\E_{x\sim \mu^*}[h''_{y_1}h''_{y_2}]  = \E_{x \sim \f\mu}[h''_{y_1}h''_{y_2}] , \ \E_{x\sim \mu^*}[(h''_{y_1})^2]  = \E_{x \sim \f\mu}[(h''_{y_1})^2] ,\ 
\E_{x\sim \mu^*}[(h''_{y_2})^2]  = \E_{x \sim \f\mu}[(h''_{y_2})^2]
$$
\end{lemma}
Using the previous calculations for $h''_{y_1},h''_{y_2}$, the first equality of this Lemma can be written as 
\begin{equation}\label{eq:system_tree}
\E_{x \sim \mu^*}\lp[H_{y_2} \lp(\sum_{i \neq y_2}r_i H_i\rp)\rp] = 
\E_{x \sim \f\mu}\lp[H_{y_2} \lp(\sum_{i \neq y_2}r_i H_i\rp)\rp]
\end{equation}
Now we are almost in the algebraic form of Lemma~\ref{lem:equation-system}. All we need to do is get rid of the expectations. So, it is time to compute them. Let $i,j \in N$ be two neighbors. Then, we know that $H_i,H_j$ are linear combinations of leaves in the partitions of $i,j$. Thus, by the multiplication property $\E[H_iH_j]$  involves the covariances of all pairs of leaves from $S_{y_i}$ to $S_{y_j}$. Suppose that $H_i = (a^i)^\top X_{y_i}$, where $X_{y_i}$ is the vector of leaves in $S_{y_i}$. By the path multiplication property
\begin{align*}
\E_{x\sim \mu^*} [H_iH_j] &= 
\E_{x\sim \mu^*} [(a^i)^\top X_{y_i} (a^j)^\top X_{y_j}]  = \sum_{x_k \in S_{y_i} , x_l \in S_{y_j}}a^i_ka^j_l \prod_{e \in P_(x_k,x_l)}\rho^*_e
\end{align*}
where $P(x_k,x_l)$ denotes the path between leaves $x_k,x_l$.
Now, notice that all the paths from $X_{y_i}$ to $X_{y_j}$ will have to go through the edges connecting $y_1$ to $y_i$ and $y_1$ to $y_j$. Let $\rho^*_i, \rho^*_j$ be the correlations in these two edges. Then, using the multiplication over paths, Equation \eqref{eq:system_tree} can be written as
 
\begin{align}\label{eq:paths_intro}
\rho^*_i
\lp(\sum_{x_r\in S_{y_i}} a^i_r \prod_{e \in P_{x_ry_i}} \rho^*_e\rp)\sum_{y_j \in N,j \neq i}\rho^*_j\lp(\sum_{x_s\in S_{y_j}} a^j_s \prod_{e \in P_{x_sy_j}} \rho^*_e\rp) \nonumber \\
= \f\rho_i
\lp(\sum_{x_r\in S_{y_i}} a^i_r \prod_{e \in P_{x_ry_i}} \f\rho_e\rp)\sum_{y_j \in N,j \neq i}\f\rho_j\lp(\sum_{x_s\in S_{y_j}} a^j_s \prod_{e \in P_{x_sy_j}} \f\rho_e\rp)
\end{align}
In this form, we can define
$$
w_i^* = \rho^*_i
\lp(\sum_{x_r \in S_{y_i}} a^i_r \prod_{e \in P(x_r,y_i)} \rho^*_e\rp) ,\
\f w_i = \f\rho_i
\lp(\sum_{x_r\in S_{y_i}} a^i_r \prod_{e \in P(x_r,y_i)} \f\rho_e\rp).
$$
Now the condition is
\begin{equation}\label{eq:final_intro}
r_j w_j^* \lp(\sum_{i\neq j} r_i w_i^*\rp) = 
r_j \f w_j \lp(\sum_{i\neq j} r_i \f w_i\rp)
\end{equation}
which is exactly in the form of the system of Lemma~\ref{lem:equation-system}. By applying this Lemma, we immediately get the following corollary.

\begin{lemma}\label{l:buildup}
Let $\mu^*,\f\mu$ be defined as in Theorem ~\ref{t:stationary}. Then for every node $i$, if we define $S_{y_i}$ in reference to some neighbor $y_1$, it holds $w_i^* = \f w_i$.
\end{lemma}
If a node $y_1$ is connected to a leaf $x_j$ with correlation $\rho_j^*$, then we can extend all the previous statements and define $w_j^* = \rho_j^*, \f w_j = \f\rho_j$. This immediately implies that $\rho_j^*=\f\rho_j$. 
The proof of Theorem ~\ref{t:stationary} relies exactly on using the equalities implied by Lemma~\ref{l:buildup} in the correct order, in order to guarantee that all correlations are the same in the two models.

\begin{proof}[Proof sketch of Theorem ~\ref{t:stationary}]
We use the following procedure: in each iteration, we select an internal node $y$ that only has one non-leaf neighbor in the remaining tree (there is always one such node). If $e$ is some edge connecting $y$ with some leaf in the tree, we declare that $\rho_e^* = \f\rho_e$ and remove this edge along with the leaf from the tree. After we do this for all such edges, the current iteration ends.

First of all, it is clear that if a node $y$ is selected for some iteration, then for the remaining iterations it will be a leaf and not be selected. Hence, the process terminates after $m$ steps, at which point all edges have been examined. We prove inductively that at each step the algorithm correctly infers the equality of the edges. For the base case, we already argued that edges that are adjacent to leaves will agree in the two models. For some arbitrary iteration, if $y$ is selected and has a leaf neighbor $x$, then we use Lemma~\ref{l:buildup} on $x$ with $y_1 = y$ to infer that $w_x^* = \f w_x$. From the definition of $w$ and the inductive hypothesis, the parentheses multiplying $\rho_i^*$ and $\f
\rho_i$ are equal in the two models, which implies that $\rho_i^* = \f\rho_i$ and the proof is complete.
\end{proof}

\section{Acknowledgements}
Constantinos Daskalakis was supported by NSF Awards CCF-1901292, DMS-2022448 and DMS2134108, a Simons Investigator Award, the Simons Collaboration on the Theory of Algorithmic Fairness, a DSTA grant, and the DOE PhILMs project (DE-AC05-76RL01830).
Vardis Kandiros was supported by a Fellowship of the Eric and Wendy Schmidt Center at the Broad Institute of MIT and Harvard and by the Onassis
Foundation-Scholarship ID: F ZP 016-1/2019-2020.

%\ifarx
\printbibliography
\appendix
\section{Deferred proofs for one latent node}
We structure this Section as follows. In Section~\ref{sec:zero} we prove that the only stationary point of EM with a zero coordinate is the all zeroes vector and EM never converges to it. 
In Section~\ref{sec:boundary} we prove that EM will not coverge to any stationary point where some $\rho_i = 1$. The main technical result in that section is Theorem~\ref{t:divergence}. In Section~\ref{sec:app-proofs} we collect the proofs of some Lemmas of the main text, including Lemma~\ref{lem:non-singular}, which implies Lemma~\ref{lem:int}. Finally, in Section~\ref{sec:app_final} we prove Lemmas~\ref{lem:stationaries} and
~\ref{lem:converge_int}, which combined with Lemma~\ref{lem:convergence-to-stat} imply Theorem~\ref{thm:one-formal}.

\subsection{The EM does not converge to $0$} \label{sec:zero}
In this section, we will prove the following two lemmas:
\begin{lemma}\label{lem:stationary-at-0}
    The point $\v 0 = (0,\dots,0)$ is a stationary point and there is no other stationary point $\v \rho$ with some $\rho_i = 0$. 
\end{lemma}
\begin{lemma}\label{lem:not-to-0}
    The iterate of the EM algorithm satisfies $\lim_{t\to\infty} \v\rho^t \ne \v 0$.
\end{lemma}
First, we notice that the point $\v 0 = (0,\dots,0)$ is a stationary point. Indeed, by \eqref{eq:conditional}, $\v\rho^t = (0,\dots,0)$ then $\lambda^t_i = 0$. Together with, Lemma~\ref{lem:update}, this implies that $\v\rho^{t+1} = (0,\dots,0)$ which concludes that $\v 0$ is a stationary point.
The following concludes the proof of Lemma~\ref{lem:stationary-at-0}:
\begin{lemma} \label{lem:not-jump-to-0}
If $\rho^t_{i_1} > 0$ for some $i_1>0$ then, $\rho^{t+1}_i > 0$ for all $i$.
Consequently, if $\v\rho$ satisfies $\rho_{i_0}=0$ and $\rho_{i_1}>0$ for some $i_0$ and $i_1$, then $\rho$ is not a stationary point.
\end{lemma}
\begin{proof}
We start with the first part of the proof.
Recall the update rule for the covariances, from Lemma~\ref{lem:update}:
\[
\E_{\mu^{t+1}}[x_i y] = \sigma^t_{x_i}\sigma^t_y \l(\lambda_i^t + \sum_{j\ne i} \rho^*_i\rho^*_j \lambda_j^t \r).
\]
Further, from \eqref{eq:conditional}, we know that $\lambda^t_i > 0$ whenever $\rho^t_i>0$. In particular, $\lambda^t_{i_1} >0$, which implies that for all $i$, $\E_{\mu^{t+1}}[x_{i} y] \ge \sigma_{x_i}^t\sigma_y^t \rho^*_{i}\rho^*_{i_1}\lambda^t_{i_1}>0$. This concludes the first part of the proof. For the second part, notice that if $\mu^t=\mu$ then, by the first part, $\rho^{t+1}_{i_0}>0=\rho^t_{i_0}$, hence $\mu$ cannot be a stationary point.
\end{proof}
The first part of the lemma follows from the update rule analyzed in Lemma~\ref{lem:update}, while the second part follows directly from the first part and the definition of a stationary point. This concludes Lemma~\ref{lem:stationary-at-0}.

For the remainder, we prove Lemma~\ref{lem:not-to-0}.
The first part of Lemma~\ref{lem:not-jump-to-0} implies that if $\v \rho^0 \ne \v 0$ then $\v \rho^t \ne \v 0$ for all $t>0$. It remains to prove that $\v\rho^t$ does not converge to $\v 0$ and we will analyze the updates of the EM rule in the neighborhood of $\v\rho^t=\v 0$, viewing $\v\rho^{t+1}$ as a function of $\v\rho^t$.
\begin{lemma}\label{lem:derivatives-at-0}
If $\v\rho^t = \v 0$ then $\v\rho^{t+1}=\v 0$. Further, there exists $C>0$ such that for all $i, j,k \in \{1,\dots,n\}$ and all $\v\rho \in [0,1/2]^n$,
\[
\frac{d\rho^{t+1}_i}{d\rho^t_j}\bigg|_{\v 0} = \begin{cases}
1 & i=j\\
\rho^*_i\rho^*_j & i\ne j
\end{cases} ~; \qquad
-C\le \frac{d^2\rho_i^{t+1}}{d\rho^t_j d\rho^t_k}\bigg|_{\v\rho}\le C.
\]
\end{lemma}
\begin{proof}
We use the formula for $\v\rho^{t+1}$ from Lemma~\ref{lem:update}:
\[
\rho_i^{t+1} = 
\frac{\rho_i^t + \sum_{j\ne i} \Delta_{ij} \lambda_j^t}{\sqrt{1 + \sum_{j\ne k} \Delta_{ij} \lambda_j^t \lambda_k^t}} := \frac{f_i(\v\rho^t)}{\sqrt{g_i(\v\rho^t)}}\enspace.
\]
Before computing the derivatives of $\rho^{t+1}_i$, let us compute the derivatives of $\lambda^t_i$ as a function of $\rho^t_j$ where the formula of $\lambda^t_i$ appears in \eqref{eq:conditional}. For any $i,j,k$ and $\v\rho\in [0,1/2]^n$,
\[
\frac{d\lambda^t_i}{d\rho^t_j}\bigg|_{\v0} = \begin{cases}
1 & i=j\\
0 & i\ne j
\end{cases} ~; \quad \text{and }
-C \le \frac{d^2\lambda^t_i}{d\rho^t_jd\rho^t_k}\bigg|_{\v\rho} \le C
\]
for some constant $C>0$. Using the fact that $\lambda^t_i=0$ if $\rho^t_i=0$, we derive that
\[
\frac{df_i(\v\rho^t)}{d\rho^t_j}\Bigg|_{\v0} = \begin{cases}
1 & j = i \\
\Delta_{ij} = \rho^*_i\rho^*_j & j \ne i
\end{cases}, \quad
f_i(\v0)=\v0, \quad
\frac{dg_i(\v\rho^t)}{d\rho^t_j}\Bigg|_{\v0} = 0, \quad g_i(\v0)=1.
\]
Hence,
\[
\frac{d\rho^{t+1}_i}{d\rho^t_j}\Bigg|_{\v0} = \frac{df_i(\v\rho^t)}{d\rho^t_j} 
= \begin{cases}
1 & j = i \\
\rho^*_i\rho^*_j & j \ne i
\end{cases}\enspace.
\]
Similarly, the second derivatives of $\rho^{t+1}_i$ are bounded, using the bounds on the second derivatives of $\lambda^t_j$.
\end{proof}
This implies that if $\v\rho^t$ approaches $\v0$, then $\v\rho^{t+1}$ repels from $\v0$:
\begin{lemma}\label{lem:goes-out-from-0}
There exists some $c>0$ such that if $\max_{i} \rho^t_i \le c$ then $\sum_{i} \rho^t_{i+1} \ge \sum_i \rho^t_i$.
\end{lemma}
\begin{proof}
Using Lemma~\ref{lem:derivatives-at-0}, we can write $\rho^{t+1}_i$ as a Taylor series around $\v 0$  
\begin{equation*}
\rho^{t+1}_i \ge \rho^t_i +  \sum_{j\ne i} \rho_j\rho^*_i \rho^*_j -\sum_{j,k} \frac{C}{2} \rho^t_j\rho^t_k.
\end{equation*}
Summing over $i$, we derive that
\[
\sum_i \rho^{t+1}_i \ge \sum_i \rho^t_i + \sum_{i\ne j} \rho_j \rho^*_i \rho^*_j -\sum_{i,j,k} \frac{C}{2} \rho^t_j\rho^t_k.
\]
While the second term in the right hand side is $\Omega(\max_j \rho^t_j)$, the third term is $O(\max_j (\rho^t_j)^2)$. In particular, if the constant $c>0$ from the definition of this lemma is sufficiently small, then the second term dominates the third and the proof follows.
\end{proof}

To conclude the proof of Lemma~\ref{lem:not-to-0}, assume toward contradiction that $\v\rho^t\to \v0$. Let $c>0$ be the parameter from Lemma~\ref{lem:goes-out-from-0}, and let $T>0$ be the iteration such that for any $t \ge T$, $\max_i \rho^t_i < c$. From Lemma~\ref{lem:goes-out-from-0}, this implies that for any $t \ge T$, $\sum_j \rho^t_j \ge \sum_j \rho^T_j$. From Lemma~\ref{lem:not-jump-to-0}, $\sum_j \rho^T_j > 0$. In particular, $\liminf_{t\to \infty} \sum_j \rho^t_j \ge \sum_j \rho^T_j > 0$ which implies that $\lim_{t\to\infty}\v\rho^t \ne \v0$, as required.
\subsection{Stationary points with some $\rho_i = 1$}\label{sec:boundary}

In this Section, we analyze the case of stationary points such that there exists at least one $i$ with $\rho_i= 1$. We show that EM will never converge to any of those points. Let us start with the case where there are at least two $i,j$ such that $\rho_i = \rho_j = 1$. In that case, let $\mu$ be the distribution at this point. Also, let 
$$
L(\rho) = \E_{x \sim \mu^*} \log \Pr_\rho(\v x)
$$
be the log-likelihood function. The following lemma shows that the EM algorithm will never converge to this stationary point.

\begin{lemma}
We have that
$$
\lim_{\rho_i\to 1, \rho_j \to 1} L(\rho) = -\infty
$$
\end{lemma}
\begin{proof}
We can write the log-likelihood function as follows:
\begin{align*}
L(\rho) &= \E_{x \sim \mu^*} \log \Pr_\rho(\v x) \\
&= \E_{x \sim \mu^*} \log\lp( \Pr_\rho( x_i,x_j) \Pr_\rho\lp(\v x_{[n]\setminus \{i,j\}}\mid x_i,x_j\rp) \rp)\\
&= \E_{x \sim \mu^*} \log \Pr_\rho( x_i,x_j)  +
\E_{x \sim \mu^*} \log\Pr_\rho\lp(\v x_{[n]\setminus \{i,j\}}\mid x_i,x_j\rp)
\end{align*}
The second term is clearly upper bounded over the whole region.
For the first term, it is clear that 
$$
\E_{x\sim \mu^*} \log \Pr_\rho(x_i,x_j) = -KL(\mu^*(x_i,x_j), \mu(x_i,x_j;\rho)) - \E_{x\sim\mu^*} \log \mu^*(x_i,x_j)
$$
and it is clear that
$$
\lim_{\rho_i\to 1,\rho_j \to 1} KL(\mu^*(x_i,x_j), \mu(x_i,x_j;\rho)) = \infty 
$$
while the second term is constant. It follows that
$$
\lim_{\rho_i \to 1,\rho_j \to 1} \E_{x \sim \mu^*} \log \Pr_\rho( x_i,x_j) = -\infty
$$
which completes the proof.
\end{proof}

Since the EM is guaranteed to improve the value of the likelihood at every step, it is impossible to converge to this point.

First, we give a Lemma that characterizes the set of stationary points where some $\rho_i = 1$. 
\begin{lemma}\label{lem:stationary-1}
Define for all $i \in [n]$ a vector $g^i \in \R^n$ as
$$
(g^i)_j = \left\{
\begin{array}{ll}
     1 &\text{ if $i = j$}\\
     \rho_i^* \rho_j^* &\text{ if $i \neq j$}
\end{array} 
\right.
$$
The set of stationary points of the population EM update where some $\rho_i = 1$ is exactly the set
$$
\{g^i : i \in [n]\}
$$
\end{lemma}
\begin{proof}
Let us consider the case of a stationary point where there is exactly one $i$ with $\rho_i = 1$. The reason is that if at least two coordinates are $1$, then as we showed earlier, the value of the likelihood is $-\infty$ so these are not stationary points. Without loss of generality, suppose $\rho_1 = 1$. The reason this is a fixpoint is that if $\rho^t_1 = 1$, then $y = x_1$ in the conditional model of $y $ given $x$, which means that $\rho^{t+1} = 1$.  Then, by the fixpoint equations we immediately get that in the fixpoint we should have
$$
\f\rho_i = \rho_i^* \rho_1^*
$$
for $i \neq 1$, which is the vector $g^1$. 
Similarly, by setting $\rho_i =1$ we get exactly one such fixpoint $g^i$. 
\end{proof}

From now on, when we refer to $\f\rho$, it will be implied that it is equal to $g^1$. 
We would like to show that when we start running EM from a point in the interior of the region,
we will not converge to this stationary point. However, we already know that if we start from a point with $\rho_1 = 1$, then inevitably in the next iteration we will converge to the stationary point values for the other variables as well. Hence, this stationary point will be a saddle point of the log-likelihood. 

In particular, we will prove the following Theorem. 

\begin{theorem}\label{t:divergence}
Let $\rho$ be a stationary point with $\rho_1 = 1$. Then, there exists an $\epsilon_0 > 0$, such that the following condition holds: if we start running EM from any initial point $\rho^0$ that satisfies $\|\rho^0 - \rho\|_2 < \epsilon_0$ and such that $\rho^0_i < 1$ for all $i$, then after a finite number of iterations $t = t(\rho_0)$
we have
$$
\|\rho^t - \rho\|_2 > \epsilon_0
$$
\end{theorem}
We stress that the number of iterations $t$ depends on the initial value $\rho_0$. Let us argue why Theorem ~\ref{t:divergence} implies that EM does not converge to $\rho$ from any initial point inside the region. Indeed, if EM converged for some initial value $\rho^0$, then for any $\epsilon > 0$, and in particular for $\epsilon = \epsilon_0$, there exists some iteration $T$ such that for $t \geq T$ we get
$$
\|\rho^t - \rho\| \leq \epsilon
$$
However, if we apply Theorem ~\ref{t:divergence} for $\rho^0 = \rho^T$, we get that there should exists some $t>T$ with $\|\rho^t - \rho\|>\epsilon_0$. This is a contradiction to the convergence claim.
Hence, it suffices to prove Theorem~\ref{t:divergence}.
The first step is to establish the following Lemma. It's purpose is to show that if the distances of all the coordinates are roughly equal, then $\rho_1$ will move away from the fixpoint value of $1$. 

\begin{lemma}\label{l:pushback-pop}
Let $\rho^t$ be the current iteration of EM and define
$$
\epsilon := \max_i |\f\rho_i - \rho^t_i|
$$
Then, there are constants $C,K >0$ such that
$$
\rho^{t+1}_1 \leq \rho^t_1 - C|\rho_1^t - 1|^2 + K\epsilon^3
$$

\end{lemma}
\begin{proof}
The proof consists of viewing the update rule as a function of the previous step and writing the Taylor expansion of the this function around the fixpoint $\rho$.
In particular, we define  the function
$$
f(\rho) = 
\frac{\rho_1 + \sum_{j\ne 1} (\rho_1\rho_j - \rho^*_1\rho^*_j)\lambda_j}{\sqrt{1 + \sum_{j, k \colon j\ne k} (\rho_j \rho_k - \rho^*_j\rho^*_k) \lambda_j \lambda_k}}
$$
where 
$$
\lambda_j = \frac{\frac{\rho_j}{1-\rho_j^2}}{1 + \sum_k \frac{\rho_k^2}{1 - \rho_k^2}}
$$
Clearly, we have that 
$$
f(\rho^t) = \rho^{t+1}
$$
We will compute the Taylor expansion of $f$ around the fixpoint $\rho$. We start by computing the first derivatives of $f$ at the fixpoint $\f \rho$.
To do that, we need the derivatives of $\lambda_j$ with $\rho_i$. We start with the case $i=j$

\begin{align*}
\frac{\partial\lambda_1}{\partial \rho_i} &= 
\frac{\frac{\partial}{\partial \rho_i}\lp(\frac{\rho_i}{1-\rho_i^2}\rp)}{1 + \sum_j \frac{\rho_j^2}{1 - \rho_j^2}}- 
\frac{\partial}{\partial \rho_i}\lp(1 + \sum_j \frac{\rho_j^2}{1 - \rho_j^2}\rp) \frac{\lp(\frac{\rho_i}{1-\rho_i^2}\rp)}{\lp(1 + \sum_j \frac{\rho_j^2}{1 - \rho_j^2}\rp)^2}\\
&= 
\frac{\frac{1}{1 - \rho_i^2} + \frac{2\rho_i^2}{(1 - \rho_i^2)^2}}{1 + \sum_j \frac{\rho_j^2}{1 - \rho_j^2}} -
\lp(\frac{2\rho_i}{(1 - \rho_i^2)^2}\rp)
\frac{\lp(\frac{\rho_i}{1-\rho_i^2}\rp)}{\lp(1 + \sum_j \frac{\rho_j^2}{1 - \rho_j^2}\rp)^2}\\
&= \frac{\lambda_1}{\rho_i} + 2\frac{\rho_i \lambda_1}{1 - \rho_i^2} - 2\frac{\lambda_1^2}{1 - \rho_i^2}
\end{align*}
Let's evaluate the limit as $\rho \to \f\rho$. We know that as $\rho_1 \to 1$, we will have $\lambda_1 \to 1$ and $\lambda_j \to 0$ for $j \neq 1$. Based on these observations, we have that 
\begin{align*}
\lim_{\rho_1 \to 1}
\frac{\rho_1 - \lambda_1}{1 - \rho_1^2}
&= \frac{\rho_1 - \frac{\frac{\rho_1}{1 - \rho_1^2}}{1 + \sum_j \frac{\rho_j^2}{1 - \rho_j^2}}}{1 - \rho_1^2} \\
&= \lim_{\rho_1 \to 1} \frac{\rho_1 + \rho_1\sum_j \frac{\rho_j^2}{1 - \rho_j^2} - \frac{\rho_1}{1 - \rho_1^2}}{\lp(1 + \sum_j \frac{\rho_j^2}{1 - \rho_j^2}\rp)(1 - \rho_1^2)}\\
&= \lim_{\rho_1 \to 1} \lp(\rho_1 + \frac{\rho_1^3 - \rho_1}{1 - \rho_1^2} + \rho_1 \sum_{j \neq 1} \frac{\rho_j^2}{1 - \rho_j^2}\rp)\\
&= \sum_{j \neq 1} \frac{\rho_j^2}{1 - \rho_j^2}
\end{align*}
Hence, 
$$
\lim_{\rho \to \f\rho} \frac{\partial\lambda_1}{\partial\rho_1} = 
2 \sum_{j \neq 1} \frac{\f\rho_j^2}{1 - \f\rho_j^2} + 1 = 
2 \sum_{j \neq 1} \frac{(\rho^*_j)^2(\rho^*_1)^2}{1 - (\rho^*_j)^2(\rho^*_1)^2} + 1 
$$

For $j \neq 1$, we have
\begin{align*}
\frac{\partial\lambda_j}{\partial\rho_i} &= 
\frac{\partial}{\partial\rho_i}\frac{\frac{\rho_j}{1 - \rho_j^2}}{1 + \sum_k \frac{\rho_k^2}{1 - \rho_k^2}} 
= - \frac{\frac{\rho_j}{1 - \rho_j^2}}{\lp(1 + \sum_k \frac{\rho_k^2}{1 - \rho_k^2}\rp)^2}\frac{\partial}{\partial \rho_i}\lp(\frac{\rho_i^2}{1 - \rho_i^2}\rp)\\
&= - \frac{\frac{\rho_j}{1 - \rho_j^2}}{\lp(1 + \sum_k \frac{\rho_k^2}{1 - \rho_k^2}\rp)^2}
\frac{2\rho_i}{(1 - \rho_i^2)^2}\\
&= - 2\frac{\lambda_1\lambda_j}{1 - \rho_i^2}
\end{align*}

We know that as $\rho_i \to 1$, we have $\lambda_1 \to 1$. Also,
$$
\lim_{\rho_i \to 1} \frac{\lambda_j}{1 - \rho_i^2} = \rho_j/(1-\rho_j^2)
$$
Hence,
$$
\lim_{\rho\to \f\rho} \frac{\partial\lambda_j}{\partial\rho_i} = -2\frac{\rho^*_j\rho_1^*}{1-(\rho^*_j\rho_1^*)^2}
$$
What we mostly care about is that these derivatives are bounded.
Now, we are ready to compute the various derivatives of $f$ at the fixpoint. We will use the notation $\Delta_{jk} := \rho_j^*\rho_k^* - \rho_j\rho_k$.
\begin{align*}
\frac{\partial f}{\partial\rho_1} &= 
\frac{\frac{\partial}{\partial \rho_1}\lp(\rho_1 + \sum_{j\ne 1} \Delta_{1j}\lambda_j\rp)}{\sqrt{1 + \sum_{j, k \colon j\ne k} \Delta_{jk} \lambda_j \lambda_k}} - \frac{1}{2}
\frac{\rho_1 + \sum_{j\ne 1} \Delta_{1j}\lambda_j}{\lp(1 + \sum_{j, k \colon j\ne k} \Delta_{jk} \lambda_j \lambda_k\rp)^{3/2}}\frac{\partial}{\partial \rho_1}\lp(1 + \sum_{j, k \colon j\ne k} \Delta_{jk} \lambda_j \lambda_k\rp)\\
&= \frac{1 + \sum_{j\neq 1}\Delta_{1j}\frac{\partial\lambda_j}{\partial\rho_1} - \sum_{j\neq 1}\lambda_j \rho_j}{\sqrt{1 + \sum_{j, k \colon j\ne k} \Delta_{jk} \lambda_j \lambda_k}}\\
&- \frac{1}{2} \frac{\rho_1 + \sum_{j\ne 1} \Delta_{1j}\lambda_j}{\lp(1 + \sum_{j, k \colon j\ne k} \Delta_{jk} \lambda_j \lambda_k\rp)^{3/2}}
\lp(2 \sum_{j \neq 1} \frac{\partial\lambda_1\lambda_j}{\partial\rho_1} \Delta_{1j} - 2 \sum_{j\neq 1} \lambda_1\lambda_j \rho_j - \sum_{j\neq k \neq 1} \frac{\partial\lambda_j\lambda_k}{\partial\rho_1}\Delta_{jk}\rp)
\end{align*}

Now, let us examine what happends when we plug in the fixpoint. We know that $\Delta_{1j} \to 0$ and $\lambda_j \to 0$ for $j \neq 1$, hence the entire first term tends to $1$. As for the second term, we have
$$
\lim_{\rho_1 \to 1} \frac{\rho_1 + \sum_{j\ne 1} \Delta_{1j}\lambda_j}{\lp(1 + \sum_{j, k \colon j\ne k} \Delta_{jk} \lambda_j \lambda_k\rp)^{3/2}} = 1
$$
Also, notice that 
$$
\frac{\partial\lambda_1 \lambda_j}{\partial\rho_1} = \frac{\partial\lambda_1}{\partial\rho_1} \lambda_j + \frac{\partial \lambda_j}{\partial\rho_1} \lambda_1
$$
is bounded, hence 
$$
\lim_{\rho \to \f\rho} \sum_{j \neq 1} \frac{d\lambda_1\lambda_j}{d\rho_i} \Delta_{1j} = 0
$$
It is also clear that
$$
\lim_{\rho_i \to 1} \sum_{j \neq 1}\lambda_1 \lambda_j \rho_j  = 0
$$
Lastly, we have
$$
\frac{\partial \lambda_j\lambda_k}{d\rho_i} = \lambda_j \frac{d\lambda_k}{d\rho_i} + \lambda_k \frac{\partial \lambda_j}{d\rho_i} \to 0
$$
Hence, the second term tends to $0$.
Hence,
$$
\lim_{\rho \to \f\rho} \frac{\partial f}{\partial\rho_1} = 1
$$
Let's now calculate $\partial f/\partial\rho_u$ for some $u \neq 1$. Using the exact same calculation, we arrive at the formula
\begin{align*}
&\frac{\partial f}{\partial\rho_u} = 
\frac{\sum_{j\neq 1}\Delta_{1j}\frac{\partial\lambda_j}{\partial\rho_u} - \lambda_u \rho_1}{\sqrt{1 + \sum_{j, k \colon j\ne k} \Delta_{jk} \lambda_j \lambda_k}}\\
&- \frac{1}{2} \frac{\rho_1 + \sum_{j\ne 1} \Delta_{1j}\lambda_j}{\lp(1 + \sum_{j, k \colon j\ne k} \Delta_{jk} \lambda_j \lambda_k\rp)^{3/2}}
\lp(2 \sum_{j \neq u} \frac{d\lambda_u\lambda_j}{\partial\rho_u} \Delta_{uj} - 2 \sum_{j\neq u} \lambda_j\lambda_u \rho_j - \sum_{j\neq k \neq u} \frac{\partial \lambda_j\lambda_k}{\partial\rho_u}\Delta_{jk}\rp)
\end{align*}
This time, both the first and the second terms converge to $0$, hence
$$
\lim_{\rho\to \f \rho} \frac{\partial f}{\partial\rho_u} = 0
$$

Now, we move on to the second derivatives.
First of all, let's compute the ones for $\lambda_1$. We just want to ensure that they will be bounded. 

\begin{align*}
\frac{\partial^2\lambda_j}{\partial^2\rho_1} &= 
\frac{\partial}{\partial\rho_1}\lp(-2\frac{\lambda_1\lambda_j}{1 - \rho_1^2}\rp) \\
&= - 2 \frac{(\partial\lambda_1/\partial\rho_1)\lambda_j}{1 - \rho_1^2} - 2\frac{(\partial\lambda_j/\partial\rho_1)\lambda_1}{1 - \rho_1^2} -4\frac{\lambda_1\lambda_j\rho_1}{(1 - \rho_1^2)^2}\\
&= - 2 \frac{(\partial\lambda_1/\partial\rho_1)\lambda_j}{1 - \rho_1^2} +4\frac{\lambda_j\lambda_1^2}{(1 - \rho_1^2)^2} -4\frac{\lambda_1\lambda_j\rho_1}{(1 - \rho_1^2)^2}\\
&=- 2 \frac{(\partial\lambda_1/\partial\rho_1)\lambda_j}{1 - \rho_1^2} +4\frac{\lambda_1\lambda_j}{1 - \rho_1^2} \lp(\frac{\lambda_1 - \rho_1}{1 - \rho_1^2}\rp)
\end{align*}
The first term is constant in the limit and the second term is also constant,  since the limit 
$\lim_{\rho_1 \to 1} (\lambda_1 - \rho_1)/(1 - \rho_1^2)$ has been shown to be constant. 
Similarly, we have
\begin{align*}
\frac{\partial^2 \lambda_j}{\partial\rho_1 \partial\rho_u} &= \frac{\partial}{\partial\rho_u}\lp(-2\frac{\lambda_1\lambda_j}{1 - \rho_1^2}\rp) \\
&= -2\frac{(\partial\lambda_1/\partial\rho_u)\lambda_j + (\partial\lambda_j/\partial\rho_u)\lambda_1}{1-\rho_1^2}\\
&= 4\frac{\lambda_1\lambda_j \lambda_u}{(1-\rho_1^2)(1-\rho_u^2)} +4\frac{\lambda_j\lambda_u\lambda_1}{(1-\rho_1^2)(1-\rho_u^2)}
\end{align*}
It's clear that both terms tend to $0$. Finally, let's calculate

\begin{align*}
\frac{\partial^2\lambda_1}{\partial^2\rho_1} &=
\frac{\partial}{\partial \rho_1} \lp(\frac{\lambda_1}{\rho_1} + 2\frac{\rho_1\lambda_1}{1-\rho_1^2} - 2\frac{\lambda_1^2}{1 - \rho_1^2}\rp)\\
&= \frac{\partial}{\partial\rho_1}\lp(\frac{\lambda_1}{\rho_1}\rp) + 2\frac{\partial\lambda_1}{\partial\rho_1}\frac{\rho_1 - \lambda_1}{1 - \rho_1^2} + 2 \lambda_1 \frac{\partial}{\partial\rho_1} \lp(\frac{\rho_1 - \lambda_1}{1 - \rho_1^2}\rp)
\end{align*}
By taking the limit $\rho_1 \to 1$ we verify easily that the second derivative is bounded.

Now, let's compute the second derivative $\partial^2 f/\partial^2\rho_1$. To make the presentation easier, we derive each of the two terms separately.
\begin{align*}
&\frac{\partial}{\partial\rho_1}\lp(\frac{1 + \sum_{j \neq 1}\Delta_{1j}\frac{\partial\lambda_j}{\partial \rho_1} - \sum_{j \neq 1}\lambda_j \rho_j}{\sqrt{1 + \sum_{j, k \colon j\ne k} \Delta_{jk} \lambda_j \lambda_k}}\rp) =
\frac{\sum_{j \neq 1}\Delta_{1j}\frac{\partial^2\lambda_j}{\partial^2\rho_1} - \sum_{j \neq 1}\frac{\partial \lambda_j}{\partial \rho_1}\rho_j- \sum_{j \neq 1}\frac{\partial \lambda_j}{\partial \rho_1} \rho_j}{\sqrt{1 + \sum_{j, k \colon j\ne k} \Delta_{jk} \lambda_j \lambda_k}} \\
&- 
\frac{1}{2} \frac{1 + \sum_{j \neq 1}\Delta_{1j}\frac{\partial \lambda_j}{\partial \rho_1} - \sum_{j \neq 1}\lambda_j \rho_j}{\lp(1 + \sum_{j, k \colon j\ne k} \Delta_{jk} \lambda_j \lambda_k\rp)^{3/2}}\lp(2 \sum_{j \neq 1} \frac{d\lambda_1\lambda_j}{\partial \rho_1} \Delta_{1j} - 2 \sum_{j \neq 1} \lambda_1\lambda_j \rho_j - \sum_{j\neq k \neq i} \frac{\partial \lambda_j\lambda_k}{\partial \rho_1}\Delta_{jk}\rp)
\end{align*}
In the first term, the final result once we take the limit is $$
-2 \sum_{j \neq 1} \rho_j \frac{\partial \lambda_j}{\partial \rho_1} \mid_{\rho = \f \rho}
$$
As for the second term, the rightmost parenthesis is $0$, as shown in the calculation of the first derivative. Now, let's move to the second term. We want 
$$
\frac{\partial}{\partial \rho_1}\frac{1}{2} \frac{\rho_1 + \sum_{j\ne i} \Delta_{1j}\lambda_j}{\lp(1 + \sum_{j, k \colon j\ne k} \Delta_{jk} \lambda_j \lambda_k\rp)^{3/2}}
\lp(2 \sum_{j \neq 1} \frac{d\lambda_1\lambda_j}{\partial \rho_1} \Delta_{1j} - 2 \sum_{j \neq 1} \lambda_1\lambda_j \rho_j - \sum_{j\neq k \neq i} \frac{\partial \lambda_j\lambda_k}{\partial \rho_1}\Delta_{jk}\rp)
$$
We can view this as a product of three terms and use the product rule. Notice that the third term is $0$ in the fixpoint, so when deriving the first two terms we will get $0$ in the final expression. Hence, the only term that matters is
\begin{align*}
\frac{1}{2} \frac{\rho_1 + \sum_{j\ne i} \Delta_{1j}\lambda_j}{\lp(1 + \sum_{j, k \colon j\ne k} \Delta_{jk} \lambda_j \lambda_k\rp)^{3/2}}
\frac{\partial}{\partial \rho_1}\lp(2 \sum_{j \neq 1} \frac{d\lambda_1\lambda_j}{\partial \rho_1} \Delta_{1j} - 2 \sum_{j \neq 1} \lambda_1\lambda_j \rho_j - \sum_{j\neq k \neq i} \frac{\partial \lambda_j\lambda_k}{\partial \rho_1}\Delta_{jk}\rp) \\
= \frac{1}{2} \frac{\rho_1 + \sum_{j\ne i} \Delta_{1j}\lambda_j}{\lp(1 + \sum_{j, k \colon j\ne k} \Delta_{jk} \lambda_j \lambda_k\rp)^{3/2}}
\lp(2 \sum_{j \neq 1} \frac{\partial^2\lambda_1\lambda_j}{\partial^2\rho_1} \Delta_{1j} - 2\sum_{j \neq 1}\frac{d\lambda_1\lambda_j}{\partial \rho_1}\rho_j - 2 \sum_{j \neq 1} \frac{d\lambda_1\lambda_j}{\partial \rho_1} \rho_j - \sum_{j\neq k \neq i} \frac{\partial^2\lambda_j\lambda_k}{\partial^2\rho_1}\Delta_{jk}\rp)
\end{align*}
First of all, as usual we have
$$
\lim_{\rho \to \f \rho} \frac{\rho_1 + \sum_{j\ne i} \Delta_{1j}\lambda_j}{\lp(1 + \sum_{j, k \colon j\ne k} \Delta_{jk} \lambda_j \lambda_k\rp)^{3/2}} = 1
$$
Also, we have 
\begin{align*}
&\lim_{\rho\to \f\rho} \sum_{j \neq 1} \frac{\partial^2\lambda_1\lambda_j}{\partial^2\rho_1} \Delta_{1j} = 0\\
&\lim_{\rho\to \f\rho} \sum_{j \neq 1}\frac{\partial\lambda_1\lambda_j}{\partial \rho_1}\rho_j = \lim_{\rho\to \f\rho} \sum_{j \neq 1}\lp(\lambda_1 \frac{\partial \lambda_j}{\partial \rho_1} + \lambda_j \frac{d\lambda_1}{\partial \rho_1}\rp)\rho_j = \sum_{j \neq 1} \rho_j \frac{\partial \lambda_j}{\partial \rho_1} \mid_{\rho = \f\rho}
\end{align*}
Hence, we get the same expression as in the first term and these two cancel each other. We are only left with computing the term
$$
\lim \sum_{j\neq k \neq 1} \frac{\partial^2\lambda_j\lambda_k}{\partial^2\rho_1}\Delta_{jk}
$$
If we set $g(x) = x/(1-x^2)$, then we get
\begin{align*}
\lim_{\rho\to \f\rho} \sum_{j\neq k \neq 1} \frac{\partial^2\lambda_j\lambda_k}{\partial^2\rho_1}\Delta_{jk} &= \lim_{\rho\to\f\rho}  \sum_{j\neq k \neq 1} \lp(\frac{\partial^2\lambda_j}{\partial^2\rho_1}\lambda_k + \frac{\partial^2\lambda_k}{\partial^2\rho_1}\lambda_j+ 2 \frac{\partial \lambda_j}{\partial \rho_1} \frac{d\lambda_k}{d\rho_i}\rp)\Delta_{jk}\\
&= 8\sum_{j\neq k\neq 1} g(\rho_1^*\rho_j^*)g(\rho_1^*\rho_k^*)\rho_j^*\rho_k^*(1 - (\rho_1^*)^2)
\end{align*}
Hence, 
$$
\frac{\partial^2 f}{\partial^2\rho_1} = - 4\sum_{j\neq k\neq 1} g(\rho_1^*\rho_j^*)g(\rho_1^*\rho_k^*)\rho_j^*\rho_k^*(1 - (\rho_1^*)^2) < 0
$$
Let's define 
$$
C := 2\sum_{j\neq k\neq 1} g(\rho_1^*\rho_j^*)g(\rho_1^*\rho_k^*)\rho_j^*\rho_k^*(1 - (\rho_1^*)^2)
$$
Clearly $C$ is a constant depending only on the true model.
Now, let's compute the other mixed derivatives. Let $u \neq 1$. We will compute $\partial^2f/\partial \rho_1\partial\rho_u$. 
Again, for convenience we will derive each of the two terms of $\partial\rho_1'/\partial\rho_1$ separately. 
We have
\begin{align*}
&\frac{\partial}{\partial\rho_u}\lp(\frac{1 + \sum_{j \neq 1}\Delta_{1j}\frac{\partial \lambda_j}{\partial\rho_1} - \sum_{j \neq 1}\lambda_j \rho_j}{\sqrt{1 + \sum_{j, k \colon j\ne k} \Delta_{jk} \lambda_j \lambda_k}}\rp) =
\frac{\sum_{j \neq 1}\Delta_{1j}\frac{\partial^2\lambda_j}{\partial\rho_1\partial\rho_u} - \frac{d\lambda_u}{\partial\rho_1} \rho_1 - \sum_{j \neq 1}\frac{\partial \lambda_j}{\partial\rho_u}\rho_j- \lambda_u}{\sqrt{1 + \sum_{j, k \colon j\ne k} \Delta_{jk} \lambda_j \lambda_k}} \\
&- 
\frac{1}{2} \frac{1 + \sum_{j \neq 1}\Delta_{1j}\frac{\partial \lambda_j}{\partial\rho_1} - \sum_{j \neq 1}\lambda_j \rho_j}{\lp(1 + \sum_{j, k \colon j\ne k} \Delta_{jk} \lambda_j \lambda_k\rp)^{3/2}}\lp(2 \sum_{j \neq u} \frac{d\lambda_u\lambda_j}{\partial\rho_u} \Delta_{uj} - 2 \sum_{j\neq u} \lambda_u\lambda_j \rho_j - \sum_{j\neq k \neq u} \frac{\partial \lambda_j\lambda_k}{\partial\rho_u}\Delta_{jk}\rp)
\end{align*}
From the first term, the only term that survives in the limit is 
$$
- \frac{d\lambda_u}{\partial\rho_1} \rho_1\mid_{\rho = \f\rho} = 
- \frac{d\lambda_u}{\partial\rho_1} \mid_{\rho=\f\rho}
$$
We turn to the second term that needs differentiation. This is similar to what we had previously, where the only term that matters is 
\begin{align*}
\frac{1}{2} \frac{\rho_1 + \sum_{j\ne i} \Delta_{1j}\lambda_j}{\lp(1 + \sum_{j, k \colon j\ne k} \Delta_{jk} \lambda_j \lambda_k\rp)^{3/2}}
\frac{\partial}{\partial\rho_u}\lp(2 \sum_{j \neq 1} \frac{d\lambda_1\lambda_j}{\partial\rho_1} \Delta_{1j} - 2 \sum_{j \neq 1} \lambda_1\lambda_j \rho_j - \sum_{j\neq k \neq i} \frac{\partial \lambda_j\lambda_k}{\partial\rho_1}\Delta_{jk}\rp)\\
= 
\frac{1}{2} \frac{\rho_1 + \sum_{j\ne i} \Delta_{1j}\lambda_j}{\lp(1 + \sum_{j, k \colon j\ne k} \Delta_{jk} \lambda_j \lambda_k\rp)^{3/2}} \times\\
\times\lp(2 \sum_{j \neq 1} \frac{\partial^2\lambda_1\lambda_j}{\partial\rho_1\partial\rho_u} \Delta_{1j} - 2  \frac{d\lambda_1\lambda_u}{\partial\rho_1} \rho_1 - 2\sum_{j \neq 1}\frac{d\lambda_1\lambda_j}{\partial\rho_u} \rho_j - 2\lambda_1\lambda_u - \sum_{j\neq k \neq 1} \frac{\partial^2\lambda_j\lambda_k}{\partial^2\rho_1\rho_u}\Delta_{jk} - 2\sum_{j \neq u,1}\frac{\partial \lambda_j\lambda_u}{\partial\rho_1}\rho_j\rp) 
\end{align*}
The only terms that survive in the second parenthesis are 
$$
- 2  \frac{\partial\lambda_1\lambda_u}{\partial\rho_1} \rho_1 
$$
In the limit, this is equal to
$$
-2\frac{\partial\lambda_u}{\partial\rho_1} \mid_{\rho = \f\rho}
$$
We see that this exactly cancels with what we got in the first term.
Hence,
$$
\frac{\partial^2 f}{\partial\rho_1 \partial\rho_u} \mid_{\rho = \f\rho} = 0
$$
Similarly, for $u \neq v \neq 1$ we get
$$
\frac{\partial^2 f}{\partial\rho_u d\rho_v} = 0
$$
By similar calculations we can easily show that the third derivatives will also be bounded in a neighborhood of $\f\rho$. 
By using the second order Taylor Theorem, we get that
\begin{align*}
\rho^{t+1}_1 &= 1 + ( \rho^t_1 - 1)\frac{\partial f}{\partial \rho_1}\mid_{\rho = \f\rho} +\frac{1}{2} \frac{\partial^2 f}{\partial^2\rho_1} \mid_{\rho = \f\rho} (1 - \rho^t_1)^2 + O(\epsilon^3)\\
&= \rho^t_1 - C|\rho_1^t - 1|^2 + O(\epsilon^3)
\end{align*}

\end{proof}

Using Lemma ~\ref{l:pushback-pop}, it is clear that there is a second order term that is pushing $\rho_1^t$ away from $1$. Notice that this implies that we do not have a strict saddle point.
However, the third order term depends on the distance of all the other coordinates from their fixpoint values. Hence, in order for the second order term to dominate the third order term, we need to prove that $|\rho_i^t - \rho_i^* \rho_1^*|$ is roughly the same size as $1 - \rho_1^t$ for all $i \neq 1$. Obviously, with a bad initialization, this might not happen. However, we can prove that after one iteration of the EM, this will always happen, no matter the starting point. 
This is the content of the following lemma.

\begin{lemma}\label{l:dominate}
Let $\rho^t$ be the current iteration of EM. Then, there exists an $\epsilon > 0$ such that if $\|\rho^t - \f\rho\|_\infty< \epsilon$,
then there exists an absolute constant $M>0$ such that for all $i \neq 1$
$$
|\rho_i^{t+1} - \rho_i^*\rho_1^*| \leq M(1 - \rho^{t+1}_1)
$$

\end{lemma}
\begin{proof}
Let us fix an $i \neq 1$ and define the functions

\begin{align*}
f_i(\rho) := \frac{\rho_i + \sum_{j\ne i} (\rho_i\rho_j - \rho^*_i\rho^*_j)\lambda_j}{\sqrt{1 + \sum_{j, k \colon j\ne k} (\rho_j \rho_k - \rho^*_j\rho^*_k) \lambda_j \lambda_k}}\\
f_1(\rho) := \frac{\rho_1 + \sum_{j\ne 1} (\rho_1\rho_j - \rho^*_1\rho^*_j)\lambda_j}{\sqrt{1 + \sum_{j, k \colon j\ne k} (\rho_j \rho_k - \rho^*_j\rho^*_k) \lambda_j \lambda_k}}\\
r_i(\rho) := \frac{f_i(\rho) - \rho_i^*\rho_1^*}{f_1(\rho) - 1}
\end{align*}
We will show that the limit 
$$
\lim_{\rho \to \f\rho} r_i(\rho)
$$
exists and is bounded. 
First, let's fix the coordinates $\rho_i$ for $i \neq 1$ and have $\rho_1 \to 1$. This means that $\lambda_1 \to 1$ and $\lambda_j \to 0$ for $j \neq 1$. Notice however that we do not have $\Delta_{ij} \to 0$, as the values of the other coordinates are not in the fixpoint yet. However, we will see that this will not matter for the computation of the limit. We have by L'Hospital's rule
\begin{align*}
\lim_{\rho_1 \to 1} r_i(\rho) &= 
\lim_{\rho_1 \to 1} \frac{\frac{\partial f_i(\rho)}{\partial\rho_1}}{\frac{\partial f_1(\rho)}{\partial\rho_1}}\\
&= \sum_{j\neq i} \frac{\partial \lambda_j}{\partial\rho_1} \Delta_{ij} - \lambda_1 \rho_i - \sum_{j\neq 1} \frac{\partial \lambda_j}{\partial\rho_1} \Delta_{1j}
\end{align*}
Now, we take the limit with respect to the remaining variables, which gives 
\begin{align*}
\lim_{\rho\to \f\rho} r_i(\rho) &= 
\sum_{j\neq i} \frac{\partial \lambda_j}{\partial\rho_1}\mid_{\rho = \f\rho} \rho_i^*\rho_j^*(1 - (\rho_1^*)^2) -  \rho_i^*\rho_1^*
\end{align*}
This is a finite quantity that depends only on the true parameters, hence it is a constant. Since the function $r_i(\rho)$ is continuous, we get that there exists an $\epsilon >0$ such that if $\|\rho - \f\rho\|_\infty < \epsilon$ then
$$
|r_i(\rho)| \leq M 
$$
for some $M>0$. Substituting $\rho = \rho^t$ gives
$$
\lp|\frac{\rho^{t+1}_i - \rho^*_i\rho^*_1}{\rho^{t+1}_1 - 1}\rp| \leq M
$$
which is what we wanted to prove.

\end{proof}

With Lemmas ~\ref{l:pushback-pop} and ~\ref{l:dominate} in our hands, we can
proceed to prove Theorem~\ref{t:divergence}. The idea is the following: if in the starting point $\rho^0$ the distance of $\rho^0_1$ from $1$ is much less that the distances of the other coordinates from the fixpoint values, then in the following iteration they will all come closer to the fixpoint, so that all the distances are comparable. This means that even if the algorithm starts from a bad angle, it will quickly correct itself to a good angle in the next iteration. After that,
in the following iterations, the second order term computed in Lemma~\ref{l:pushback-pop} will dominate the third order term and repel $\rho_1$ away from $1$, until it gets out of the ball. 

\begin{proof}[Proof of Theorem ~\ref{t:divergence}]
Let $\epsilon_1$ be the value of $\epsilon$ from Lemma~\ref{l:pushback-pop} and $\epsilon_2$ be the value from Lemma~\ref{l:dominate}. Also, let 
 $\epsilon_3$ be small enough so that 
$$
C\epsilon^2/2 > KM\epsilon^3
$$
for all $\epsilon < \epsilon_3$. We choose $\epsilon_0 = \min(\epsilon_1,\epsilon_2,\epsilon_3)$. 
Let $\rho^0$ be a starting point of EM with $\|\rho^0 - \f\rho\|_2 \leq \epsilon_0$ and $\rho_i^0  < 1$ for all $i$.

Let $t \geq 1$. If $\|\rho^{t'} - \f \rho\|_2 > \epsilon_0$ for some $t' \leq t$, then we have nothing to prove. 
Suppose $\|\rho^{t'} - \f \rho\|_2 < \epsilon_0$ for all $t'\leq t$. Then,
by applying Lemma~\ref{l:dominate} 
we get that
$$
|\rho_i^t - \rho_i^*\rho_1^*| \leq M (1 - \rho^t_1)
$$
for all $i \neq 1$. 
This means that
by applying Lemma~\ref{l:pushback-pop} and by the particular choice of $\epsilon_0$, we have that
\begin{align*}
\rho_1^{t+1} &\leq \rho_1^t - C(1 - \rho_1^t)^2 + KM(1 - \rho^t_1)^3
\leq \rho_1^t - \frac{C}{2}(1 - \rho_1^t)^2
\end{align*}
This implies that 
$$
1 - \rho_1^{t+1} \geq 1 - \rho_1^t + \frac{C}{2}(1 - \rho_1^t)^2
$$
This holds for all times up to $t$. If we set $\alpha_t = 1 - \rho_1^{t}$, this implies the
recursive relation
$$
\alpha^{t+1} \geq \alpha^t + \frac{C}{2} (\alpha^t)^2
$$
It is obvious that $\alpha^t$ is an increasing function of $t$, meaning that $\alpha^t \geq \alpha^0$. This means that $\alpha^{t+1} \geq \alpha^t + \frac{C}{2} (\alpha^0)^2$. Since $\alpha^0 = 1 - \rho^0_1 > 0$, we have that $\lim_{t\to \infty}\alpha^t = \infty$. This implies that there exists a finite time $T$ such that 
$$
1 - \rho_1^t = \alpha^t > \epsilon_0
$$
This implies that $\|\rho^t - \f\rho\|_2 > \epsilon_0$ and the proof is complete. 

\end{proof}

\subsection{Deffered proofs from Section~\ref{sec:onelatent}}\label{sec:app-proofs}

\begin{proof}[Proof of Lemma~\ref{lem:update}]
Using Lemma~\ref{lem:covariance-conserve} and Claim~\ref{cla:lambda}, we derive that
\begin{equation}\label{eq:first-cov-exp}
\begin{aligned}
\E_{\mu^{t+1}}[x_iy]
&= \E_{\mu^{t,*}}[x_iy]
= \E_{\mu^*_{x}}\lp[ \E_{\mu^t_{y\mid x}}[x_i y] \rp]
= \E_{\mu^*_{x}}\lp[ x_i \E_{\mu^t_{y\mid x}}[y] \rp]
= \E_{\mu^*_{x}}\lp[x_i \sum_{j=1}^n \frac{\sigma^t_y}{\sigma^t_j} \lambda^t_j x_j \rp]\\
&= \sum_{j=1}^n \frac{\sigma^t_y}{\sigma^t_j} \lambda^t_j \E_{\mu^*_{x}}\lp[x_i x_j \rp]
= \sigma^t_y \sigma^t_i \lambda^t_i + \sum_{j \ne i} \sigma^t_y \sigma^t_i \lambda^t_j \rho^*_i \rho^*_j = \sigma^t_y \sigma^t_i \lp(\lambda^t_i + \sum_{j \ne i} \lambda^t_j \rho^*_i \rho^*_j \rp)\enspace.
\end{aligned}
\end{equation}
This concludes the first expression for $\E_{\mu^{t+1}}[x_i y]$. Next, we would like to show that the second expression equals the first expression. First, recall that in \eqref{eq:first-cov-exp} we have shown that
\[
\E_{\mu^*_{x}}\lp[ \E_{\mu^t_{y\mid x}}[x_i y] \rp] = 
\sigma^t_y \sigma^t_i \lp(\lambda^t_i + \sum_{j \ne i} \lambda^t_j \rho^*_i \rho^*_j \rp)\enspace.
\]
If we substitute $\mu^t$ instead of $\mu^*$, we derive that
\[
\E_{\mu^t_{x}}\lp[ \E_{\mu^t_{y\mid x}}[x_i y] \rp] = 
\sigma^t_y \sigma^t_i \lp(\lambda^t_i + \sum_{j \ne i} \lambda^t_j \rho^t_i \rho^t_j \rp)\enspace.
\]
Since  the right hand side equals
\[
\E_{\mu^t}[x_i y]
= \sigma^t_i \sigma^t_y \rho^t_i,
\]
we derive that
\begin{equation}\label{eq:rho-t-equivalent}
\sigma^t_i \sigma^t_y\rho^t_i -  \sigma^t_i \sigma^t_y\lp(\lambda^t_i + \sum_{j \ne i} \lambda^t_j \rho^t_i \rho^t_j \rp)
= 0\enspace.
\end{equation}
Adding the left hand side of \eqref{eq:rho-t-equivalent} to 
\eqref{eq:first-cov-exp}, we derive that
\[
\E_{\mu^{t+1}}[x_iy]
= \sigma^t_i \sigma^t_y \lp( 
\rho^t_i + \sum_{j\ne i} (\rho^*_i \rho^*_j - \rho^t_i\rho^t_j) \lambda^t_j
\rp),
\]
which is exactly the second expression for $\E_{\mu^{t+1}}[x_iy]$.
Next, we compute the variance for $y$, again, using Lemma~\ref{lem:covariance-conserve} and Claim~\ref{cla:lambda}:
\begin{align} \label{eq:var-first-exp}
\E_{\mu^{t+1}}[y^2]
&= \E_{\mu^{t,*}}[y^2]
= \E_{\mu^*_{x}}\lp[ \E_{\mu^t_{y\mid x}}[y^2] \rp]
= \E_{\mu^*_{x}}\lp[\lp(\sum_{j=1}^n \frac{\sigma^t_y}{\sigma^t_j} \lambda^t_j x_j\rp)^2 \rp]\\
&= \sum_{i=1}^n\sum_{j=1}^n
\frac{\sigma^t_y}{\sigma^t_i} \lambda^t_i \frac{\sigma^t_y}{\sigma^t_j} \lambda^t_j \E_{\mu^*_{x}}\lp[x_i x_j \rp]
= (\sigma^t_y)^2\lp(\sum_{i=1}^n (\lambda^t_i)^2 + \sum_{i\ne j \in \{1,\dots,n\}}
\lambda^t_i \lambda^t_j \rho^*_i \rho^*_j\rp)\enspace.\notag
\end{align}
This derives the first expression for the variance of $y$. We derive the second expression using a similar logic as in the calculation in the second expression for $\E_{\mu^{t+1}}[x_i y]$. First, substituting $\rho^t$ instead of $\rho^*$ in \eqref{eq:var-first-exp}, we derive that
\[
(\sigma_y^t)^2 = \E_{\mu^t}[y^2]
= \E_{\mu^t_{x}}\lp[ \E_{\mu^t_{y\mid x}}[y^2] \rp]
= (\sigma^t_y)^2\lp(\sum_{i=1}^n (\lambda^t_i)^2 + \sum_{i\ne j \in \{1,\dots,n\}}
\lambda^t_i \lambda^t_j \rho^t_i \rho^t_j\rp)\enspace.
\]
This implies that
\[
(\sigma_y^t)^2 \lp( 1 - \sum_{i=1}^n (\lambda^t_i)^2 - \sum_{i\ne j \in \{1,\dots,n\}}
\lambda^t_i \lambda^t_j \rho^t_i \rho^t_j\rp)
= 0.
\]
Adding this to the expression in \eqref{eq:var-first-exp}, this yields the second expression for the covariance of $y$, as required. 

Lastly, the expression of $\rho^{t+1}_i$ is derived, by definition of $\rho_i$,
\[
\rho^{t+1}_{i}
= \frac{\E_{\mu^{t+1}}[x_iy]}{\sqrt{\Var_{\mu^{t+1}}[x_i]\Var_{\mu^{t+1}}[y]}},
\]
and by substituting the expressions for the covariance of $x_i$ and $y$, the variance of $y$, and using Lemma~\ref{lem:covariance-conserve} which argues that $\Var_{\mu^{t+1}}[x_i] = \sigma_i^2$.
\end{proof}

\begin{proof}[Proof of Lemma~\ref{lem:convergence-to-stat}]
The convergence of $\v\rho^t$ to some $\v\rho\in \mathcal{S}$ follows directly from \cite{Wu83}. 
For the second part of the lemma, assume that $\v\rho\in\mathcal{S}$ and by definition of a stationary point, $\v\rho^{t+1}=\v\rho^t$. Further, by Lemma~\ref{lem:update}, we know that $\sigma_{x_i}^{t+1} = \sigma_{x_i}^* = \sigma_{x_i}^t$ for all $t$. This implies that $\mu^t_{x_1\cdots x_n} = \mu^{t+1}_{x_1\cdots x_n}$. In particular, $\KL(\mu^*\|\mu^t) = \KL(\mu^*\|\mu^{t+1})$. Yet, \cite{Wu83} implies that if $\mu^{t+1}\ne \mu^t$ then $\KL(\mu^*\|\mu^{t+1}) < \KL(\mu^*\|\mu^t)$. Hence, $\mu^{t+1}=\mu^t$ as required.
\end{proof}

\begin{proof}[Proof of Lemma~\ref{lem:non-singular}]
We will analyze the transpose of the matrix in \eqref{eq:nonsingular-mat}, assuming that $J_{ij} = u_j$ for any $i \ne j$. Indeed, the original matrix is singular if and only if its transpose is.

We would like to show that there is no nontrivial solution $a$ for $Ja=0$. We have
\[
(Ja)_i = \sum_{j \ne i} a_j u_j  + a_i (\sum_{j \ne i} u_j)
= \sum_j a_j u_j + a_i(\sum_j u_j - 2u_i).
\]
Substitute $s = \sum_j u_j$ and $u = \sum_j a_j u_j$, we have
\[
(Ja)_i = u + a_i (s-2u_i).
\]
To solve the system of equations $Ja=0$, we assume without loss of generality that $u_1 = \max_j u_j$
and divide into cases according to $u_1$. First, assume that $u_1 < s/2$. Then,
\begin{equation}\label{eq:ai-equation-raw}
(Ja)_i = u + a_i(s-2u_i) = 0
\end{equation}
implies
\[
a_i = \frac{-u}{s-2u_i}.
\]
The equation $\sum_i u_i a_i = u$ implies
\begin{equation}\label{eq:with-u}
\sum_i u_i a_i = \sum_i \frac{-u u_i}{s-2u_i} = u.
\end{equation}
Notice that $u \ne 0$. Assume towards contradiction that $u = 0$. Then, we have
\[
0 = a_i (s- 2 u_i).
\]
Dividing by $s- 2 u_i$ we get $a_i = 0$ for all $i$. Since we look for nontrivial solutions $a$ to $Ja=0$, we assume that $a\ne 0$ which implies that $u \ne 0$.
Dividing \eqref{eq:with-u} by $u$ we get
\begin{equation}\label{eq:theta1small}
\sum_i \frac{-u_i}{s-2u_i} = 1.
\end{equation}
By the assumption $u_1 < s/2$ we have $s-2u_i \ge s-2u_1 > 0$. In particular, the left hand side of \eqref{eq:theta1small} is negative while the right hand side is positive, hence there is no solution! Next, assume that $u_1 > s/2$. In this case, we have $u_i < s/2$ for all $i\ge 2$, since $\sum_i u_i = s$. Here, \eqref{eq:theta1small} is still valid, and is equivalent to
\[
\frac{u_1}{2u_1-s} - 1 = \sum_{i>1} \frac{u_i}{s-2u_i},
\]
which is equivalent to
\[
\frac{s - u_1}{2u_1-s} = \sum_{i>1} \frac{u_i}{s-2u_i}\enspace.
\]
Here, we will show that the left hand side is strictly greater than the right hand side, which implies that there is no equality. First, notice that for any $i \ge 2$, $2u_1 - s < s - 2 u_i$. Indeed, this follows from $s = \sum_j u_j > u_1 + u_i$, since all $u_j$ are positive and $n\ge 3$. We derive that
\[
\frac{s - u_1}{2u_1-s} =
\sum_{i \ge 2} \frac{u_i}{2u_1-s} >
\sum_{i \ge 2} \frac{u_i}{s - 2u_i}\enspace,
\]
which arrives at a contradiction. Lastly, assume that $u_1 = s/2$. Since $s = \sum_{i=1}^n u_i$ and $n \ge 3$, then for all $i\ge 2$ we have $u_i < s/2$. Recall \eqref{eq:ai-equation-raw} which states that 
\[
0 = (Ja)_i = u + a_i(s-2u_i).
\]
Applying with $i=1$ we obtain
\[
u = 0.
\]
For all $i \ge 2$ we have
\[
a_i(s - 2 u_i) = 0
\]
which implies that $a_i = 0$ for $i \ge 2$. By definition of $u$ and by the computations above,
\[
0 = u = \sum_i a_i u_i = a_1 u_1.
\]
Since $u_1 > 0$ we have
\[
a_1 = 0.
\]
We conclude that $a_i = 0$ for all $i$ which implies that there is no nontrivial solution for $Ja=0$, as required.
\end{proof}
\section{Proofs of Lemmas~\ref{lem:stationaries} and
~\ref{lem:converge_int}}\label{sec:app_final}
\begin{proof}[Proof of Lemma~\ref{lem:stationaries}]
By Lemma~\ref{lem:stationary-at-0} we know that the only stationary point with some $\rho_i = 0$ is $(0,\ldots,0)$. By Lemma~\ref{lem:stationary-1} we know that the only stationary points with some $\rho_i = 1$ are the $g^i$ for $i \in [n]$ (see the statement of Lemma~\ref{lem:stationary-1} for definitions).
Finally, Lemma~\ref{lem:int} implies that
the only stationary point where $\rho_i \in (0,1)$ for all $i$ is $\rho^*$. All these facts imply the statement of the Lemma.

\end{proof}

\begin{proof}[Proof of Lemma~\ref{lem:converge_int}]
By Lemma~\ref{lem:not-to-0} we know that $\lim_{t\to \infty} \rho^t \neq \overline{0}$.
By the discussion after the statement of Theorem~\ref{t:divergence} we know that $\lim_{t \to \infty} \rho^t \neq g^i$ for all $i$. Hence, we should have for all $i$ $\lim_{t \to \infty} \rho_i^t \in (0,1)$. 

\end{proof}

\section{Upper and Lower bounds for empirical EM iteration} \label{sec:bounded-iter}
This Section is devoted to proving that the finite sample iterates of EM are upper and lower bounded with high probability. This fact will be used repeatedly in the proof of Theorem~\ref{thm:general-finite}, which will presented in later Sections.  We denote by $\mu_\rho$ the density of the distribution on leaves with correlations $\rho$. The main Theorem that we prove in this Section is the following.

\begin{restatable}{theorem}{bounded}\label{t:bounded}
Suppose that $\min(\min_i \rho^0_i, \min_i \rho^*_i) \geq \alpha, 1 - \max(\max_i \rho^0_i,\max_i \rho^*_i)  \geq \beta$, where $\alpha,\beta \in (0,1)$ are constants. Suppose we have access to $m$ i.i.d. samples $x^{(1)}, \ldots , x^{(m)}$ from the distribution $\mu_x^*$. Let $\rho_i^t$ be the correlations produced by the EM iteration \eqref{eq:update-rho-finite} run using these samples. Suppose $m = \Omega(\log(n/\delta)/\min(\alpha,\beta)^2)$.  
Then, there are constants $C(\alpha,\beta),C'(\alpha,\beta) > 0$  such that with probability at least $1- \delta$,  for all $t > 0$, $i$ and $k = 12$:
$$
\frac{C(\alpha,\beta)}{n^{k+2}} \leq \rho_i^t \leq 1-\frac{C'(\alpha,\beta)}{n^{4k+9}}
$$

\end{restatable}
Theorem~\ref{t:bounded} says that if we initialize EM at a constant distance away from the optimum, it will always remain in a bounded distance within that optimum.
Theorem~\ref{t:bounded} essentially follows directly by combining Lemmas~\ref{l:allbounded} and~\ref{l:lower_bound}.
We will break the proof of these two Lemmas into multiple lemmas over the next few Sections. The final proof of Lemma~\ref{l:allbounded} is in Section~\ref{s:upper_bound} and the one for Lemma~\ref{l:lower_bound} is in Section~\ref{s:lower_bound}. 
In the next Section, we derive the precise formula for the finite sample EM iterate. 
\subsection{Finite sample update rule for EM}

To describe the EM update, we have the following analogue of Lemma~\ref{lem:covariance-conserve} for the sample EM:

\begin{lemma}\label{lem:finite-conserve}
Denote by $\hat{\mu}_x$ the uniform distribution over the $m$ samples. Let $\mu^{t,'}$ denote the joint distribution over $x_1\cdots x_n,y$ such that
\[
\Pr_{\mu^{t,'}}[x_1,\cdots, x_n,y] 
= \Pr_{\hat\mu_x}[x_1,\cdots, x_n]\Pr_{\mu^t}[y\mid x_1,\dots, x_n].
\]
Then, for any $i$, we have that
\[
\E_{\mu^{t+1}}[x_i y] = \E_{\mu^{t,'}}[x_i y], \ \Var_{\mu^{t,'}}[x_i] = \Var_{\mu^{t+1}}[x_i],\ 
\Var_{\mu^{t,'}}[y] = \Var_{\mu^{t+1}}[y]\enspace.
\]
\end{lemma}
The proof follows the same lines as the proof of Lemma~\ref{lem:covariance-conserve}, where the only difference is that here we are considering the sample-EM. As a consequence, we have the following analogue of Lemma~\ref{lem:update}:

\begin{lemma}\label{lem:update-sample}
For any $i$, denote 
\[
\hat{\sigma}_i = \sqrt{\frac{1}{m} \sum_{k=1}^m \lp(x_i^{(k)}\rp)^2}
\]
For any $i \ne j$, denote
\[
\hat{\alpha}_{ij} = 
\frac{1}{\hat{\sigma}_i \hat{\sigma}_j}
\frac{1}{m} \sum_{k=1}^m x_i^{(k)} x_j^{(k)}\enspace.
\]
Denote by $\Delta^t_{ij} = \hat{\alpha}_{ij} -  \rho_i\rho_j$. Then, for any $t > 0$ and any $i$:
\[
\sigma^t_{i} = \hat{\sigma}_{i},
\]
\[
\E_{\mu^{t+1}}[x_i y] =
\sigma^t_i \sigma^t_y \l(\lambda_i^t + \sum_{j\ne i} \hat{\alpha}_{ij} \lambda_j^t \r)
= \sigma^t_i \sigma^t_y \l(\rho_i^t + \sum_{j\ne i} \Delta_{ij}^t \lambda_j^t \r)
\enspace,
\]
\begin{equation*} 
\E_{\mu^{t+1}}[y^2]
= (\sigma_y^{t+1})^2
= (\sigma^t_y)^2\lp(\sum_{i=1}^n (\lambda^t_i)^2 + \sum_{i\ne j \in \{1,\dots,n\}}
\lambda^t_i \lambda^t_j \hat{\alpha}_{ij}\rp)
= (\sigma_y^t)^2\l( 
1 + \sum_{j\ne k} \Delta_{ij}^t \lambda_j^t \lambda_k^t
\r)
\end{equation*}
\begin{equation}\label{eq:update-rho-finite}
\rho_i^{t+1} = 
\frac{\lambda_i^t + \sum_{j\ne i} \lambda_j^t \hat{\alpha}_{ij}}{\sqrt{\sum_{i=1}^n (\lambda^t_i)^2 + \sum_{i\ne j \in \{1,\dots,n\}}
\lambda^t_i \lambda^t_j \hat{\alpha}_{ij}}}
= \frac{\rho_i^t + \sum_{j\ne i} \Delta_{ij} \lambda_j^t}{\sqrt{1 + \sum_{j\ne k} \Delta_{ij} \lambda_j^t \lambda_k^t}}
\enspace.
\end{equation}
\end{lemma}
\begin{proof}
The proof is similar to the proof of Lemma~\ref{lem:update}. First, $\hat{\sigma}_i = \sigma^t_i$ follows directly from Lemma~\ref{lem:covariance-conserve}. Next,
Using Lemma~\ref{lem:finite-conserve} and Claim~\ref{cla:lambda}, we derive that
\begin{equation*}
\begin{aligned}
\E_{\mu^{t+1}}[x_iy]
&= \E_{\mu^{t,'}}[x_iy]
= \E_{\hat\mu_{xo}}\lp[ \E_{\mu^t_{y\mid x}}[x_i y] \rp]
= \E_{\hat\mu_{x}}\lp[ x_i \E_{\mu^t_{y\mid x}}[y] \rp]
= \E_{\hat\mu_{x}}\lp[x_i \sum_{j=1}^n \frac{\sigma^t_y}{\sigma^t_j} \lambda^t_j x_j \rp]\\
&= \sum_{j=1}^n \frac{\sigma^t_y}{\sigma^t_j} \lambda^t_j \E_{\hat\mu_{x}}\lp[x_i x_j \rp]
= \sigma^t_y \sigma^t_i \lambda^t_i + \sum_{j \ne i} \sigma^t_y \sigma^t_i \lambda^t_j \hat{\alpha}_{ij} = \sigma^t_y \sigma^t_i \lp(\lambda^t_i + \sum_{j \ne i} \lambda^t_j \hat{\alpha}_{ij} \rp)\enspace.
\end{aligned}
\end{equation*}
This concludes the first expression for $\E_{\mu^{t+1}}[x_i y]$. The second equality is derived similarly to the proof of Lemma~\ref{lem:update}.
Next, we compute the variance for $y$, again, using Lemma~\ref{lem:finite-conserve} and Claim~\ref{cla:lambda}:
\begin{align*}
\E_{\mu^{t+1}}[y^2]
&= \E_{\mu^{t,'}}[y^2]
= \E_{\hat{\mu}_{x}}\lp[ \E_{\mu^t_{y\mid x}}[y^2] \rp]
= \E_{\hat{\mu}_{x}}\lp[\lp(\sum_{j=1}^n \frac{\sigma^t_y}{\sigma^t_j} \lambda^t_j x_j\rp)^2 \rp]\\
&= \sum_{i=1}^n\sum_{j=1}^n
\frac{\sigma^t_y}{\sigma^t_i} \lambda^t_i \frac{\sigma^t_y}{\sigma^t_j} \lambda^t_j \E_{\hat{\mu}_{x}}\lp[x_i x_j \rp]
= (\sigma^t_y)^2\lp(\sum_{i=1}^n (\lambda^t_i)^2 + \sum_{i\ne j \in \{1,\dots,n\}}
\lambda^t_i \lambda^t_j \hat{\alpha}_{ij}\rp)\enspace.\notag
\end{align*}
This derives the first expression for the variance of $y$. The second expression is derived similarly to its analogue in Lemma~\ref{lem:update}.

Lastly, the expression of $\rho^{t+1}_i$ is derived, by definition of $\rho_i$,
\[
\rho^{t+1}_{i}
= \frac{\E_{\mu^{t+1}}[x_iy]}{\sqrt{\Var_{\mu^{t+1}}[x_i]\Var_{\mu^{t+1}}[y]}},
\]
and by substituting the expressions for the covariance of $x_i$ and $y$, the variance of $y$, and using that $\Var_{\mu^{t+1}}[x_i] = \hat\sigma_i^2$.
\end{proof}

From Lemma~\ref{lem:update-sample}, it follows that the convergence rate of the EM is not affected by the variance of the nodes. In particular, the correlation parameters $\rho^t_i$ are independent of these. Therefore, we will assume for simplicity that $\sigma^*_i = 1$ for all $i$.

\subsection{Two correlations cannot be close to $1$}

We will first be concerned with establishing that all $\rho_i$ remain bounded away from $1$ in all iterations. 
Our first Lemma contributes in this direction. It says that no $2$ coordinates $i,j$ can have $\rho_i,\rho_j$ simultaneously close to $1$.

\begin{lemma}\label{l:tworho}
Suppose that $\min(\min_i \rho^0_i, \min_i \rho^*_i) \geq c_1, 1 - \max(\max_i \rho^0_i,\max_i \rho^*_i)  \geq c_2$, where $c_1,c_2 \in (0,1)$ are constants. Suppose we have access to $m$ i.i.d. samples $x^{(1)}, \ldots , x^{(m)}$ from the distribution $\mu_x^*$. Let $\rho_i^t$ be the correlations produced by the EM iteration run using these samples. Suppose $m = \Omega(\log(n/\delta))$. Then, with probability at least $1-\delta$, there exists a constant $c = c(c_1,c_2)$ such that for all $t$, there exists at most one $i \in [n]$ such that 
$\rho^t_i > 1 - c/(n^2 ).$
\end{lemma}
\begin{proof}
We use the fact that EM always improves the value of the likelihood function. We consider the likelihood of a given observation on the leaves as a function of the standard deviations at the leaves $\sigma_{x_i}$ and 
the correlations $\rho_i$ between $x_i,y$. Let $\sigma_x$ be the vector of standard deviations and $\rho$ the vector of correlations. 
The empirical likelihood function can be written as
$$
L(\rho,\sigma_x;x^{(1)},\ldots,x^{(k)}) =  \frac{1}{m}\sum_{k=1}^m \log \mu_{\rho,\sigma_x}( x^{(k)}) 
$$
From now on, we will omit the dependence on the samples $x^{(1)},\ldots,x^{(k)}$ whenever it is implied. 
We have that
\begin{align}\label{eq:likelihood}
L(\rho,\sigma_x) &= \frac{1}{m}\sum_{k=1}^m \log \mu_{\rho,\sigma_x}( x^{(k)})  \nonumber \\
&= 
-\frac{1}{2} \log \lp(2\pi \lp|\Sigma_{\rho,\sigma_x}\rp|\rp)- \frac{1}{2m}\sum_{k=1}^m (x^{(k)})^\top \Sigma_{\rho,\sigma_x}^{-1} x^{(k)}
\nonumber\\
&= 
-\frac{1}{2} \log \lp(2\pi \lp|\Sigma_{\rho,\sigma_x}\rp|\rp) - \frac{1}{2} tr\lp(\Sigma_{\rho,\sigma_x}^{-1} \hat{\Sigma}\rp)
\end{align}
where $\hat{\Sigma}$ is the empirical covariance matrix of the samples.
Recall the closed form expression for KL of Gaussians.

$$
KL(\mathcal{N}(0,\Sigma_1)||\mathcal{N}(0,\Sigma_2)) = 
\frac{1}{2} \lp(\log \frac{|\Sigma_{2}|}{|\Sigma_{1}|}
- n + tr(\Sigma_{2}^{-1} \Sigma_{1})\rp)
$$
From this, we conclude that
$$
L(\rho,\sigma_x) = -\frac{n}{2} - \frac{1}{2} \log \lp( 2\pi\lp|\hat{\Sigma}\rp|\rp) - KL(\mathcal{N}(0,\hat{\Sigma})||\mathcal{N}(0,\Sigma_{\rho,\sigma_x}))
$$
Let us start by lower bounding $L(\rho^0,\sigma_x^0)$.
We are going to upper bound $KL(\mathcal{N}(0,\hat{\Sigma})||\mathcal{N}(0,\Sigma_{\rho^0,\sigma_x^0})))$. 
The matrix $\Sigma_{\rho,\sigma_x}$ can be written as a sum of a rank $1$ matrix and a diagonal matrix. We will use this to evaluate determinants and inverses of these matrices.
Specifically, for any $\rho$, if $\sigma_x$ is the all $1$'s vector, we have that
\begin{equation}\label{eq:sigma-formula}
\Sigma_{\rho,\sigma_x} =\begin{pmatrix}
1 - \rho_1^2 &0 &\ldots &0\\
0 & 1 - \rho_2^2 &\ldots &0 \\
\vdots &\vdots & &\vdots \\
0 &0 &\ldots &1 - \rho_n^2
\end{pmatrix} 
+ 
\rho \rho^\top 
\end{equation}
We will denote this diagonal matrix as $diag(1 - \rho^2)$ for convenience. From now on, we will omit the dependence on $\sigma_x$ if we have unit variances (which happens for $\rho^0$ and $\rho^*$ for example). 
Using the Sherman-Morrison formula, we get that
\begin{equation}\label{eq:sigma-inv-formula}
\Sigma_{\rho}^{-1} = diag(1 - \rho^2)^{-1} - \frac{diag(1 - \rho^2)^{-1} \rho \rho^\top diag(1 - \rho^2)^{-1}}{1 + \rho^\top diag(1 - \rho^2)^{-1}\rho}
\end{equation}
We will first try to bound $tr(\Sigma_{\rho^0}^{-1} \Sigma_{\rho^*})$ and then show that it is close to $tr(\Sigma_{\rho^0}^{-1} \hat{\Sigma})$
Thus, after some algebraic manipulations, we have that
\begin{align*}
tr(\Sigma_{\rho^0}^{-1} \Sigma_{\rho^*}) &= 
n + \frac{\lp(\sum_i \frac{(\rho_i^0)^2}{1 - (\rho_i^0)^2}\rp)^2 - \lp(\sum_i \frac{\rho_i^0\rho_i^*}{1 - (\rho_i^0)^2}\rp)^2 + \sum_i \frac{(\rho_i^0)^2 (\rho_i^*)^2}{(1 - (\rho_i^0)^2)^2} - \sum_i \frac{(\rho_i^0)^4}{(1 - (\rho_i^0)^2)^2}}{1 + \sum_i \frac{(\rho_i^0)^2}{1 - (\rho_i^0)^2}}\\
&= n + 
\frac{\lp(\sum_i \frac{(\rho_i^0)(\rho_i^0 - \rho_i^*)}{1 - (\rho_i^0)^2}\rp)\lp(\sum_i \frac{(\rho_i^0)(\rho_i^0 + \rho_i^*)}{1 - (\rho_i^0)^2}\rp) + \sum_i \frac{(\rho_i^0)^2 ((\rho_i^*)^2 - (\rho_i^0)^2)}{(1 - (\rho_i^0)^2)^2}}{1 + \sum_i \frac{(\rho_i^0)^2}{1 - (\rho_i^0)^2}}
\end{align*}
Now, notice that the function $x \mapsto x^2/(1 - x^2)$ is increasing. This implies, using the assumption of the Lemma, that 
$$
\sum_i \frac{(\rho_i^0)^2}{1 - (\rho_i^0)^2} \geq n 
\frac{c_1^2}{1 - c_1^2}
$$
Also, since the function $x \mapsto x/(1 - x^2) $ is increasing, this implies that
$$
\sum_i \frac{(\rho_i^0)(\rho_i^0 + \rho_i^*)}{1 - (\rho_i^0)^2} \leq 
2 n
\frac{1 - c_2}{1 - (1 - c_2)^2}
$$
and 
$$
\lp|\sum_i \frac{(\rho_i^0)(\rho_i^0 - \rho_i^*)}{1 - (\rho_i^0)^2}\rp| \leq 
2 n \frac{1 - c_2}{1 - (1 - c_2)^2}
$$
Also, for all $i$ we have
$$
|(\rho_i^*)^2 - (\rho_i^0)^2| = |\rho_i^* - \rho_i^0||\rho_i^* + \rho_i^0| \leq 2 
$$
By monotonicity again, this implies 
$$
\lp|\sum_i \frac{(\rho_i^0)^2 ((\rho_i^*)^2 - (\rho_i^0)^2)}{(1 - (\rho_i^0)^2)^2}\rp| \leq 2 \sum_i \frac{(\rho_i^0)^2 }{(1 - (\rho_i^0)^2)^2} \leq 2 n \frac{(1 - c_2)^2}{(1 - (1 - c_2)^2)^2}
$$

Putting all these facts together, we conclude that there exists a constant $C = C(c_1,c_2)$, such that 

\begin{equation}\label{eq:trace}
\lp|\frac{\lp(\sum_i \frac{(\rho_i^0)(\rho_i^0 - \rho_i^*)}{1 - (\rho_i^0)^2}\rp)\lp(\sum_i \frac{(\rho_i^0)(\rho_i^0 + \rho_i^*)}{1 - (\rho_i^0)^2}\rp) + \sum_i \frac{(\rho_i^0)^2 ((\rho_i^*)^2 - (\rho_i^0)^2)}{(1 - (\rho_i^0)^2)^2}}{1 + \sum_i \frac{(\rho_i^0)^2}{1 - (\rho_i^0)^2}}\rp| \leq C n
\end{equation}
Now, let's focus on $tr(\Sigma_{\rho^0}^{-1} \hat{\Sigma})$.
By standard Chernoff bounds, we have that with probability at least $1-\delta$, for all $i,j$ $$
\lp|(\Sigma_{\rho^*})_{ij} - \hat{\Sigma}_{ij}\rp| = O\lp(\sqrt{\frac{\log(n/\delta)}{m}}\rp) := \eta
$$
By assumption on the number of samples $m$, $\eta = O(1)$. 
We have
$$
tr\lp(\Sigma_{\rho^0}^{-1} (\hat{\Sigma} - \Sigma_{\rho^*})\rp) = \sum_{i,j} (\Sigma_{\rho^0}^{-1})_{ij} \lp((\Sigma_{\rho^*})_{ij} - \hat{\Sigma}_{ij}\rp)
$$
Hence, all we have to do is bound the entries of $\Sigma_{\rho^0}^{-1}$. 
For this, we can use the Sherman-Morrison formula, which we also used earlier. Let's start with a non-diagonal element $i\neq j$. 
Then, the formula gives
$$
\lp|(\Sigma_{\rho^0})_{ij}\rp| =  \frac{\frac{\rho^0_i}{1 - (\rho^0_i)^2}\frac{\rho^0_j}{1 - (\rho^0_j)^2}}{1 + \sum_i \frac{(\rho_i^0)^2}{1 - (\rho_i^0)^2}} \leq \frac{1}{c_2^2} 
$$
For $i=j$, we have
$$
\lp|(\Sigma_{\rho^0})_{ii}\rp| = \lp|\frac{1}{1 - (\rho^0_i)^2} -  \frac{\lp(\frac{\rho^0_i}{1 - (\rho^0_i)^2}\rp)^2}{1 + \sum_i \frac{(\rho_i^0)^2}{1 - (\rho_i^0)^2}} \leq \frac{1}{c_2^2} \rp| \leq \frac{1}{c_2} + \frac{1}{c_2^2}
$$
Hence, with probability at least $1-\delta$
\begin{equation}\label{eq:trace_bound}
\lp|tr\lp(\Sigma_{\rho^0}^{-1} (\hat{\Sigma} - \Sigma_{\rho^*})\rp)\rp| \leq Cn^2\eta  \leq C'n^2
\end{equation}
where $C$ is some constant. 
Now, let's calculate the determinant. 
We use the matrix determinant lemma to compute $|\Sigma_\rho|$. 

$$
|\Sigma_\rho| = (1 + \rho^\top diag(1 - \rho^2)^{-1} \rho) |diag(1- \rho^2)| = (1 + \sum_i \frac{\rho_i^2}{1 - \rho_i^2})\prod_i (1 - \rho_i^2)
$$
This gives us
$$
\log \frac{|\Sigma_{\rho^0}|}{|\Sigma_{\rho^*}|} = \log \frac{1 + \sum_i \frac{(\rho_i^0)^2}{1 - (\rho_i^0)^2}}{1 + \sum_i \frac{(\rho_i^*)^2}{1 - (\rho_i^*)^2}} + \sum_i \log \frac{1 - (\rho_i^0)^2}{1 - (\rho_i^*)^2}
$$
The function $x \mapsto \log(1 - x^2)$ is decreasing, hence we have
$$
\sum_i \log \frac{1 - (\rho_i^0)^2}{1 - (\rho_i^*)^2} \leq n \log \frac{1 - c_1^2}{c_2}
$$
We also have 
$$
\sum_i \frac{(\rho_i^0)^2}{1 - (\rho_i^0)^2} \leq n  \frac{1}{c_2}
$$
Finally
$$
\sum_i \frac{(\rho_i^*)^2}{1 - (\rho_i^*)^2} \geq Cn  
$$
It follows that there exists a constant $C' = C'(M,\rho^*)$, such that 
\begin{equation}\label{eq:det}
\lp|\log \frac{|\Sigma_{\rho^0}|}{|\Sigma_{\rho^*}|} \rp| \leq C' n
\end{equation}
However, in the expression we have $\hat{\Sigma}$ instead of $\Sigma^*$. We use the property that
$$
\frac{\partial \log |A|}{\partial A} = A^{-1}
$$
We already showed that the entries of $\Sigma_{\rho}^{-1}$ are bounded if $\rho$ is upper and lower bounded by constants.
Using this and Taylor's Theorem, we get
\begin{equation}\label{eq:det_bound}
\lp|\log \frac{|\Sigma_{\rho^*}|}{|\hat{\Sigma}|}\rp| \leq C n^2
\end{equation}
for some constant $C$. 

Using inequalities \eqref{eq:trace} and \eqref{eq:det} together with the expression of KL, we conlcude that 
$$
KL(\mathcal{N}(0,\Sigma_{\rho^*})||\mathcal{N}(0,\Sigma_{\rho^0})) \leq K n
$$
where $K = K(\rho^*)$. 
Also, using this result and inequalities~\eqref{eq:det_bound} and \eqref{eq:trace_bound} we get
$$
KL(\mathcal{N}(0,\hat{\Sigma})||\mathcal{N}(0,\Sigma_{\rho^0})) \leq K' n^2 
$$
Overall, by equation\eqref{eq:likelihood} we get that
$$
L(\rho^0) \geq -\frac{n}{2} - \frac{1}{2} \log \lp( 2\pi\lp|\hat{\Sigma}\rp|\rp) - K' n^2 
$$
for some $K' = K'(c_1,c_2, \rho^*)$. 

Now, suppose for some iteration $t$ there exist two indices $i,j$ with $\rho_i^t > 1 - \epsilon, \rho_j^t > 1 - \epsilon$, where $\epsilon$ will be determined in the sequel. 
Now, let us upper bound $L(\rho^t)$. By KL subadditivity, we have
\begin{align*}
L(\rho^t, \sigma_x^t) &= -\frac{n}{2} - \frac{1}{2} \log \lp( 2\pi\lp|\hat{\Sigma}\rp|\rp) - KL(\mathcal{N}(0,\hat{\Sigma})||\mathcal{N}(0,\Sigma_{\rho^t,\sigma_x^t}))\\
&-\frac{n}{2} - \frac{1}{2} \log \lp( 2\pi\lp|\hat{\Sigma}\rp|\rp) - KL(\mathcal{N}(0,\hat{\Sigma}_{x_i,x_j})||\mathcal{N}(0,(\Sigma_{\rho^t,\sigma_x^t})_{x_i,x_j}))
\end{align*}
where in the last equation we are comparing the marginals of the distributions on $x_i,x_j$. 
The marginal distribution of $x_i,x_j$ is a gaussian with zero mean.
First, we will analyze the ideal situation where we have $\Sigma^*$ instead of $\hat{\Sigma}$.
Using the closed form of the KL divergence between two Gaussians, we obtain:
\begin{align*}
&KL(\mathcal{N}(0,\Sigma^*_{x_i,x_j})||\mathcal{N}(0,(\Sigma_{\rho^t,\sigma_x^t})_{x_i,x_j})) \\
&= 
\frac{1}{2} \lp(\log \frac{1 - (\rho_i^t\rho_j^t)^2}{1 - (\rho_i^*\rho_j^*)^2}  - 2 + \frac{1}{(1 - (\rho_i^t\rho_j^t)^2} tr\lp(\begin{pmatrix} 1 &- \rho_i^t\rho_j^t\\
-\rho_i^t\rho_j^t &1
\end{pmatrix}\begin{pmatrix} 1 &\rho_i^*\rho_j^*\\
\rho_i^*\rho_j^* &1
\end{pmatrix}\rp)\rp)\\
&= \frac{1}{2} \lp(\log \frac{1 - (\rho_i^t\rho_j^t)^2}{1 - (\rho_i^*\rho_j^*)^2}  - 2 + 2\frac{1 - \rho_i^t\rho_j^t\rho_i^*\rho_j^*}{1 - (\rho_i^t\rho_j^t)^2} \rp)
\end{align*}
Then, 
$$
KL(\mathcal{N}(0,\Sigma^*_{x_i,x_j})||\mathcal{N}(0,(\Sigma_{\rho^t,\sigma_x^t})_{x_i,x_j})) \geq \frac{1}{2} \lp(\log (1 - (\rho_i^t\rho_j^t)^2) - 2 + 2 \frac{1 - \rho_i^*\rho_j^*/4}{1 - (\rho_i^t\rho_j^t)^2}\rp)
$$

Now, for any constant $c > 0$, consider the function $f(x) = x (\log x + c/x)$. This is a continuous function on $(0,1)$ and it is easy to see that it is decreasing for $x < 1/e$. Let $x_0$ be a small enough constant such that $x_0 \log x_0 + c > c/2$. Then, for $x < x_0$ 
$$
f(x) \geq f(x_0) \implies \log x + c/x \geq \frac{c/2}{x}
$$
Now let $c = 2 ( 1 - \rho_i^*\rho_j^*/4)$. If we set $\epsilon = c'/ ( n^2)$ for some suitable constant $c'$, then we have that
$$
1 - (\rho_i^t\rho_j^t)^2  < x_0
$$
and
$$
KL(\mathcal{N}(0,\Sigma^*_{x_i,x_j})||\mathcal{N}(0,(\Sigma_{\rho^t,\sigma_x^t})_{x_i,x_j})) > (K' + K'') n^2  
$$
where $K''$ is a constant that will be determined now. 
Using the exact same arguments as in the case of $\rho^0$, we can show that 
$$
\lp|KL(\mathcal{N}(0,\hat{\Sigma}_{x_i,x_j})||\mathcal{N}(0,(\Sigma_{\rho^t,\sigma_x^t})_{x_i,x_j})) - KL(\mathcal{N}(0,\Sigma^*_{x_i,x_j})||\mathcal{N}(0,(\Sigma_{\rho^t,\sigma_x^t})_{x_i,x_j}))\rp| \leq C \eta \leq K''
$$
for some constant $C$. 
It follows that 
$$
L(\rho^t,\sigma_x^t) < -\frac{n}{2} - \frac{1}{2} \log \lp( 2\pi\lp|\hat{\Sigma}\rp|\rp) -   K'n^2  < L(\rho^0)
$$
This is a contradiction, since the likelihood value increases at each step of EM. This gives us the desired result.

\end{proof}
\subsection{No correlation can be too close to $1$}\label{s:upper_bound}
In the previous Section, we showed that two corrrelations cannot be close to $1$ at the same time. In this Section, building on this result, we show
that in fact no correlation can be too close to $1$. Hence, each $\rho_i$ is upper bounded at all iterations of the algorithm. This is the topic of the following Lemma:

\begin{lemma}\label{l:allbounded}
Suppose that $\min(\min_i \rho^0_i, \min_i \rho^*_i) \geq c_1, 1 - \max(\max_i \rho^0_i,\max_i \rho^*_i)  \geq c_2$, where $c_1,c_2 \in (0,1)$ are constants. Suppose we have access to $m$ i.i.d. samples $x^{(1)}, \ldots , x^{(m)}$ from the distribution $\mu_x^*$. Let $\rho_i^t$ be the correlations produced by the EM iteration \eqref{eq:update-rho-finite} run using these samples. Suppose $m = \Omega(\log(n/\delta)/\min(c_1,c_2)^2)$. Then, with probability at least $1-\delta$, for all $i$ and for all $t$, $\rho_i^t \leq 1 - C''/n^{4k+9}$, where $C''$ is an absolute constant and $k \geq 12$. 
\end{lemma}

To prove it, we will essentially prove that no $\rho_i$ will ever get really close to $1$. The proof will be similar to the one that established divergence from the saddle points at the boundary. We start by stating a direct Corollary of Lemma ~\ref{l:tworho}.

\begin{corollary}\label{cor1}
Suppose that $\min_i \rho^0_i \geq c_1, 1 - \max_i \rho^0_i  \geq c_2$, where $c_1,c_2 \in (0,1)$ are constants. 
Suppose we run EM \eqref{eq:update-rho-finite} with $m = \Omega(\log(n/\delta))$ samples. 
Then, with probability at least $1-\delta$, we have the following: if for some $t$ we have $\rho_1^t > 1 - c/n^2$, then for all $i \neq 1$ we have $\rho_i^t \leq 1 - c/n^2$, where $c$ is the constant of Lemma~\ref{l:tworho}. 
\end{corollary}

Our strategy now will be the following: we want to show that $\rho_1$ will not be very close to $1$, without loss of generality. 
First of all, we want to argue that if it comes close to $1$, it will immediately start moving away.  We do this by showing that after one iteration of EM, the errors of the others will be comparable to the error of $\rho_1$. We showed a similar claim in the proof of divergence from stationary points. The 
difference here is that the errors of the other $\rho_i$ can be close to $1$, while in the original proof they were assumed to lie in some ball around the fixpoint. 
Hence, we prove the following Lemma.

\begin{lemma}\label{l:error-align}
Suppose that for some $t$ we have $\rho_i^t \leq \alpha $ for all $i \neq 1$ and $\rho_1^t > \beta$, for some $\alpha,\beta \in (0,1)$. Also, suppose $\hat{\alpha}_{ij} \in [0,1]$ for all $i,j$. 
Define
\begin{equation}\label{eq:assumption}
R = R(\alpha,\beta,n) := n^2 \frac{1 - \beta}{\beta^4 (1 - \alpha)^2} + \frac{n}{\beta^2(1- \alpha)}
\end{equation}
and assume that $(1 - \beta) R \leq 1/2$. 
Then, for all $i \neq 1$ 
$$
|\rho_i^{t+1} - \hat{\alpha}_{1i}| \leq C \frac{1 + \frac{2n}{1-\alpha} + R}{1 - \lp(1 + \frac{2n}{1 - \alpha}\rp)\frac{n(1 - \beta)}{\beta^2(1-\alpha)} - Cn^2 \frac{(1-\beta)}{\beta^4(1 - \alpha)^2} - C\frac{n(1-\beta)R}{\beta^2(1 - \alpha) } }(1 - \rho_1^{t+1})
$$
where $C$ is some absolute constant(it is assumed that $\alpha,\beta$ are such that the quantity in the denominator is positive). 
\end{lemma}
The first step to proving this Lemma will be to show that if all $\rho_i$ are bounded away from $1$ except $\rho_1$, then $\lambda_i$ will be small for $i \neq 1$. 
This is the content of the following Lemma.

\begin{lemma}\label{l:lambda-bound}
Suppose for some $\rho$ we have $\rho_i \leq \alpha$ for all $i \neq 1$ and $\rho_1 > \beta$, for some $\alpha,\beta \in (0,1)$.Then, 
$$
\lambda_1 > \frac{\rho_1}{1 + 2n \frac{1 - \rho_1}{1 - \alpha}}
$$
and 
$$
\lambda_i \leq  \frac{2(1 - \rho_1)}{\beta^2 (1 - \alpha)}
$$

for all $i \neq 1$. 
\end{lemma}
\begin{proof}
We have 
$$
\lambda_i = \frac{\frac{\rho_i}{1 - \rho_i^2}}{1 + \sum_j \frac{\rho_j^2}{1 - \rho_j^2}} = \frac{\rho_i}{ 1 + \sum_{j\neq i} \frac{\rho_j^2(1 - \rho_i^2)}{1 - \rho_j^2}}
$$
For $i \neq 1$, this implies
$$
\lambda_i \leq \frac{1}{\sum_{j\neq i} \frac{\rho_j^2(1 - \rho_i^2)}{1 - \rho_j^2}} \leq \frac{1 - \rho_1^2}{\rho_1^2(1 - \rho_i^2)}
$$
We have $1 - \rho_1^2 = (1 - \rho_1)(1 + \rho_1) < 2(1 -\rho_1)$ and $1 - \rho_i^2 = (1 - \rho_i)(1 + \rho_i) \geq 1 - \alpha$. 
This gives 
$$
\lambda_i \leq \frac{2(1 - \rho_1)}{\beta^2 (1 - \alpha)}
$$
We also have
$$
\lambda_1 = \frac{\rho_1}{1 + \sum_{j\neq 1} \frac{\rho_j^2(1 - \rho_1^2)}{1 - \rho_j^2}}
$$
We have that
$$
\sum_{j\neq 1} \frac{\rho_j^2(1 - \rho_1^2)}{1 - \rho_j^2} \leq 2 n \frac{1 - \rho_1}{1 - \alpha}
\implies
\lambda_1 \geq \frac{\rho_1}{1 + 2n \frac{1 - \rho_1}{1 - \alpha}}
$$

\end{proof}
We are now ready to prove Lemma~\ref{l:error-align}. 

\begin{proof}[Proof of Lemma~\ref{l:error-align}]
% First of all, for the entirety of the proof we will assume that for all $i,j$ 
% $$
% |\hat{\alpha}_{ij} - \rho_i^*\rho_j^*| \leq \eta
% $$
% and that $\eta$ is sufficiently small so that $\hat{\alpha}_{ij}$ is bounded away from $0$ and $1$ by constants. 
% This holds with probability at least $1-\delta$ since the number of samples is $m = \Omega(\log(n/\delta)/\eta^2)$. 
We will first prove that 
\begin{equation}\label{eq:firstbound}
|\rho_i^{t+1} - \hat{\alpha}_{1i}| \leq C(\alpha,\beta) |\rho_1^t - 1|
\end{equation}
for some constant $C(\alpha,\beta)$ that will be specified in the sequel. 
We have that
$$
|\rho_i^{t+1} - \hat{\alpha}_{1i}| = \lp|\frac{\rho_i^t + \sum_{j\neq i}(\hat{\alpha}_{ij} - \rho_i^t \rho_j^t)\lambda_j^t}{\sqrt{1 + \sum_{j\neq k} \lambda_j^t\lambda_k^t (\rho_j^*\rho_k^* - \rho_j^t \rho_k^t)}} - \hat{\alpha}_{1i}\rp|
$$
As usual, denote $\Delta_{ij}^t = \hat{\alpha}_{ij} - \rho_i^t \rho_j^t$. Also, let
$$
U^t := \sum_{j\neq k} \lambda_j^t\lambda_k^t (\hat{\alpha}_{jk} - \rho_j^t \rho_k^t)
$$
Hence, we can write
\begin{align*}
|\rho_i^{t+1} - \hat{\alpha}_{1i}| &= \lp|\frac{\rho_i^t + \sum_{j \neq i}\Delta_{ij}^t \lambda_j^t - \hat{\alpha}_{1i} \sqrt{1 + U^t}}{\sqrt{1 + U^t}}\rp|\\
&= 
\lp|\frac{\rho_i^t + \lambda_1^t(\hat{\alpha}_{1i} - \rho_i^t \rho_1^t) - \hat{\alpha}_{1i} \sqrt{1 + U^t} +  \sum_{j \neq i,1}\Delta_{ij}^t \lambda_j^t }{\sqrt{1 + U^t}}\rp|\\
&= 
\lp|\frac{ (\lambda_1 - \sqrt{1 + U^t})(\hat{\alpha}_{1i} - \rho_i^t \rho_1^t) + \rho_i^t(1 - \rho_1^t\sqrt{1 + U^t})  +  \sum_{j \neq i,1}\Delta_{ij}^t \lambda_j^t }{\sqrt{1 + U^t}}\rp|
\end{align*}
Let's start by bounding $U^t$. 
Using lemma~\ref{l:lambda-bound} have that for $j , k \neq 1$:
$$
|\lambda_j^t\lambda_k^t\Delta_{jk}^t| \leq 2\lp(\frac{2(1 - \rho_1^t)}{\beta^2(1 - \alpha)}\rp)^2
$$
where we have used the fact that $\hat{\alpha}_{ij} \leq 1$. If $j = 1$ then
$$
|\lambda_1^t\lambda_k^t(\hat{\alpha}_{1k} - \rho_1^t \rho_k^t)| \leq 
2 \frac{2(1 - \rho_1^t)}{\beta^2(1 - \alpha)}
$$
Hence, we conclude that
$$
|U^t| \leq n^2 \frac{8(1 - \rho_1^t)^2}{\beta^4 (1 - \alpha)^2} + 4n\frac{1 - \rho_1^t}{\beta^2(1 - \alpha)}
$$
By Assumption\eqref{eq:assumption} we have that the right hand side of the last inequality is $< 1/2$. 
This gives, by a simple Taylor approximation, that
$$
|\sqrt{1 + U^t} - 1| \leq C |U^t|
$$
for some constant $C$. Hence, 
\begin{align*}
|\lambda_1^t - \sqrt{1 + U^t}| &\leq |\lambda_1^t - 1| + |\sqrt{1 + U^t} - 1| 
\leq 1 - \frac{\rho_1^t}{1 + 2n\frac{1 - \rho_1^t}{1 - \alpha}} + Cn^2 \frac{8(1 - \rho_1^t)^2}{\beta^4 (1 - \alpha)^2} + 4Cn\frac{1 - \rho_1^t}{\beta^2(1 - \alpha)}\\
&\leq \lp(1 + \frac{2n}{1 - \alpha}\rp)(1 - \rho_1^t) + Cn^2 \frac{8(1 - \rho_1^t)^2}{\beta^4 (1 - \alpha)^2} + 4Cn\frac{1 - \rho_1^t}{\beta^2(1 - \alpha)}
\end{align*}
We also have
\begin{align*}
|1 - \rho_1^t\sqrt{1 + U^t}| &\leq 1 - \rho_1^t + \rho_1^t|1 - \sqrt{1 + U^t}| 
\leq 1 - \rho_1^t + Cn^2 \frac{8(1 - \rho_1^t)^2}{\beta^4 (1 - \alpha)^2} + 4Cn\frac{1 - \rho_1^t}{\beta^2(1 - \alpha)}
\end{align*}
Lastly, we have
$$
\sum_{j\neq i,1} \Delta_{ij}^t \lambda_j^t \leq 4n \frac{1 - \rho_1^t}{\beta^2(1-\alpha)}
$$
Hence,
\begin{align*}
|\rho_i^{t+1} - \hat{\alpha}_{1i}| &\leq 
\frac{\lp(2 + \frac{2n}{1 - \alpha}\rp)(1 - \rho_1^t)  + Cn^2 \frac{8(1 - \rho_1^t)^2}{\beta^4 (1 - \alpha)^2} + 4Cn\frac{1 - \rho_1^t}{\beta^2(1 - \alpha)} +  4n \frac{1 - \rho_1^t}{\beta^2(1-\alpha)}}{\sqrt{1 - n^2 \frac{8(1 - \rho_1^t)^2}{\beta^4 (1 - \alpha)^2} - 4n\frac{1 - \rho_1^t}{\beta^2(1 - \alpha)}}}\\
&\leq C (1 - \rho_1^t) \lp(1 + \frac{2n}{1 - \alpha} + n^2 \frac{1 - \beta}{\beta^4 (1 - \alpha)^2} + \frac{n}{\beta^2(1- \alpha)}\rp)
\end{align*}
for some constant $C$, since we have assumed \eqref{eq:assumption}. If we set $C(\alpha,\beta)$ equal to the multiplier of $1 - \rho_1^t$, we have obtained \eqref{eq:firstbound}. 
The second step is to show that 
\begin{equation}\label{eq:secondbound}
1 - \rho_1^{t+1} \geq K(\alpha,\beta) (1 - \rho_1^t) 
\end{equation}
Inequalities \eqref{eq:firstbound} and \eqref{eq:secondbound} together give us the desired result. 
We have that
\begin{align*}
\rho_1^{t+1} &= \frac{\rho_1^t + \sum_{j\neq 1}(\hat{\alpha}_{1j} - \rho_1^t \rho_j^t)\lambda_j^t}{\sqrt{1 + \sum_{j\neq k} \lambda_j^t\lambda_k^t (\hat{\alpha}_{jk} - \rho_j^t \rho_k^t)}} 
\end{align*}
First of all, by Taylor we have that if $|x| \leq 1/2$
$$
(1 + x)^{-1/2} \leq  1 - \frac{x}{2} + cx^2
$$
 We have already shown that $|U^t|$ is smaller than a constant (using also assumption \eqref{eq:assumption}). This gives
\begin{align*}
\rho_1^{t+1} &\leq \lp(\rho_1^t + \sum_{j\neq 1}(\hat{\alpha}_{1j} - \rho_1^t \rho_j^t)\lambda_j^t\rp)\lp(1 - \frac{U^t}{2} + c|U^t|^2 \rp)\\
&= \rho_1^t + \sum_{j\neq 1}(\hat{\alpha}_{1j} - \rho_1^t \rho_j^t)\lambda_j^t - \rho_1^t \lambda_1^t \sum_{j\neq 1}(\hat{\alpha}_{1j} - \rho_1^t \rho_j^t)\lambda_j^t - \frac{\rho_1^t}{2} \sum_{j\neq k \neq 1} \Delta_{ij}^t \lambda_j^t\lambda_k^t - \frac{U^t}{2}  \sum_{j\neq 1}\Delta_{1j}^t\lambda_j^t\\ 
& \quad\quad\quad + c|U^t|^2 (\rho_1^t + \sum_{j\neq 1}\Delta_{1j}^t\lambda_j^t)\\
&\leq
\rho_1^t +(1 - \rho_1^t\lambda_1^t)\sum_{j\neq 1}\Delta_{1j}^t\lambda_j^t + Cn^2 \frac{(1-\rho_1^t)^2}{\beta^4(1 - \alpha)^2} + C\frac{n|U^t|(1 - \rho_1^t)}{\beta^2(1 - \alpha) } + C\frac{n|U^t|^2 (1 - \rho_1^t)}{\beta^2(1 - \alpha)}
\end{align*}
We know that
$$
\sum_{j\neq 1} \Delta_{1j}^t \lambda_j^t \leq 4n\frac{1-\rho_1^t}{\beta^2(1 - \alpha)}
$$
and 
$$
1 - \rho_1^t\lambda_1^t = 1 - \rho_1^t + \rho_1^t(1 - \lambda_1^t) \leq (1 - \rho_1^t) \lp(1 + \frac{2n}{1-\alpha}\rp)
$$
Hence, we have that 
\begin{align*}
1 - \rho_1^{t+1} &\geq (1 - \rho_1^t)\lp(1 - \lp(1 + \frac{2n}{1 - \alpha}\rp)\frac{n(1 - \rho_1^t)}{\beta^2(1-\alpha)} - Cn^2 \frac{(1-\rho_1^t)}{\beta^4(1 - \alpha)^2} - C\frac{n|U^t|}{\beta^2(1 - \alpha) } \rp)  
\end{align*}

\end{proof}

We now have almost all the necessary ingredients required for the proof of Lemma~\ref{l:allbounded}. 
The final Lemma we will prove says that if $\rho_1$ happens to come very close to $1$, then after one iteration it will start moving away from $1$. This is qualitatively similar to the proof of Theorem 25 in the manuscript, but now we have to make precise quantitative predictions for how close it should be to $1$. 
This is the content of the following Lemma. 

\begin{lemma}\label{l:repel}
Suppose that $\min(\min_i \rho^0_i, \min_i \rho^*_i) \geq c_1, 1 - \max(\max_i \rho^0_i,\max_i \rho^*_i)  \geq c_2$, where $c_1,c_2 \in (0,1)$ are constants. Suppose we run EM \eqref{eq:update-rho-finite} with $m$ samples and suppose $\rho_1^t > 1 - c'/n^k > 1 - c/n^2$, where $c'>0$ is a sufficiently small constant, $k \geq 12$ and $c$ is the constant of Corollary~\ref{cor1}.
Finally, assume that $m = \Omega(\log(n/\delta)/\min(c_1,c_2)^2)$. Then, with probability at least $1-\delta$, for all $t' \geq t+1$ we have that if $\rho_1^{t'} > c'/n^k$, then for some absolute constant $C>0$ we have
$$
\rho_1^{t'+1} \leq  \rho_1^{t'} - \frac{C}{2} (1 - \rho_1^{t'})^2
$$
\end{lemma}

\begin{proof}
We first state a slight modification of Lemma~\ref{l:pushback-pop}. The proof is exactly the same, but now we are applying Taylor in the update rule of the empirical EM instead of the population. 

\begin{lemma}\label{l:pushback}
Let $\rho^t$ be the current iteration of empirical EM \eqref{eq:update-rho-finite} and define
$$
\epsilon^t :=\max(1 - \rho_1^t, \max_i |\hat{\alpha}_{1i}- \rho^t_i|)
$$
Suppose also that $\min_i \hat{\alpha}_{1i} \geq c_1, 1- \max_i \hat{\alpha}_{1i} \geq c_2$, where $c_1,c_2$ are constants. 
Then, there exist absolute constants $C,K >0$ depending on $c_1,c_2$ such that
$$
\rho^{t+1}_1 \leq \rho^t_1 - C|\rho_1^t - 1|^2 + Kn^3(\epsilon^t)^3
$$
\end{lemma}

By Corollary~\ref{cor1} we have that $\rho_i^t < 1 - c/n^2$. Thus, if we set $\alpha = 1- c/n^2$ and $\beta = 1 - c'/n^k$ we get that
$$
(1 - \beta) R(\alpha,\beta,n) \leq C \frac{1}{n^k} \lp(\frac{1}{n^{k-6}} + n^3\rp) 
$$
which can be made less than $1/2$ with $k \geq 3$. 
Thus, all the assumptions of Lemma~\ref{l:error-align} are satisfied. By 
applying this result, we get that for some constant C
$$
|\rho_i^{t+1} - \hat{\alpha}_{1i}| \leq 
C\frac{n^3}{1 - \frac{C'}{n^{k-6}}}(1 - \rho_1^{t+1}) \leq C'' n^3(1 - \rho_1^{t+1})
$$
Similarly, for any $t' \geq t+1$ such that $\rho_1^{t'} > 1 - c'/n^k$ the previous calculations apply, hence we get for all $t' \geq t+1$
$$
|\rho_i^{t'} - \hat{\alpha}_{1i}| \leq 
C\frac{n^3}{1 - \frac{C'}{n^{k-6}}}(1 - \rho_1^{t+1}) \leq C'' n^3(1 - \rho_1^{t'})
$$
Thus, for every $t' \geq t+1$ we have
$$
\epsilon^{t'} \leq C''n^3 (1 - \rho_1^{t'})
$$
To apply Lemma~\ref{l:pushback}, notice that the assumption on $\hat{\alpha}$ is automatically satisfied by our assumption on the number of samples $m$. 
Hence, by applying Lemma~\ref{l:pushback}, we get that
$$
\rho_1^{t'+1} = \rho_1^{t'} - C(1 -\rho_1^{t'})^2 + K' n^{12} (1 - \rho_1^{t'+1})^3 
$$
Now, if $k \geq 12$ and $c'$ is a small enough constant, we have
$$
K' n^{12} (1 - \rho_1^{t'+1}) < C/2
$$
which implies that
$$
\rho_1^{t'+1} \leq  \rho_1^{t'} - \frac{C}{2} (1 - \rho_1^{t'})^2
$$

\end{proof}

We are now ready to present the proof of the main Lemma of this section.

\begin{proof}[Proof of Lemma~\ref{l:allbounded}]
First of all, our assumption on $m$ implies that $\hat{\alpha}_{ij}$ is bounded away from $0$ and $1$ with probability at least $1-\delta$. Hence, all the claims that follow hold with probability at least $1-\delta$ (because the boundedness property is required to apply some of these Lemmas). 
Initially, we have $\rho_i^0 < 1 - c'/n^k$ for all $i$. Let $t$ by the first time where $\rho_i^t > 1 - c'/n^k$ for some $i$. Let $T$ be the first time after $t$ such that $\rho_i^t \leq 1 - c'/n^k$. Then, by Lemma~\ref{l:repel} we have that 
$\rho_i^{t'+1} \leq \rho_i^{t'}$ for all $T > t' \geq t+1$. Hence, we have 
$\rho_i^{t'} \leq max(\rho_i^t,\rho_i^{t+1})$ for all $t' < T$. Let us now try to upper bound $\rho_i^{t+1}$ and $\rho_i^t$. 
First of all, we can write $y^t = Ax_i^t + r^t$, where $x_i^t$ is the random variable of the leaf $i$ sampled according to the parameters $\sigma_{x_i}^{t-1}$, $y^t$ is the random variable that is sampled conditional on all $x_i^t$ and with model parameters $\rho_i^{t-1}$ and $r^t$ is independent of $x_i^t$. Let $v^2$ be the variance of $r^t$. Then, by definition
$$
\rho_i^t = \frac{A}{\sqrt{A^2 + v^2}} = \frac{1}{\sqrt{1 + (v/A)^2}}
$$
Then,
we have that $A = Cov(y^t,x_i^t)/ \sigma_{x_i}^t \leq \sum_i \lambda_i^{t-1}/\sigma_{x_i}^t \leq Cn$. Here we used the fact that $|\sigma_{x_i}^t - 1| \leq C$, where $C$ is a small constant, which is implied by the assumption on the number of samples $m$.  As for $v^2$, we can show 
that
$$
v^2 \geq \frac{1}{1 + \sum\frac{(\rho_i^{t-1})^2}{1-(\rho_i^{t-1})^2}}
$$
Indeed, the right hand side of this inequality is the conditional variance of $y$ conditioned an all $x_j^t$ for $j \in [n]$. Since conditioning reduces the variance, the inequality is obvious. 
Now, we know that $1 - (\rho_i^{t-1})^2 > 1 - c'/n^k$ for all $i$ by definition. Hence,
$$
v^2 \geq\frac{C}{n^{2k+1}}
$$
Hence, 
$$
\rho_i^t \leq \frac{1}{\sqrt{1 + C/n^{2k+3}}} \leq 1 - \frac{C'}{n^{2k+3}}
$$
Using this bound, in a completely similar fashion we get that
$$
\rho_i^{t+1} \leq 1 - \frac{C''}{n^{4k+9}}
$$
Hence, we have that $\rho_i^{t'} \leq 1 - C/n^{4k+9}$ for all $t' < T$. Now, let's examine what happens after time $T$. We proved 
in Lemma~\ref{l:pushback} that $|\rho_j^{T} - \hat{\alpha}_{ij}| \leq Cn^3(1 - \rho_i^{T-1}) \leq C/n^{k-3}$
for all $j \neq i$. Hence, for $k \geq 2$ and small enough constant $q \in (0,1)$ we have that $\rho_j^T \leq \hat{\alpha}_{ij} + q \leq 1 - c'/n^k$ for $j \neq i$. We already know that $\rho_i^T \leq 1 - c'/n^k$ by definition of $T$. Hence, at the $T$-th iteration all correlations will remain bounded away from $0$. We can then consider the first time after $T$ such that some correlation becomes bigger than $1 - c'/n^k$ and the same bounds that we already established apply again and so on. Hence, we conclude that for all $i,t$, $\rho_i^t \leq 1 - C''/n^{4k+9}$.

\end{proof}
\subsection{No correlation is close to $0$}\label{s:lower_bound}
In this Section, we prove a lower bound for all the correlations at any iteration of the algorithm. 
We will prove the following Lemma.

\begin{lemma}\label{l:lower_bound}
Suppose that $\min(\min_i \rho^0_i, \min_i \rho^*_i) \geq c_1, 1 - \max(\max_i \rho^0_i,\max_i \rho^*_i)  \geq c_2$, where $c_1,c_2 \in (0,1)$ are constants. Suppose we have access to $m$ i.i.d. samples $x^{(1)}, \ldots , x^{(m)}$ from the distribution $\mu_x^*$. Let $\rho_i^t$ be the correlations produced by the EM iteration \eqref{eq:update-rho-finite} run using these samples. Suppose $m = \Omega(\log(n/\delta)/\min(c_1,c_2)^2)$. Then, there exists a constant $c = c(c_1,c_2)$ such that with probability at least $1-\delta$, for all $t$ and for all $i \in [n]$
$\rho_i > c/n^{k+2}$ for $k \geq 4$. 
\end{lemma}

% \begin{remark}
% The proof is assumes that the executed dynamics are according to the population EM. The extension to the sample EM is straightforward.
% \end{remark}

\begin{proof}
We start by stating a slight modification of Lemma 22 in the case of sample EM. 
\begin{lemma}\label{l:der_zero}
Let $\rho^t$ be the iteration of sample EM \eqref{eq:update-rho-finite}. 
If $\rho^t =  0$ then $\rho^{t+1}= 0$. Further, there exists $C>0$ such that for all $i, j,k \in \{1,\dots,n\}$ and all $\rho \in [0,1/2]^n$,
$$
\frac{d\rho^{t+1}_i}{d\rho^t_j}\bigg|_{ 0} = \begin{cases}
1 & i=j\\
\hat{\alpha}_{ij} & i\ne j
\end{cases} ~; \qquad
-C\le \frac{d^2\rho_i^{t+1}}{d\rho^t_j d\rho^t_k}\bigg|_{\rho}\le C.
$$
\end{lemma}

By our assumption on the number of samples $m$, we get that all $\hat{\alpha}_{ij}$ are bounded away from $0$ and $1$ depending on the constants $c_1,c_2$, with probability $1-\delta$. From now on, we will assume this holds deterministically.
Using Lemma~\ref{l:der_zero} and Taylor's Theorem, we get that for all $i$
\begin{equation}\label{eq:taylor}
\rho^{t+1}_i \ge \rho^t_i +  \sum_{j\ne i} \rho_j\hat{\alpha}_{ij} -\sum_{j,k} \frac{C}{2} \rho^t_j\rho^t_k.
\end{equation}
We now prove that during all iterations $t$ and for all $i \neq j$, we have $\rho_i^t \geq c(n) \rho_j^t$, for some constant $c(n)$ that is independent of the iteration $t$. 
For $t = 0$, we can find some constant $c$ independent of $n$ such that this is satisfied (because we initially start at a constant distance from the optimum). 
Now, suppose $t > 1$. Suppose without loss of generality that $\rho_1^{t-1} \geq \rho_i^{t-1}$ for all $i \neq 1$. Let us fix $i \neq j$. 
By the update rule, we have that
$$
\rho_i^t \geq \frac{\lambda_1^{t-1} \hat{\alpha}_{1i}}{\sqrt{1 + 2n^2}} \geq \frac{\hat{\alpha}_{1i}}{n\sqrt{3}} \lambda_1^{t-1}
$$
On the other hand, we have that
$$
\rho_j^t \leq \sum_k \lambda_k^{t-1} \leq n \lambda_1^{t-1}
$$
since $\lambda_1^{t-1} \geq \lambda_k^{t-1}$ for all $k \neq 1$. Thus, we conclude that 
$$
\rho_i^t \geq \frac{\hat{\alpha}_{1i}}{n^2\sqrt{3}} \rho_j^t \geq \frac{K}{n^2} \rho_j^t
$$
so we set $c(n) := K/n^2$, where $K$ is some constant that depends on the lower bound on $c_1,c_2$ (it is lower bounded by assumption).

Now that we have proven this claim, we can use inequality\eqref{eq:taylor}. Suppose w.l.o.g. that $1 = \arg\max_i \rho_i^t$. Then
$$
\rho^{t+1}_i \ge \rho^t_i +  \sum_{j\ne i} \rho_j\hat{\alpha}_{ij} -\sum_{j,k} \frac{C}{2} \rho^t_j\rho^t_k \geq \rho_i^t + n\frac{K}{n^2} \rho_1^t -Cn^2 (\rho_1^t)^2 = 
\rho_i^t + \frac{K}{n} \rho_1^t -Cn^2 (\rho_1^t)^2
$$
Suppose $\rho_1^t < c/n^k$, where $k$ will be determined later. Then,
$$
\rho_i^{t+1} \geq \rho_i^t + \rho_1^t(K/n - Cn^2/n^k) > \rho_i^t
$$
if $k-2 \geq 2$ and $c$ is small enough. Now suppose $\rho_1^t \geq c/n^k$. Then, we have already proved that 
$$
\rho_i^{t+1} \geq \frac{K}{n} \lambda_1^t
$$
Now, we have
$$
\lambda_1^t = \frac{\rho_1^t}{1 + \sum_{j \neq 1}\frac{(\rho_j^t)^2 (1 - (\rho_1^t)^2)}{1 - (\rho_j^t)^2}} \geq \frac{\rho_1^t}{n+1}
$$
since $1 -(\rho_1^t)^2 \leq 1 - (\rho_j^t)^2$ for all $j \neq 1$. It follows that
$$
\rho_i^{t+1} \geq \frac{cK}{n^{k+2}}
$$
Hence, if the maximum exceeds the threshold, all correlations are lower bounded, otherwise they do not decrease. We conclude that for all iterations $t$ and all $i$ the required bound holds. 
\end{proof}
\section{Proof of Theorem~\ref{thm:general-finite}: finite sample and finite iterate}
\label{app:finite_convergence}

\subsection{A deterministic assumption}

Note that our theorem holds with high probability. To remove the probabilistic part, we assume a deterministic assumption on the sample, that will hold with high probability. First, we present a definition:
\begin{definition}
Let $\eta>0$. We say that a sample $x^{(1)},\dots,x^{(m)}$ is $\eta$-representative of $\mu_x^*$ if the following hold:
\begin{itemize}
    \item For all $i=1,\dots,n$, 
    \[1 - \eta \le \sqrt{\frac{1}{m}\sum_{k=1}^m \lp( x^{(k)}_i \rp)^2} \le 1 + \eta.
    \]
    (recall that we assumed that the variance of each coordinate of $\mu^*_x$ is $1$).
    \item For all $i \ne j \in \{1,\dots,n\}$:
    \[
    \rho^*_i \rho^*_j - \eta \le \frac{1}{m}\sum_{k=1}^m x^{(k)}_i x^{(k)}_j
    \le \rho^*_i \rho^*_j + \eta\enspace.
    \]
    \item For all $i\ne j$,
    \[
    \rho^*_i\rho^*_j - \eta \le \hat{\alpha}_{ij} = \frac{\frac{1}{m}\sum_{k=1}^m x^{(k)}_i x^{(k)}_j}{\sqrt{\frac{1}{m}\sum_{k=1}^m \lp(x^{(k)}_i\rp)^2}\sqrt{\frac{1}{m}\sum_{k=1}^m \lp(x^{(k)}_j\rp)^2}}
    \le \rho^*_i\rho^*_j + \eta\enspace.
    \]
\end{itemize}
\end{definition}
We will assume that the sample is $\eta$-representative, where $\eta$ is sufficiently small. We note that from Chernoff-Hoeffding bound, the sample is $\eta$-representative with probability $1-\delta$, if $m \ge \Omega(\log(n/\delta)/\eta^2)$.

\subsection{Iterates are bounded away from 0 and 1}

We start by recalling that in all iterates, the correlations $\rho^t_i$ are always bounded away from $0$ and $1$, assuming finite sample: (proof is in Section~\ref{sec:bounded-iter})
\bounded*
% \begin{lemma}
% Let $\alpha,\beta>0$ and assume that $\rho^0_i,\rho^*_i \ge \alpha$ for all $i$ while $\rho^0_i,\rho^*_i \le 1-\beta$ for all $\beta$. Then, there exist constants, $C(\alpha,\beta)$ $C_1(\alpha,\beta)$ and $C_2(\alpha,\beta)$, that depend only on $\alpha$ and $\beta$ and a universal constant $C'$ such that for all $t\ge 0$ and all $i$,
% \[
% \frac{C_1(\alpha,\beta)}{n^{12}} \le \rho^t_i \le 1- \frac{C_2(\alpha,\beta)}{n^{29}}\enspace,
% \]
% assuming that $\eta \le C(\alpha,\beta) /n^{C'}$.
% \end{lemma}

Next, we would use the fact that the iterates are always bounded away from $0$ and $1$ for the remainder of the proof. For that purpose, we have the following definition:

\begin{definition}
We say that the iterates of the EM are $(A,B)$-bounded if $A \le \rho^*_i \le 1-B$ and $A \le \rho^t_i \le 1-B$ for all $i$ and $t \ge 0$.
\end{definition}

\subsection{Sample-EM is close to the population EM}

Next, we will argue that one iterate of the sample EM close to one iterate of the population EM, if the sample is $\eta$-representative.

\begin{lemma}\label{lem:sample-close-to-pop}
Assume that the sample is $\eta$-representative, for some $\eta>0$. Fix values of $\rho^t_1,\dots,\rho^t_n$ and let $\rho^{t+1}_{1},\dots,\rho^{t+1}_n$ denote the value of the next iterate according to the finite-sample update of Lemma~\ref{lem:update-sample}. Similarly, denote by $\tilde{\rho}^{t+1}_1,\dots,\tilde{\rho}^{t+1}_n$ the result of applying the \emph{population EM} update, as described in Lemma~\ref{lem:update}. Assume that the iterates of the EM are $(A,B)$-bounded, that $A \le 1/n$ for some $A>0$ and assume that $\eta \le A^2/2$. Then,
\[
\lp| \rho^{t+1}_i - \tilde{\rho}^{t+1}_i \rp| \le 
\eta \lp(\frac{4\sqrt{8}n^3}{A^3} + \frac{\sqrt{8}n}{A}\rp)
\]
\end{lemma}
\begin{proof}
Fix $i \in \{1,\dots,n\}$.
Let us denote the numerator and the denominator, for both the expressions for $\rho^{t+1}_i$ and $\tilde{\rho}^{t+1}_i$, as follows:
\[
X = \lambda^t_i + \sum_{j\ne i} \lambda^t_j \hat{\alpha}_{ij}, \quad X' = \lambda^t_i + \sum_{j\ne i} \lambda^t_j \rho^*_i \rho^*_j
\]
and
\[
Y = \sum_{i=1}^n (\lambda^t_i)^2 + \sum_{i\ne j \in \{1,\dots,n\}}
\lambda^t_i \lambda^t_j \hat{\alpha}_{ij}, \quad
Y' = \sum_{i=1}^n (\lambda^t_i)^2 + \sum_{i\ne j \in \{1,\dots,n\}}
\lambda^t_i \lambda^t_j \rho^*_i\rho^*_j\enspace.
\]
Then, the quantity that we wish to bound is:
\begin{equation}
\begin{aligned}\label{eq:numerator-and-denom}
\lp|\frac{X}{\sqrt{Y}} - \frac{X'}{\sqrt{Y'}}\rp|
&= \lp|\frac{X}{\sqrt{Y}} - \frac{X}{\sqrt{Y'}} + \frac{X}{\sqrt{Y'}} -
\frac{X'}{\sqrt{Y'}}\rp|
\le |X| \lp| \frac{1}{\sqrt Y} - \frac{1}{\sqrt{Y'}} \rp|
+ \frac{1}{\sqrt{Y'}} \lp|X-X' \rp|\\
&= \frac{X \lp|\sqrt{Y}-\sqrt{Y'}\rp|}{\sqrt{Y Y'}}
+ \frac{1}{\sqrt{Y'}} \lp|X-X' \rp|
= \frac{X \lp|Y-Y'\rp|}{\sqrt{Y Y'} \lp| \sqrt{Y} + \sqrt{Y'}\rp|}
+ \frac{1}{\sqrt{Y'}} \lp|X-X' \rp| \enspace.
\end{aligned}
\end{equation}
Our goal is to upper bound $X$, lower bound $Y$ and $Y'$, and upper bound $|X-X'|$ and $|Y-Y'|$, as computed below.
First, notice that
\begin{equation}\label{eq:compare-numerator}
|X-X'| = \lp| \lp(\lambda^t_i + \sum_{j\ne i} \lambda^t_j \hat{\alpha}_{ij}\rp) - \lp(\lambda^t_i + \sum_{j\ne i} \lambda^t_j \rho^*_i\rho^*_j\rp) \rp|
\le \sum_{j\ne i} \lambda^t_j |\hat{\alpha}_{ij} - \rho^*_i\rho^*_j| \le \eta \sum_{j=1}^n \lambda^t_j,
\end{equation}
where the last inequality follows by assumption of this lemma that $|\hat{\alpha}_{ij} - \rho^*_i\rho^*_j| \le \eta$. Recall the expression for $\lambda^t_i$ in \eqref{eq:conditional}, and notice that we can bound it as follows:
\[
\lambda^t_i = \frac{\rho^t_i/(1-(\rho^t_i)^2)}{1 + \sum_{j=1}^n (\rho^t_j)^2/(1-(\rho^t_j)^2)}
\le \frac{\rho^t_i/(1-(\rho^t_i)^2)}{1 + (\rho^t_i)^2/(1-(\rho^t_i)^2)}
= \rho^t_i \le 1.
\]
Then, $\sum_i \lambda_i \le n$, which implies that 
\[
|X-X'| \le \eta n.
\]
Next:
\[
|Y-Y'| = \lp| \sum_{j \ne i} \lambda^t_i \lambda^t_j \hat{\alpha}_{ij}
- \sum_{j \ne i} \lambda^t_i \lambda^t_j \rho^*_i\rho^*_j
\rp|
\le \sum_{j\ne i} \lambda^t_i \lambda^t_j |\hat{\alpha}_{ij} - \rho^*_i\rho^*_j|
\le \eta n^2,
\]
using the fact that $\lambda^t_i \le 1$ and the assumption of this lemma that $|\hat{\alpha}_{ij} - \rho^*_i\rho^*_j| \le \eta$. Further,
\[
X = \lambda^t_i + \sum_{j\ne i} \lambda^t_j \hat{\alpha}_{ij} \le \sum_{j=1}^n \lambda^t_j \le n.
\]
Further,
\[
Y' \ge \sum_{i=1}^n\sum_{j=1}^n \lambda^t_i \lambda^t_j \rho^*_i\rho^*_j \ge \sum_{i=1}^n \sum_{j=1}^n A^2
= \lp( \sum_{i=1}^n \lambda^t_i \rp)^2 A^2,
\]
using the assumption that $\rho^*_i \ge A$ for all $i$. Similarly, we derive that
\[
Y \ge \sum_{i=1}^n\sum_{j=1}^n \lambda^t_i \lambda^t_j \hat{\alpha}_{ij}
\ge \sum_{i=1}^n\sum_{j=1}^n \lambda^t_i \lambda^t_j A^2/2
= \lp( \sum_{i=1}^n \lambda^t_i \rp)^2 A^2/2,
\]
using the assumptions that $\rho^*_i \ge A$ for all $i$, and that $|\hat{\alpha}_{ij} - \rho^*_i\rho^*_j| \le A^2/2$. It remains to lower bound the value of $\sum_i\lambda^t_i$:
\[
\sum_i \lambda^t_i = \frac{\sum_{i=1}^n\rho^t_i/(1-(\rho^t_i)^2)}{1 + \sum_{i=1}^n (\rho^t_i)^2/(1-(\rho^t_i)^2)}
\ge \frac{\sum_{i=1}^n\rho^t_i/(1-(\rho^t_i)^2)}{1 + \sum_{i=1}^n \rho^t_i/(1-(\rho^t_i)^2)}
\ge \frac{\sum_{i=1}^n\rho^t_i}{1 + \sum_{i=1}^n \rho^t_i}
\ge \frac{nA}{1+nA}\enspace,
\]
where, for the last two inequalities, we used the fact that if $a\ge a'>0$ and $b\ge 0$ then $a/(a+b) \ge a'/(a'+b)$. We will further use the assumption that $A \ge 1/n$ to derive that the right hand side is lower bounded by $1/2$. We derive that
\[
Y, Y' \ge A^2/8.
\]
Substituting our estimates from \eqref{eq:numerator-and-denom}, we derive that, for some universal constant $C>0$,
\[
\lp| \frac{X}{\sqrt{Y}} - \frac{X'}{\sqrt{Y'}}\rp| \le
\frac{n \cdot \eta n^2}{A^2/8\cdot 2(A/\sqrt{8})}
+ \frac{1}{A/\sqrt{8}} \eta n
= \eta \lp(\frac{4\sqrt{8}n^3}{A^3} + \frac{\sqrt{8}n}{A}\rp)\enspace.
\]
\end{proof}

\subsection{KL is comparable to the parameter $\ell_2$ distance}

In our argument, we will prove that in each iteration, the KL divergence between the iterate and the true distribution shrinks by a constant factor. In order to argue about that, we would like to claim that the KL divergence between two models is comparable to the $\ell_2$ distance between their correlation parameters $\rho_i$, provided that those are bounded away from $0$ and $1$:

\begin{lemma}\label{lem:kl-to-l2}
Assume that the iterates of the EM are $(A,B)$-bounded. Then, for some universal constant $C>0$ and for all $t\ge 0$,
\[
(n/AB)^{-C} \|\rho^t-\rho^*\|^2
\le \KL(\mu^*_x\|\mu^t_x) \le (n/(AB))^C \|\rho^t-\rho^*\|^2.
\]
\end{lemma}

In the sections below we will prove statements that are slightly more general (which will also be used further in the proof).

\subsubsection{Upper bounding the KL}

We start by proving the following lemma:
\begin{lemma} \label{lem:comp-two-kls}
Assume that $\rho^*_i,\rho^t_i \le 1-B$ for all $i$ and some $B>0$. Then, 
\[
\lp|
\frac{d \KL(\mu^*_x \|\mu^t_x)}{d\rho^t_i}\rp|
\le \poly(n,1/B) \cdot \|\rho^t - \rho^*\|_2.
\]
Consequently, for $\rho$ and $\rho'$,
\[
\lp| \KL(\mu^*_x\|\mu^{\rho}_x) - \KL(\mu^*_x\|\mu^{\rho'}_x) \rp|
\le \poly(n,1/B) \max(\|\rho'-\rho^*\|, \|\rho-\rho^*\|) \|\rho'-\rho\|.
\]
Lastly,
\[
\KL(\mu^*_x \|\mu^t_x) \le \poly(n,1/B) \cdot \|\rho^t - \rho^*\|_2^2.
\]
\end{lemma}

\begin{proof}
We start by proving the first inequality. We use a few folklore equations: first, we define for any real-valued function $f$ of a matrix $A$ by $df(A)/dA$ the matrix whose $ij$-entry is the derivative of $f(A)$ as a function of $A_{ij}$. The first folklore equation is:
\begin{equation}\label{eq:derive-det}
\frac{d \log \det(A)}{dA} = (A^{-1})^\top.
\end{equation}
For the second equation, assume that $A$ is a function of some parameter $\lambda$. Then,
\[
\frac{dA^{-1}}{d\lambda}
= -A^{-1} \frac{dA}{d\lambda} A^{-1},
\]
where $dA/d\lambda$ is a matrix whose $ij$ entry is the derivative of $A_{ij}$ as a function $\lambda$.
Further:
\begin{equation}\label{eq:matrix-chain-rule}
\frac{df(A)}{d\lambda}
= \tr\lp(\lp(\frac{df(A)}{dA}\rp)^\top \frac{dA}{d\lambda} \rp).
\end{equation}
Lastly, if $A \in \mathbb{R}^{n\times n}$ is symmetric and $\|A\|$ is the operator norm of $A$ then
\begin{equation}\label{eq:bnd-operator}
\tr(A) \le n \|A\|.
\end{equation}
Indeed, this is true because $\tr(A)$ is the sum of its singular values of $A$ while $\|A\|$ is the largest singular value in absolute value.

Let us continue by analyzing the KL divergence. We note that the KL divergence between two mean-zero Gaussian vectors and covariances $\Sigma_1$ and $\Sigma_2$, is

\[
\KL(\Sigma_1 \| \Sigma_2)
= \frac{1}{2} \lp(
-\log \frac{\det((\Sigma_2)^{-1})}{\det((\Sigma_1)^{-1})}
- n + \tr((\Sigma_2)^{-1} \Sigma_1)
\rp).
\]
Applying this on the covariances $\Sigma^*$ and $\Sigma^t$ that correspond to $\rho^*$ and $\rho^t$, we get that
\[
\KL(\mu^*_x \| \mu^t_x)
= \frac{1}{2} \lp(
-\log \frac{\det((\Sigma^t)^{-1})}{\det((\Sigma^*)^{-1})}
- n + \tr((\Sigma^{t})^{-1} \Sigma^*)
\rp).
\]
Let us first compute the first derivative as a function of $\rho^t_i$. By \eqref{eq:matrix-chain-rule}, we have that
\begin{equation}\label{eq:apply-chain-1}
\frac{d\KL(\mu^*_x \| \mu^t_x)}{d\rho^t_i}
= \tr\lp(\lp(\frac{d\KL(\mu^*_x \| \mu^t_x)}{d(\Sigma^t)^{-1}}\rp)^{\top} \frac{d(\Sigma^t)^{-1}}{d\rho^t_i}\rp) ~.
\end{equation}
Let us expand on the first term. By \eqref{eq:derive-det},
\[
\frac{d\KL(\mu^*_x \| \mu^t_x)}{d(\Sigma^t)^{-1}}
= \frac{1}{2} \frac{d}{d(\Sigma^t)^{-1}} \lp(
- \log \det((\Sigma^t)^{-1}) + \tr((\Sigma^{t})^{-1} \Sigma^*)
\rp)
= \frac{1}{2} \lp(
- (\Sigma^t)^{\top} + \Sigma^* \rp)
= \frac{1}{2} \lp(
\Sigma^* - \Sigma^t \rp)\enspace.
\]
Note that the derivative at $\rho^*=\rho^t$ equals $0$, which implies that
\begin{equation}\label{eq:der-0}
\lp.\frac{d\KL(\mu^*_x \| \mu^t_x)}{d\rho^t_i} \rp|_{\rho^t=\rho^*} = 0.
\end{equation}
By \eqref{eq:apply-chain-1}, for a general $\rho^t$, we have that
\[
\frac{d\KL(\mu^*_x \| \mu^t_x)}{d\rho^t_i}
= \frac{1}{2}\tr\lp((\Sigma^*-\Sigma^t)^\top \frac{d(\Sigma^t)^{-1}}{d\rho^t_i} \rp)
= \frac{1}{2}\tr\lp((\Sigma^*-\Sigma^t) \frac{d(\Sigma^t)^{-1}}{d\rho^t_i} \rp)~.
\]
Let us differentiate this again. We have that
\[
\frac{d^2\KL(\mu^*_x \| \mu^t_x)}{d\rho^t_i d\rho^t_j}
= \frac{d}{d\rho^t_j} \frac{1}{2}\tr\lp((\Sigma^* - \Sigma^t) \frac{d(\Sigma^t)^{-1}}{d\rho^t_i}\rp)
= \frac{1}{2} \tr\lp( 
- \frac{d\Sigma^t}{d\rho^t_j} \frac{d(\Sigma^t)^{-1}}{d\rho^t_i}
+ (\Sigma^* - \Sigma^t) \frac{d^2(\Sigma^t)^{-1}}{d\rho^t_i d\rho^t_j}
\rp)\enspace.
\]
In order to bound the left hand side above, we use \eqref{eq:bnd-operator}, to obtain that
\begin{align*}
&\tr\lp( 
- \frac{d\Sigma^t}{d\rho^t_j} \frac{d(\Sigma^t)^{-1}}{d\rho^t_i}
+ (\Sigma^* - \Sigma^t) \frac{d^2(\Sigma^t)^{-1}}{d\rho^t_i d\rho^t_j}
\rp)
\le n \lp\| - \frac{d\Sigma^t}{d\rho^t_j} \frac{d(\Sigma^t)^{-1}}{d\rho^t_i}
+ (\Sigma^* - \Sigma^t) \frac{d^2(\Sigma^t)^{-1}}{d\rho^t_i d\rho^t_j}\rp\|\\
&\le \lp\|\frac{d\Sigma^t}{d\rho^t_j} \rp\| \lp\| \frac{d(\Sigma^t)^{-1}}{d\rho^t_i} \rp\|
+ \lp\|(\Sigma^* - \Sigma^t) \rp\| \lp\| \frac{d^2(\Sigma^t)^{-1}}{d\rho^t_i d\rho^t_j}\rp\|
\end{align*}
We note that each of the components above is bounded by $\poly(n,1/B)$, using the formulas for $\Sigma^t$ and $(\Sigma^t)^{-1}$ that appear in \eqref{eq:sigma-formula} and \eqref{eq:sigma-inv-formula}. To conclude the first part of this lemma, we use the fact that the derivative equals $0$ at $\rho^t=\rho^*$ as shown in \eqref{eq:der-0}, and we integrate the second derivative along the path $\rho(\tau) = (1-\tau)\rho^* + \tau \rho^t$. Define $\mu_x^\tau$ the distribution obtained with correlations $\rho(\tau)$, then
\begin{align*}
\lp|\frac{d \KL(\mu^*_x\|\mu^t_x)}{d\rho^t_i}\rp|
&= \lp|\frac{d \KL(\mu^*_x\|\mu^t_x)}{d\rho^t_i} - \frac{d \KL(\mu^*_x\|\mu^*_x)}{d\rho^t_i}\rp|
= \lp|\int_0^1 \frac{d}{d\tau}\frac{d \KL(\mu^*_x\|\mu^\tau_x)}{d\rho^t_i} d\tau \rp|\\
&= \lp|\int_0^1 \sum_{j=1}^n \frac{d \KL(\mu^*_x\|\mu^\tau_x)}{d\rho^t_i d\rho^t_j} \frac{d\rho(\tau)_j}{d\tau} d\tau\rp|
= \lp|\int_0^1 \sum_{j=1}^n \frac{d \KL(\mu^*_x\|\mu^\tau_x)}{d\rho^t_i d\rho^t_j} (\rho^t-\rho^*) d\tau\rp|\\
&\le \int_0^1 \sum_{j=1}^n \lp| \frac{d \KL(\mu^*_x\|\mu^\tau_x)}{d\rho^t_i d\rho^t_j} \rp| \lp|\rho^t-\rho^*\rp| d\tau
\le \int_0^1 \sum_{j=1}^n \poly(n,1/B) \lp|\rho^t-\rho^*\rp| d\tau\\
&\le \poly(n,1/B) \|\rho^t-\rho^*\|_1
\le \sqrt{n} \poly(n,1/B) \|\rho^t-\rho^*\|_2.
\end{align*}
This concludes the first part of the lemma. The second part of this lemma is bounded similarly, by integrating, now over the first derivative: define by $\rho(\tau) = (1-\tau) \rho + \tau \rho'$. Then,
\begin{align*}
&\lp| \KL(\mu^*_x\|\mu^\rho_x) - \KL(\mu^*_x\|\mu^{\rho'}_x) \rp|
= \lp| \int_0^1 \frac{d\KL(\mu^*_x\|\mu^{\rho(\tau)}_x)}{d\tau} d\tau \rp|
= \lp| \int_0^1 \sum_{i=1}^n \frac{d\KL(\mu^*_x\|\mu^{\rho(\tau)}_x)}{d\rho(\tau)_i} \frac{d\rho(\tau)_i}{d\tau} d\tau \rp|\\
&= \lp| \int_0^1 \sum_{i=1}^n \frac{d\KL(\mu^*_x\|\mu^{\rho(\tau)}_x)}{d\rho(\tau)_i} (\rho'_i-\rho_i) d\tau \rp|
= \lp| \int_0^1 \sum_{i=1}^n \frac{d\KL(\mu^*_x\|\mu^{\rho(\tau)}_x)}{d\rho(\tau)_i} (\rho'_i-\rho_i) d\tau \rp| \\
&\le \int_0^1 \sum_{i=1}^n \lp|\frac{d\KL(\mu^*_x\|\mu^{\rho(\tau)}_x)}{d\rho(\tau)_i}\rp| |\rho'_i-\rho_i| d\tau
\le \int_0^1 \sum_{i=1}^n \poly(n,1/B) \|\rho(\tau)-\rho^*\|_2 |\rho'_i-\rho_i| d\tau\\
&\le \poly(n,1/B)\max_\tau \|\rho(\tau)-\rho^*\|_2\|\rho'_i - \rho_i\|_1\\
&\le \sqrt{n}\poly(n,1/B) \max(\|\rho-\rho^*\|_2, \|\rho'-\rho^*\|)\|\rho'_i - \rho_i\|_2\enspace.
\end{align*}
The last part of this lemma is obtained from the second part by substituting $\rho' = \rho^*$.
\end{proof}

\subsubsection{KL upper bounds the parameter difference}

\begin{lemma} \label{lem:kl-larger}
Let $\mu$ and $\mu'$ be two distribution. Assume that $\Var_{\mu}[y] = 1$ and further that $\Var_{\mu}[x_i] = \Var_{\mu'}[x_i] = 1$. Then,
\[
\KL(\mu'\| \mu)
\ge \max_i \frac{(\E_{\mu}[x_i y] - \E_{\mu'}[x_i y])^2}{2}.
\]
\end{lemma}
 We would like to lower bound $\KL(\mu'\| \mu)$ by $(\E_{\mu}[x_i y] - \E_{\mu'}[x_i y])^2$. Notice that by the data-processing inequality and by the chain rule for KL divergence,
\begin{equation}\label{eq:kl-diff}
\KL(\mu' \|\mu)
\ge \KL(\mu'_{x_i y} \|\mu_{x_iy})
= \KL(\mu'_{x_i} \| \mu_{x_i}) + \E_{x_i \sim \mu_{x_i}^{t+1}}[\KL(\mu'_{y\mid x_i} \| \mu_{y\mid x_i})]
= \E_{x_i \sim \mu_{x_i}^{t+1}}[\KL(\mu'_{y\mid x_i} \| \mu_{y\mid x_i})].
\end{equation}
We would like to compute the conditional distribution of $y$ given $x_i$, for both $\mu$ and $\mu'$. We use the following formula for conditional Gaussians: 
\begin{lemma} \label{lem:gauss-cond}
Let $Z$ and $W$ be jointly distributed Gaussian variables. Then,
\[
\E[W \mid Z] = \E W + \E[ZW] \Var[Z]^{-1}Z,
\]
and 
\[
\Var[W\mid Z] = \Var[W] - \E[WZ]^2\Var[Z]^{-1}.
\]
\end{lemma}
Using the formula of Lemma~\ref{lem:gauss-cond},
\[
\E[y\mid x_i] = \E[y] + \E[x_i y] \Var[x_i]^{-1} x_i,
\]
while
\[
\Var[y \mid x_i]
= \Var[y] - \E[x_iy]^2/\Var[x_i].
\]
Computing some values for $\mu$ and $\mu'$, we get that
\[
\E_{\mu}[y\mid x_i] = \E_{\mu}[x_i y] x_i~; \quad \E_{\mu'}[y\mid x_i] = \E_{\mu'}[x_i y] x_i; \Var_{\mu}[y\mid x_i] = 1 - \E_{\mu}[x_iy]^2 \le 1.
\]
We use the formula for the KL of two univariate Gaussians:
\begin{lemma} \label{lem:KL-formula}
For two univariate Gaussians $X_1,X_2$ with means $\mu_1,\mu_2$ and covariances $\sigma_1^2,\sigma_2^2$, respectively,
\[
\KL(X_1\|X_2)
= \log\frac{\sigma_2}{\sigma_1} + \frac{\sigma_1^2}{2\sigma_2^2} - \frac{1}{2}
+ \frac{(\mu_1-\mu_2)^2}{2\sigma_2^2}
\ge \frac{(\mu_1-\mu_2)^2}{2\sigma_2^2}
\]
\end{lemma}
\begin{proof}
The first equality is folklore and follows from a direct calculation. To get the inequality, notice that $\log\frac{\sigma_2}{\sigma_1} + \frac{\sigma_1^2}{2\sigma_2^2} - \frac{1}{2}$ is the KL of two Gaussians with variances $\sigma_1$ and $\sigma_2$ and the same means, hence it is nonnegative.
\end{proof}

we have that
\[
\KL(\mu'_{y\mid x_i} \| \mu_{y\mid x_i})
\ge \frac{(\E_{\mu'}[y\mid x_i] - \E_{\mu^{t}}[y\mid x_i])^2}{2\Var_{\mu^{t}}[y\mid x_i]}
\ge \frac{x_i^2 (\E_{\mu'}[x_i y] - \E_{\mu}[x_i y])^2}{2} \enspace.
\]
Substituting this in the right hand side of \eqref{eq:kl-diff} and using the fact that $\E_{\mu'}[x_i^2] = 1$, we get the desired result.

\subsubsection{Another upper bound on the KL}

\begin{lemma} \label{lem:ub-kl-2}
Let $\mu_x$ and $\mu'_x$ be two leaf distributions with parameters $\rho$ and $\rho'$, respectively, which are $(A,B)$-bounded, namely, for all $i$, $A \le \rho_i,\rho'_i\le 1-B$. Assume further that $\Var_\mu[x_i] = \Var_{\mu'}[x_i]=1$ for all $i$. Then,
\[
\KL(\mu'_x\|\mu_x) \ge cA^4 \max_i |\rho_i - \rho'_i|,
\]
where $c>0$ is a universal constant.
\end{lemma}

Below, we prove Lemma~\ref{lem:ub-kl-2}. We note that there exist $i,j$ such that $|\rho_i\rho_j - \rho'_i \rho'_j|$ is large. Indeed, let $i$ be the maximizer of $|\rho_i - \rho'_i|$. For the purpose of lower bounding $|\rho_i\rho_j - \rho'_i \rho'_j|$, we can assume that $\rho_i > \rho'_i$. Denote, let $\lambda = \rho_i / \rho'_i$. Denote $M = \max_j |\rho_j - \rho'_j|$. Notice that
\[
\lambda -1 = (\rho_i - \rho'_i)/ \rho'_i \ge \rho_i - \rho'_i = \max_l |\rho_l - \rho'_l| = M.
\]
Then, divide into cases:
\begin{itemize}
    \item If there exist two different values, $j,k$, such that $\rho_j/\rho'_j, \rho_k/\rho'_k \le \sqrt{1/\lambda}$. Then, 
    \[
    \frac{\rho_j'\rho_k'}{\rho_j\rho_k} \ge \lambda,
    \]
    hence
    \[
    \rho_j'\rho_k'-\rho_j\rho_k
    = \frac{\rho_j'\rho_k' - \rho_j\rho_k}{\rho_j\rho_k} \rho_j\rho_k
    \ge (\lambda-1) \rho_j\rho_k
    \ge M \rho_j\rho_k \ge MA^2.
    \]
    \item Otherwise, there is some $j$ such that $\rho_j/\rho'_j \ge \sqrt{1/\lambda}$. Then,
    \[
    \frac{\rho_i \rho_j}{\rho'_i\rho'_j} \ge \sqrt{\lambda}
    \ge \sqrt{M+1} \ge 1 + cM,
    \]
    for some universal constant $c>0$. This implies that
    \[
    \rho_i \rho_j - \rho'_i\rho'_j
    = \frac{\rho_i \rho_j - \rho'_i\rho'_j}{\rho_i\rho_j} \rho_i\rho_j
    \ge cM \rho_i\rho_j
    \ge cMA^2.
    \]
\end{itemize}
This derives that there exist some $j,k$ such that $|\rho_j\rho_k-\rho'_j\rho'_k|\ge cA^2 M$ for some universal constant $c>0$. This implies that
\[
|\E_{\mu}[x_jx_k] - \E_{\mu'}[x_jx_k]| \ge cA^2 M.
\]
From this point onwards, the proof is analogous to the proof of Lemma~\ref{lem:kl-larger}. We start by arguing that by the data processing inequality,
\[
\KL(\mu'_x\|\mu_x) \ge \KL(\mu'_{x_jx_k}\|\mu_{x_jx_k}).
\]
Then, we lower bound $\KL(\mu'_{x_jx_k}\|\mu_{x_jx_k})$ using the fact that $|\E_{\mu}[x_jx_k] - \E_{\mu'}[x_jx_k]|$, exactly the same way as in Lemma~\ref{lem:kl-larger} we lower bounded $\KL(\mu'_{x_iy}\|\mu_{x_iy})$ using the fact that $|\E_{\mu}[x_iy] - \E_{\mu'}[x_iy]|$ is large. The proof follows.

\subsubsection{Proof of Lemma~\ref{lem:kl-to-l2}}

The proof follows directly from Lemma~\ref{lem:comp-two-kls} and Lemma~\ref{lem:ub-kl-2} that were proven in the previous subsections.

\subsection{Contraction in KL - population EM}

In this section, we prove that the KL between the true and the current model contracts by a constant factor in each iteration, for the populatoin EM. Later, we will show how to derive the same results for the sample-EM as well.

\begin{proposition} \label{prop:contraction-pop}
Let $\rho^t$ denote iterate $t$ of the population-EM, and assume that the iterates are $(A,B)$-bounded. Then, 
\[
\KL(\mu_x^*\|\mu_x^{t+1})
\le (1-\kappa) \KL(\mu_x^*\|\mu_x^t),
\]
where $\kappa = A^{C_1}B^{C_2}/n^{C_3}$ for some universal constants $C_1,C_2,C_3>0$.
\end{proposition}

We describe the proof of Proposition~\ref{prop:contraction-pop}, in the subsections below.

\subsubsection{First bound on the KL difference}
We will prove the following lemma:
\begin{lemma}\label{lem:KL-step}
\[
\mathrm{KL}(\mu_x^* \| \mu_x^t) - \mathrm{KL}(\mu_x^* \| \mu_x^{t+1}) \ge \mathrm{KL}(\mu^{t+1} \|\mu^t) 
\]
\end{lemma}
First of all, we use the following well-known behavior of the EM:
\begin{lemma}[Folklore.]\label{lem:EM-folklore-lb}
\begin{equation}\label{eq:normal-EM-bnd}
\mathrm{KL}(\mu_x^* \| \mu_x^t) - \mathrm{KL}(\mu_x^* \| \mu_x^{t+1}) \ge 
\E_{x,y\sim \mu^{t,*}}[\log \Pr_{\mu^{t+1}}[(x,y)]] - \E_{x,y\sim \mu^{t,*}}[\log \Pr_{\mu^{t}}[(x,y)]]\enspace.
\end{equation}
\end{lemma}

\begin{proof}[Proof of Lemma~\ref{lem:KL-step}]
We conclude Lemma~\ref{lem:KL-step} by analyzing the right hand side of \eqref{eq:normal-EM-bnd}. Notice that by Lemma~\ref{lem:covariance-conserve},
\begin{align*}
\E_{x,y\sim \mu^{t,*}}[\log \Pr_{\mu^{t+1}}[(x,y)]]
= \E_{x,y\sim \mu^{t,*}}\lp[\log \Pr_{\mu^{t+1}}[y] + \sum_i \log \Pr_{\mu^{t+1}}[x_i\mid y]\rp]\\
= \E_{x,y\sim \mu^{t+1}}\lp[\log \Pr_{\mu^{t+1}}[y] + \sum_i \log \Pr_{\mu^{t+1}}[x_i\mid y]\rp]
= \E_{x,y\sim \mu^{t+1}}[\log \Pr_{\mu^{t+1}}[(x,y)]]\enspace.
\end{align*}
Similarly,
\[
\E_{x,y\sim \mu^{t,*}}[\log \Pr_{\mu^{t}}[(x,y)]]
= \E_{x,y\sim \mu^{t+1}}[\log \Pr_{\mu^{t}}[(x,y)]]\enspace.
\]
Hence, the right hand side of \eqref{eq:normal-EM-bnd} equals
\[
\E_{x,y\sim \mu^{t+1}}[\log \Pr_{\mu^{t+1}}[(x,y)] - \log \Pr_{\mu^{t}}[(x,y)]]
= \KL(\mu^{t+1}\|\mu^t),
\]
as required.
\end{proof}

\subsubsection{Bounding the KL by the covariance difference}

We apply Lemma~\ref{lem:kl-larger} with $\mu = \mu^t$ amd $\mu'=\mu^{t+1}$. Note that this lemma requires that $\Var_{\mu^t}[y] = 1$. Indeed, for the purpose of proving Proposition~\ref{prop:contraction-pop}, we can assume that, since this proposition only argues about the $x$-marginal, and since the correlation parameters $\rho^t$ and $\rho^{t+1}$ are not affected by $\Var_{\mu^t}[y]$. Then, Lemma~\ref{lem:kl-larger}, in combination with Lemma~\ref{lem:KL-step}, imply that
\begin{equation}\label{eq:bnd-pop-2}
\mathrm{KL}(\mu_x^* \| \mu_x^t) - \mathrm{KL}(\mu_x^* \| \mu_x^{t+1}) \ge \mathrm{KL}(\mu^{t+1} \|\mu^t)
\ge \max_i \frac{\lp(\E_{\mu^t}[x_i y] - \E_{\mu^{t+1}}[x_i y] \rp)^2}{2}\enspace.
\end{equation}

\subsubsection{Minimal eigenvalue bound}

We would like to lower bound the right hand side of \eqref{eq:bnd-pop-2}. To do so, we start with an auxiliary lemma, bounding the minimal eigenvalue of some matrix.

\begin{lemma} \label{lem:min-eigenvalue}
Let $U$ be a matrix with $U_{ij} = u_j$ if $i\ne j$ and $\sum_{k\ne i} u_k$ otherwise, where $u_i>0$ for all $i$. 
Then,
\[
\sigma_{\min}(U) := \max_{a \colon \|a\|_2 = 1} \|Ua\|_2
\ge \frac{ (\min_i u_i)^3}{\|u\|_2\|u\|_1}
\cdot \frac{(n-2)^3}{128 n^3} \enspace.
\]
\end{lemma}

Our goal is to show that the minimal singular value of this matrix is bounded away from $0$. In particular, denote $Ua = K$ for some unit vector $a=(a_1,\dots,a_n)$ and $K=(K_1,\dots,K_n)$ and our goal is to lower bound $\|K\|$. The proof has multiple ingredients. First, we denote by $s = \sum_i u_i$ and $t = \sum_i a_i u_i$. We have that 
\[
K_i = a_i \sum_{j\ne i} u_j + \sum_{j \ne i} a_j u_j =
t + a_i(s - 2u_i),
\]
hence
\begin{equation} \label{eq:Ki}
t + a_i(s-2u_i) - K_i = 0,
\end{equation}
In particular, we have for all $i\ne j$
\[
a_i(s-2u_i) - K_i = a_j(s-2u_j) - K_j.
\]
Let us assume that $u_1 \ge u_2 \ge \cdots \ge u_n \ge$.

\paragraph{If $|s-2u_1|$ is small.}

\begin{lemma}\label{lem:not-close-to-half}
Assume that $|s-2u_1| \le (n-2)u_n/(8n)$. Then, $\max_i |K_i| \ge (n-2)u_n/(16n)$.
\end{lemma}
\begin{proof}
Denote $\epsilon = |s-2u_1|$ and let us lower bound $\|K\|$.
Then, we have for all $i$,
\[
a_i(s-2u_i)-K_i = a_1\epsilon - K_1.
\]
Hence, for all $i>1$,
\[
a_i = \frac{a_1 \epsilon - K_1 + K_i}{s-2u_i}.
\]
Denote $K = \max_i K_i$. Then, \[
|a_i| = \lp| \frac{a_1 \epsilon - K_1 + K_i}{s-2u_i}\rp|
\le \frac{|\epsilon|+2K}{s-2u_i}.
\]
Notice that $s-2u_i \ge s - u_1 - u_2 \ge (n-2)u_n$. Then,
\[
|a_i| \le \frac{|\epsilon|+2K}{(n-2)u_n}.
\]
Denote the right hand side by $m$. Then,
\[
|t| = |\sum_i a_i u_i|
\ge |a_1 u_1| - \sum_{i>1}|a_i| u_i \ge |a_1|u_1 - m \sum_{i>1} u_i.
\]
Further, recall that $\sum_i a_i^2 = 1$ hence $a_1^2 = 1 - \sum_{i>1} a_i^2 \ge 1 - nm$ hence $|a_i| \ge \sqrt{1-nm} \ge 1/2$ assuming that $m < 1/(2n)$. Hence, $t \ge u_1/2 - m \sum_{i>1} u_i \ge u_1/2 - mnu_1 \ge u_1/4$ assuming that $m < 1/(4n)$. We derive that 
\[|K_1| = |t + a_1(s-2u_1)| 
\ge |t| - |a_1||s-2u_1|
\ge |t| - |\epsilon|
\ge u_1/4 - |\epsilon|.
\]
In particular, if $|\epsilon| \le u_1/8$ then $K_1 \ge u_1/8$. Recall that we further assumed that $m \le 1/(4n)$ which is satisfied whenever
\[
|\epsilon| \le \frac{(n-2)u_n}{8n}
\]
and
\[
K \le \frac{(n-2)u_n}{16n}.
\]
We derive that if 
\[
|\epsilon| \le 
\frac{(n-2)u_n}{8n}
\]
then 
\[
K \ge \frac{(n-2)u_n}{16n}.
\]
\end{proof}

\paragraph{$t$ is large or $K_i$ is large.}
\begin{lemma}\label{lem:t-large}
Assume that $|s-2u_1| \ge \epsilon$. Then, either $\sum_i |K_i| \ge \epsilon/2$ or $t \ge \epsilon/2n$.
\end{lemma}

\begin{proof}
We assume that $|s-2u_1| \ge \epsilon$. Then, $|s-2u_i|\ge \epsilon$ for all $i$. Further, recall from \eqref{eq:Ki} that $|a_i(s-2u_i) + t| = |K_i|$, which implies that $|K_i| \ge |a_i(s-2u_i)| - |t| \ge |a_i|\epsilon - t$. Then, $\sum_i |K_i| \ge \epsilon \sum_i |a_i| - nt = \epsilon \|a\|_1 - nt \ge \epsilon \|a\|_2 - nt = \epsilon - nt$. In particular, either $t \ge \epsilon / 2n$ or $\sum_i |K_i| \ge \epsilon/2$.
\end{proof}

\paragraph{Assuming that $u_1 < s/2$.}
\begin{lemma}\label{lem:less-than-half}
Assume that $t \ne 0$ and that $s-2u_1 \ge \epsilon$. Then, $\|K\|_2 \ge \epsilon t/\|u\|_2$.
\end{lemma}
\begin{proof}
From \eqref{eq:Ki}, for all $i$,
\[
a_i = \frac{K_i-t}{s-2u_i}.
\]
From definition of $t$,
\begin{equation} \label{eq:t-equals}
t = \sum_i a_i u_i
= \sum_i \frac{u_i K_i - u_i t}{s-2u_i} \le \sum_i \frac{u_i K_i - u_i t}{s-2u_i}
\le \sum_i \frac{|a_i K_i|}{s-2u_i}
\le \sum_i \frac{|u_i K_i|}{\epsilon}
\le \frac{\|u\|_2\|K\|_2}{\epsilon}
\end{equation}
Hence
\[
\|K\|_2 \ge \epsilon t / \|u\|_2.
\]
\end{proof}

\paragraph{The case that $u_1 > s/2$.}
\begin{lemma} \label{lem:more-than-half}
Assume that $2u_1 - s \ge \epsilon$ and that $t \ne 0$. Then,
\[
\|K\|_2 \ge \frac{2(s-u_1-u_2) \epsilon |t|}{\|u\|_2 s}\enspace.
\]
\end{lemma}
\begin{proof}
Using the first few equalities of \eqref{eq:t-equals} we get that
\[
t = \sum_i \frac{u_i K_i - u_i t}{s-2u_i}\enspace.
\]
or equivalently,
\[
1 = \sum_i \frac{u_i K_i/t - u_i}{s-2u_i}
= -\sum_i \frac{u_i}{s-2u_i} + R, \quad \text{where } R = 
\sum_i \frac{u_i K_i}{t(s-2u_i)}.
\]
Then, 
\[
\sum_{i>1} \frac{u_i}{s-2u_i} - R
= \frac{u_1}{2u_1-s} - 1
= \frac{s-u_1}{2u_1-s}
= \sum_{i>1} \frac{u_i}{2u_1-s}.
\]
hence
\[
-R = \sum_i u_i \lp( \frac{1}{2u_1-s} - \frac{1}{s-2u_i} \rp)
= 2\sum_i u_i  \frac{s-u_1-u_i}{(s-2u_i)(2u_1-s)}
\ge 2 \sum_i u_i \frac{s-u_1-u_2}{s^2}
= 2 \frac{s-u_1-u_2}{s}.
\]
We derive that
\[
\frac{2(s-u_1-u_2)}{s}
\le |R|
\le \frac{1}{|t|} \lp|\sum_i \frac{u_i K_i}{|s-2u_i|} \rp|
\le \frac{1}{|\epsilon t|} \lp| \sum_i u_i K_i \rp|
\le \frac{\|u\|_2 \|K\|_2}{|\epsilon t|}.
\]
In particular,
\[
\|K\|_2 \ge \frac{2(s-u_1-u_2) \epsilon |t|}{\|u\|_2 s}\enspace.
\]
\end{proof}

We would like to culminate the proof of the main lemma. From Lemma~\ref{lem:not-close-to-half} we can assume that $|s-2u_1| \ge (n-2)u_n/8n := \epsilon$, otherwise 
\[
\|K\|_2 \ge \max_i K_i \ge (n-2)u_n/(16n)
\]
and the proof follows.
From Lemma~\ref{lem:t-large} we can assume that $|t| \ge \epsilon/2n$,
otherwise
\[
\|K\|_2 \ge \|K\|_1/\sqrt{n} \ge \epsilon/(2\sqrt{n})
\]
and the proof follows.
Lastly, from Lemma~\ref{lem:less-than-half} and Lemma~\ref{lem:more-than-half} we can derive that
\[
\|K\|_2 
\ge \frac{\epsilon |t|}{\|u\|_2}\min\lp(1,\ \frac{2(s-u_1-u_2)}{\|u\|_2 s} \rp) 
\ge \frac{\epsilon |t| (n-2)u_n}{\|u\|_2s}
\ge \frac{\epsilon^2 (n-2)u_n}{2n\|u\|_2s}
= \frac{(n-2)^3 u_n^3}{128n^3\|u\|_2s}\enspace.
\]

\subsubsection{Bounding the difference in covariance}

We are ready to bound the right hand side of \eqref{eq:bnd-pop-2}. This is stated in the following lemma:

\begin{lemma} \label{lem:move}
Assume that the iterates of the population EM are $(A,B)$ bounded, let $t\ge 0$. Then, 
\[
\lp(\E_{\mu^{t+1}}[x_i y] - \E_{\mu^{t}}[x_i y]\rp)^2
\ge 
c\frac{A^{12}(1-B)^8}{n^{11}}
\sum_i (\rho^t_i - \rho^*_i)^2\enspace,
\]
where $c>0$ is a universal constant.
\end{lemma}

We note that Lemma~\ref{lem:move} together with \eqref{eq:bnd-pop-2} imply that
\begin{equation}\label{eq:bnd-pop-3}
\KL(\mu_x^*\|\mu_x^t)
- \KL(\mu_x^*\|\mu_x^{t+1})
\ge \frac{C A^{C_1} B^{C_2}}{n^{C_3}} \sum_{i=1}^n \lp( \rho_i^t - \rho_i^*\rp)^2\enspace.
\end{equation}

Below, we prove Lemma~\ref{lem:move}. In the proof below, we use $a,b$ to denote bounds on the iterates: $a \le \rho^t_i \le 1-b$ for all $t$ and $i$. Further, we use $a^*$ and $b^*$ to denote bounds on $\rho_i^*$: $a^* \le \rho_i^* \le b^*$.
We recall that we assume $\sigma_{x_i}^t = \sigma_y^t = 1$ for all $i$. By Lemma~\ref{lem:update}, we derive that 
\[
\E_{\mu^{t+1}}[x_i y] =
\lambda_i^t + \sum_{j\ne i} \rho^*_i\rho^*_j \lambda_j^t.
\]
We can use the same lemma to derive $\E_{\mu^t}[x_i y]$: indeed, if we apply this lemma in a different scenario where the underlying distribution is $\mu^t$ (i.e. substituting $\mu^*=\mu^t$), then $\mu^{t+1}=\mu^t$ as well. Hence, if we substitute $\mu^*=\mu^t$ and $\mu^{t+1} = \mu^t$ in this lemma we get that
\[
\E_{\mu^{t}}[x_i y] =
\lambda_i^t + \sum_{j\ne i} \rho^t_i\rho^t_j \lambda_j^t.
\]
We get that 
\[
(\lambda^t_i)^2 \lp(\E_{\mu^{t+1}}[x_i y] - \E_{\mu^{t}}[x_i y]\rp)^2
= \lp(\sum_{j \ne i} \rho^*_i \rho^*_j \lambda^t_i \lambda^t_j - \rho^t_i \rho^t_j \lambda^t_i \lambda^t_j \rp)^2\enspace.
\]
Substituting $u_i = \lambda^t_i \rho^*_i$ and $v_i = \lambda^t_i \rho^t_i$, we derive that
\[
(\lambda^t_i)^2 \lp(\E_{\mu^{t+1}}[x_i y] - \E_{\mu^{t}}[x_i y]\rp)^2
= \lp(\sum_{j \ne i} u_i u_j - v_i v_j \rp)^2\enspace.
\]
We will prove the following lemma:
\begin{lemma} \label{lem:using-Jaco}
If $u = (u_1,\dots,u_n)$, $v = (v_1,\dots,v_n)$ where $u_i,v_i>0$ for all $i$, then
\[
\sum_i \lp(\sum_{j \ne i} u_i u_j - v_i v_j \rp)^2 
\ge \sigma_{\min}(A)^2 \cdot \|u-v\|_2^2
\]
where $A_{ij} = (u_i + v_i)/2$ if $i \ne j$ and $A_{ii} = \sum_{j \ne i} (u_i + v_i)/2$.
\end{lemma}
We will first use Lemma~\ref{lem:using-Jaco} to complete the proof of Lemma~\ref{lem:move}, and then we'll prove Lemma~\ref{lem:using-Jaco}.

\begin{proof}[Proof of Lemma~\ref{lem:move}]
First,
\[
\sum_i (\rho^t_i - \rho^*_i)^2
\le \lp(\min_i \lambda^t_i\rp)^{-2} \sum_i (\lambda_i^t)^2 (\rho^t_i - \rho^*_i)^2
= \lp(\min_i \lambda^t_i\rp)^{-2} \sum_i (u_i - v_i)^2.
\]
Then, by Lemma~\ref{lem:using-Jaco},
\begin{align*}
\sum_i (u_i-v_i)^2 \le
\frac{1}{\sigma_{\min}(A)^2}
\sum_i \lp(\sum_{j \ne i} u_i u_j - v_i v_j \rp)^2
= \frac{1}{\sigma_{\min}(A)^2}\sum_i (\lambda_i^t)^2 \lp(\E_{\mu^{t+1}}[x_i y] - \E_{\mu^{t}}[x_i y]\rp)^2 \\
\le \frac{1}{\sigma_{\min}(A)^2}\sum_i \lp(\E_{\mu^{t+1}}[x_i y] - \E_{\mu^{t}}[x_i y]\rp)^2,
\end{align*}
as $\lambda^t_i \le 1$.
We conclude that
\begin{equation}\label{eq:starting-iter-bound}
\sum_i (\rho^t_i - \rho^*_i)^2
\le \frac{1}{(\min_i \lambda^t_i)^2\sigma_{\min}(A)^2}\sum_i \lp(\E_{\mu^{t+1}}[x_i y] - \E_{\mu^{t}}[x_i y]\rp)^2 \enspace.
\end{equation}
We would like to expand on this bound. Using the expression of $\lambda^t_i$ from \eqref{eq:conditional}, we derive that
\[
\lambda^t_i = \frac{\rho^t_i/(1-(\rho^t_i)^2)}{1 + \sum_{i=1}^n (\rho^t_j)^2/(1-(\rho^t_j)^2)}\enspace.
\]
From the assumption that $0 < a \le \rho_i < b \le 1$, we derive that
\[
\rho^t_i/(1-(\rho^t_i)^2) \ge \rho^t_i \ge a,
\]
while
\[
(\rho^t_j)^2/(1-(\rho^t_j)^2)
\le 1/(1-(\rho^t_j)^2)
\le 1/(1-b^2).
\]
Therefore,
\begin{equation}\label{eq:lb-lambda-t}
\lambda^t_i \ge \frac{a}{1 + n/(1-b^2)}
= \frac{a(1-b^2)}{1-b^2 + n}
\ge \frac{a(1-b^2)}{n+1}\enspace.
\end{equation}
Further, we want to bound $\sigma_{\min}(A)$. Using Lemma~\ref{lem:min-eigenvalue} applied on $u \gets (u+v)/2$, it suffices to estimate properties of $(u+v)/2$. First of all, using the assumption $a^* \le \rho^*_i \le b^*$,
\[
\frac{u_i + v_i}{2}
\ge \frac{u_i}{2}
= \frac{\lambda^t_i \rho^*_i}{2}
\ge \frac{a(1-b^2)a^*}{2n+2}\enspace.
\]
Further, we have that
\[
\frac{\lambda^t_i (\rho^t_i + \rho^*_i)}{2} \le 1,
\]
which implies that
\[
\|(\v u+\v v)/2\|_1 \le n,\ \|(\v u + \v v)/2\|_2 \le \sqrt{n}\enspace.
\]
Using Lemma~\ref{lem:min-eigenvalue}, the minimal eigenvalue of $A$ is at least
\[
\frac{ (\min_i (u_i+v_i)/2)^3}{\|(u+v)/2\|_2\|(u+v)/2\|_1}
\cdot \frac{(n-2)^3}{128 n^3}
\ge \frac{a^3(1-b^2)^3 (a^*)^2}{(2n+2)^3 n^{3/2}} \cdot \frac{(n-2)^3}{128 n^3}
\ge c \frac{a^3(1-b)^3 (a^*)^2}{n^{9/2}} \enspace,
\]
where $c>0$ is some universal constant.
Substituting this and the lower bound on $\lambda^t_i$ (Eq.~\eqref{eq:lb-lambda-t}), in \eqref{eq:starting-iter-bound}, 
we derive that
\[
\lp(\E_{\mu^{t+1}}[x_i y] - \E_{\mu^{t}}[x_i y]\rp)^2
\ge 
c'\frac{a^8(1-b)^8 (a^*)^4}{n^{11}}
\sum_i (\rho^t_i - \rho^*_i)^2\enspace
\]
where $c'>0$ is a universal constant.
\end{proof}

Now we will prove Lemma~\ref{lem:using-Jaco}. Define for $i=1,\dots,n$ functions $p_i(w) = \sum_{j \ne i} w_i w_j$ where $w=(w_1,\dots,w_n)$. Let $p(w) = (p_1(w),\dots,p_n(w))$, then we want to show that $\|p(u) - p(v)\|^2 \ge  C \cdot \|u-v\|^2$.
We use the following simple observation:
\begin{lemma}\label{lem:one-dim}
Let $p(x)$ be a second-degree polynomial of one variable $x$ and let $p'(x)$ denote its derivative. Then, for any $s,t \in \mathbb{R}$, \[
p(s) - p(t) = p'((s+t)/2) (s-t).
\]
\end{lemma}
\begin{proof}
Let $p(x) = ax^2 + bx + c$. Then, $p'(x) = 2ax + b$. Then,
\[
p(s) - p(t) = a(s^2-t^2) + b(s-t)
\]
while
\[
p'((s+t)/2) (s-t)
= (a(s+t) + b)(s-t)
= as^2 - at^2 + b(s-t) = p(s)-p(t).
\]
\end{proof}
We use the following corollary:
\begin{lemma}
Let $p \colon \mathbb{R}^n \to \mathbb{R}^n$ be a system of second-degree polynomials, namely, $p(x) = (p_1(x),\dots,p_n(x))$ where each $p_i$ is a second degree polynomial. Then,
\begin{equation}\label{eq:Jacob-high-dim}
\v p(\v s)-\v p(\v t) = J^p|_{(\v s+\v t)/2}(\v s-\v t), 
\end{equation}
where $J^{\v p}|_{\v x}$ is the Jacobian matrix of $\v p$ evaluated at $\v x$, namely,
\[
\lp(J^{\v p}|_{\v x}\rp)_{ij} = \frac{dp_i(\v x)}{dx_j}\enspace.
\]
Consequently,
\begin{equation}\label{eq:jacob-sigma-min}
\|\v p(s) - \v{p}(t)\|
\ge \sigma_{\min}(J^{\v p}|_{\v x}) \|s-t\|,
\end{equation}
where $\sigma_{\min}(A)$ is the minimal singular value of a matrix $A$.
\end{lemma}
\begin{proof}
To show the first part, look at the path $\gamma(\lambda) = \lambda \v t + (1-\lambda) \v s$. Then, $\gamma(1) = \v t$ while $\gamma(0) = \v s$. Applying Lemma~\ref{lem:one-dim} on the polynomials $p_i(\gamma(\lambda))$ while substituting $s=0$ and $t=1$, one obtains \eqref{eq:Jacob-high-dim}. Then, \eqref{eq:jacob-sigma-min} follows from taking the norm in both sides of the above equality, and using the fact that for a matrix $A$ and a vector $\v v$, $\|A\v v\| \ge \sigma_{\min}(A) \|\v v\|$.
\end{proof}

To complete the proof of Lemma~\ref{lem:using-Jaco}, notice that the matrix $A$ in this lemma is the Jacobian of $\v p$ evaluated at $(u+v)/2$.

\subsubsection{Bounding the parameter difference by the KL: conclusion of Proposition~\ref{prop:contraction-pop}}

To conclude Proposition~\ref{prop:contraction-pop}, we would like to bound the right hand side of \eqref{eq:bnd-pop-3}. Recall that this right hand side contains the term $\|\rho^t - \rho^*\|_2^2$. By Lemma~\ref{lem:comp-two-kls}, this is lower bounded by
\[
\frac{B^{C_1}}{n^{C_2}} \KL(\mu_x^*\|\mu_x^t).
\]
In combination with \eqref{eq:bnd-pop-3}, this concludes the proof of Proposition~\ref{prop:contraction-pop}.

\subsection{Contraction for the sample EM}

While we have proven in Proposition~\ref{prop:contraction-pop} that the sample-EM contracts, we now prove that the population EM contracts as well:
\begin{lemma}\label{lem:contract-sample}
	Let $\mu^t$ denote the $t$'th iterate of the sample EM, and assume that the sample is $\eta$-representative. Assume that $\eta \le \min\lp(1,\sqrt{\KL(\mu_x^*\|\mu_x^t)}\rp)(AB/n)^C$ for some sufficiently large universal constant $C>0$. Then,
	\[
	\KL(\mu_x^*\|\mu_x^{t+1}) \le (1-\kappa') \KL(\mu_x^*\|\mu_x^t),
	\] 
	where $\kappa \ge (AB/n)^{C'}$ for some universal constant $C'>0$.
\end{lemma}

\begin{proof}
Below, we prove Lemma~\ref{lem:contract-sample}.
To argue about that, we would use the Lemma~\ref{lem:sample-close-to-pop} which argues that the sample step is close to the population step, in parameter distance, and further, we will use Lemma~\ref{lem:comp-two-kls} to argue that this implies that the KL distance after one population step is close to that after one sample step. To be more concrete, let $\rho^t$ denote the $t$'th iterate of the sample EM, let $\tilde{\rho}^{t+1}$ denote the result of applying one step of the population EM on $\rho^t$ and let $\rho^{t+1}$ denote iterate $t+1$ of the sample EM, namely, $\rho^{t+1}$ is obtained from $\rho^t$ via one iterate of the sample EM. Let $\tilde{\mu}_x^{t+1}$ and $\mu_x^{t+1}$ denote the corresponding distributions. Then, from Lemma~\ref{lem:sample-close-to-pop}, we obtain that for all $i$,
\[
|\tilde{\rho}^{t+1} - \rho^{t+1}| \le \eta \poly(1/A,1/B,n).
\]
By Lemma~\ref{lem:comp-two-kls} this implies that
\begin{align*}
&\lp| \KL(\mu_x^*\|\mu^{t+1}_x) - \KL(\mu_x^*\|\tilde\mu^{t+1}_x) \rp|\\
&\le \poly(n,1/A,1/B) \|\rho^{t}-\rho^{t+1}\| \max\lp(\|\rho^{t+1}-\rho^*\|, \|\tilde{\rho}^{t+1} - \rho^*\|\rp)\\
&\le \poly(n,1/A,1/B) \eta \max\lp(\|\rho^{t+1}-\rho^*\|, \|\tilde{\rho}^{t+1} - \rho^*\|\rp)\\
&\le \poly(n,1/A,1/B) \eta \max\lp(\|\rho^t-\rho^*\| + \|\rho^t-\rho^{t+1}\|, \|\tilde{\rho}^{t+1} - \rho^*\|\rp)\\
&= \poly(n,1/A,1/B) \eta \lp(\|\rho^t-\rho^*\| + \|\rho^t-\rho^{t+1}\|\rp)\\
&\le\poly(n,1/A,1/B) \eta \lp(\|\rho^t-\rho^*\| + \poly(n,1/A,1/B) \eta\rp).
\end{align*}
By Lemma~\ref{lem:kl-to-l2}, the right hand side is at least
\[
\poly(n,1/A,1/B) \eta \lp(\sqrt{\KL(\mu_x^*\|\mu_x^t)} + \eta\rp)
\]
By the assumption of Lemma~\ref{lem:contract-sample} on $\eta$, this is at least
\[
2\poly(n,1/A,1/B)\eta \sqrt{KL(\mu_x^*\|\mu_x^t)}.
\]
Again, by the assumption on this lemma on $\eta$, if the constant $C$ in this assumption is sufficiently large, then the last term is bounded by
\[
KL(\mu_x^*\|\mu_x^t) (AB/n)^{C_1},
\]
where $C_1$ can be chosen arbitrarily large (if $C$ is).
Combining with Proposition~\ref{prop:contraction-pop}, this concludes the proof: indeed, let $\kappa$ be the constant from the proposition. Then,
\[
\KL(\mu_x^*\|\mu_x^{t+1})
\le \KL(\mu_x^*\|\tilde\mu_x^{t+1}) + \lp| \KL(\mu_x^*\|\mu^{t+1}_x) - \KL(\mu_x^*\|\tilde\mu^{t+1}_x) \rp|
\le (1-\kappa) \KL(\mu_x^*\|\mu_x^t) + KL(\mu_x^*\|\mu_x^t) (AB/n)^{C_1}.
\]
Recall that $\kappa = (AB/n)^{C_2}$ for some universal constant $C_2$. Since we can select $C_1$ to be arbitrarily large, we can ensure that $(AB/n)^{C_1} \le \kappa/2$. This implies that
\[
\KL(\mu_x^*\|\mu_x^{t+1})
\le (1-\kappa/2) \KL(\mu_x^*\|\mu_x^t)
\]
which suffices to conclude the proof.
\end{proof}

\subsection{Concluding the proof of Theorem~\ref{thm:general-finite}}

\begin{proof}[Proof of Theorem~\ref{thm:general-finite}]
	We will assume that the sample is $\eta$-representative, for $\eta \le (AB/n)^C \epsilon$ for a sufficiently large $C>0$. This can be guaranteed if the sample size is $m = O(\log(1/\delta)/\eta^2) = \poly(n,1/A,1/B) \log(1/\delta)/\epsilon^2$. With this value of $\eta$, Lemma~\ref{lem:contract-sample} guarantees that as long as $\KL(\mu_x^*\|\mu_x^t) \ge \epsilon (AB/n)^{C_1}$, then, 
	\[
	\KL(\mu_x^*\|\mu_x^{t+1}) \le (1-\kappa') \KL(\mu_x^*\|\mu_x^t),
	\]
	where $C_1$ can be chosen arbitrarily large and $\kappa'$ is polynomial in $A,B,1/n$. The initial value of $\KL(\mu_x^*\|\mu_x^0)$ is bounded by $\poly(1/A,1/B,n)\|\rho^0-\rho^*\|^2 \le n \poly(1/A,1/B,n)$, from Lemma~\ref{lem:kl-to-l2}.
	Hence, the number of iterations that it takes $\KL(\mu_x^*\|\mu_x^t)$ to drop below $\epsilon (AB/n)^{C_1}$ is bounded by $\poly(n,1/A,1/B) \log(1/\epsilon)$. Using Lemma~\ref{lem:kl-to-l2}, if $C_1$ is sufficiently large then once the KL divergence drops below that value, $|\rho^t_i-\rho^*_i| \le \epsilon$ for all $i$.
	
	Lastly, notice that for all $t\ge 1$ and all $i$, $\sigma^t_i = \hat{\sigma}_i$, from Lemma~\ref{lem:update-sample}. Since the sample is $\eta \le \epsilon$ representative, this implies that $|\sigma^t_i - \sigma^*_i| \le \eta \le \epsilon$. This concludes the proof.
\end{proof}
\section{General tree model} \label{sec:tree-full-pr}

We consider a multivariate Gaussian  latent-tree distribution, that is characterized by a tree $G = (V,E)$. Each vertex $u \in V$ corresponds to a random variable $z_u$. Suppose the total number of nodes in the tree is $n+m$. From now on, we might refer to the node itself as the random variable, when it is clear from the context what we mean.
We define a joint probability distribution over the nodes as follows:
\begin{equation}\label{eq:product}
\Pr[z_1,\ldots,z_{n+m}] = \prod_{(i,j) \in E} \Pr[z_i,z_j] 
\end{equation}
The vertices $V$ are divided into two groups: nodes of degree $1$(leaves), denoted by $x_1,\ldots,x_n$ and nodes of degree at least $2$ (internal nodes), denoted by $y_1,\ldots,y_m$. When we want to refer to some node in the tree without wanting to specify whether it is a leaf or an internal node, we will use the symbol $z$. The edges $e \in E$ are divided into two groups: the ones that are between $y$ nodes, called \emph{internal edges} and denoted by $E_{yy}$, and the ones between one $y$ node and one $x$ node, called \emph{boundary edges} and denoted by $E_{xy}$.
Each leaf $x_i$ has a variance $\sigma_{x_i}^2$ and so it corresponds to a random variable $x_i \sim N(0,\sigma_{x_i}^2)$. Likewise, each internal node $y_i$ has a variance $\sigma_{y_i}^2$ and so it corresponds to a random variable $y_i \sim N(0,\sigma_{y_i}^2)$. 
For each edge $(y_i,y_j) \in E_{yy}$ we define the correlation $\rho_{y_iy_j}$ of variables $y_i,y_j$ and similarly, for each edge $(x_i,y_j) \in E_{xy}$ we define the correlation $\rho_{x_iy_j}$ between $x_i,y_j$. 
It can be shown that in a Gaussian Graphical Model that satisfies the product decomposition ~\ref{eq:product}, it can be shown  that these parameters are enough to specify the joint distribution over all $x$ and $y$.
%In particular, for two arbitrary nodes $z_i,z_j$ in the tree, define $P(z_i,z_j)$ to be the set of edges in the unique path connecting $z_i$ to $z_j$. Then, it can be proved \vardis{proof in appendix} that
$$
\rho_{z_iz_j} := \frac{\E[z_iz_j]}{\sigma_{z_i}\sigma_{z_j}} = \prod_{(z_u,z_v)\in P(z_i,z_j)}
\rho_{z_uz_v}
$$
In other words, the correlation between a pair of nodes is the product of the correlations along the edges of the path that connects them.

Define such distribution by $\mu_{G,\v\rho, \v\sigma}$ where $G$ is ommitted if it is clear from context. A sample $\v z\sim \mu_{\v\rho,\v\sigma}$ can be drawn as follows:
first, $z_r \sim N(0,\sigma_r^2)$ is drawn for the root $r$ (the root can be assigned as any node of the tree). For any choice of the root, there is a unique way to direct the edges going away from the root. This defines parent-child relationships between the nodes. One assigns random values from the remaining of the nodes, from top to bottom. Assuming that the value $z_u$ on the parent $u$ of a node $v$ was already set, we draw $z_v$ as follows:
\[
z_v = \sigma_v \l(
\rho_{uv} \frac{z_u}{\sigma_u} + \sqrt{1 - \rho_{uv}^2} \epsilon_v,
\r)
\]
where $\epsilon_v \sim N(0,1)$ independently of the other variables and $\rho_{uv} = \rho_e$ for the edge $e$ that connects $u$ and $v$. 
Up to scaling the individual variables, one can assume that $\sigma_v = 1$ for all $v$.

An equivalent way to characterize the distribution is through the information matrix $J = \Sigma^{-1}$. Because of the factorization of the distribution, the only non-zero entries $J_{ij}$ will be when $(i,j)$ is an edge or if $i = j$. In general, the distribution could have an external field $h$. This will not happen in the joint distribution, since we assume the means to be $0$. 

In the following, we will also need the conditional distribution of the internal nodes $y$ given the leaves $x$. Since the model is Gaussian, we know that there exists a matrix $\lambda \in \mathbb{R}^{m\times n}$ such that 
\begin{equation}\label{eq:conditional_tree}
\E[y_i|x] = \sum_{j} \lambda_{y_ix_j} x_j
\end{equation}
Lastly, the external field of $y$ in the conditional distribution is given by the relation
$$
h_y = - J_{yx}x
$$
All these properties will prove useful in the sequel.

In an estimation setting, we observe independent samples from the latent distribution over the leaves of the tree. In particular, each sample contains the information of
$ (x_l)_{l = 1}^n$, obtained using one draw from the marginal (joint) distribution over the leaves. The goal is to learn the parameters $\rho_e$ for all the edges of the tree and the variances $\sigma_l^2$ for all leaves $l$.

\begin{remark}\label{r:var}
 The variances $\sigma_{y_v}^2$ on the internal nodes $y_v$ of the tree cannot be estimated. This is due to the fact that samples from two distributions $\mu_{\v\rho,\v\sigma}, \mu_{\v\rho,\v\sigma'}$ that have the same correlations $\v\rho$ and the same variances on the leaves, differ only by a scaling of the unobseved nodes. In particular, we can transfer a sample from one distribution to a sample from the other by just multiplying the hidden nodes by constants. This change does not affect the marginal distribution over the observed nodes.
 \end{remark}
 \begin{remark}\label{r:degree}
   One has to assume that each internal node has degree at least $3$. Indeed, if $v$ is an internal node with neighbors $u_1,u_2$, then, as long as $\rho_{u_1v}\rho_{vu_2}$ remains constant, the distribution over the leaves is not affected. In case that such a node $v$ exists, one can remove it while keeping the same distribution over the leaves. In particular, $u_1$ is connected with $v$ and one sets $\rho_{u_1u_2}$ as equal to the value $\rho_{u_1 v}\rho_{vu_2}$ in the old graph.
\end{remark}

\section{EM and the likelihood function for general trees}

We are interested in analyzing the landscape of the likelihood function over the space of unknown parameters $J$. In particular, we show that we can characterize the stationary points of the likelihood function for a general tree. To help us identify the stationary points, it is convenient to view them as fixpoints of the 
EM algorithm.

Let $\mathcal P_G$ be the set of all distributions of the form $\mu_{G,\rho,\sigma}$ for any possible values of $\rho$ and $\sigma$. In each iteration $t = 0,1,\dots$, the algorithm will hold some distribution $\mu^t$, such that $\mu^0$ is arbitrary, and $\mu^t$ is obtained  using $\mu^{t-1}$.
The goal is to find some unknown distribution $\mu^* \in \mathcal{P_G}$, given only samples from $\mu^*_{x_1\cdots x_n}$, namely, $y_1,\ldots,y_m$ is not observed. The algorithm can be described as follows: we set $\mu^0 \in \mathcal{P}$ arbitrarily. Then, at any $t > 0$, define
\begin{equation}\label{eq:em}
\mu^{t+1} = \max_{\mu \in \mathcal P_G} \E_{x_1\cdots x_n \sim \mu^*_{x_1\cdots x_n}} \E_{y \sim \mu^t_{y \mid x_1 \cdots x_n}}[\log \Pr_\mu(x_1,\dots,x_n, y)],
\end{equation}
where $\Pr_\mu$ denotes the density with respect to $\mu$. Denote by $\sigma_{\cdot}^t, \rho_i^t, J^t$ the parameters corresponding to $\mu^t$ and by $\lambda^t$ the coefficients from \eqref{eq:conditional_tree}. Similarly, $\sigma_{\cdot}^*,\rho_i^*$ and $\lambda^*$ correspond to $\mu^*$. We would like to understand what are the fixpoints of this iteration rule and how they relate to the stationary points of the likelihood function. For this purpose, we first describe more explicitly the update rule in each iteration.
The proof follows along the same lines as Lemma~\ref{lem:covariance-conserve}.
\begin{lemma}\label{l:update_rule}
Let $\mu^{t,*}$ denote the joint distribution over $x_1\cdots x_n,y_1,\ldots,y_m$ such that
\[
\mu^{t,*}(x_1,\cdots, x_n,y_1,\ldots,y_m) = \mu^*(x_1,\cdots, x_n)\mu^t_{y_1,\ldots,y_m\mid x_1,\cdots, x_n}(y).
\]
Then, for any $x_i,y_j,y_k$, we have that
\[
\E_{\mu^{t+1}}[x_i y_j] = \E_{\mu^{t,*}}[x_i y_j], \ 
\E_{\mu^{t+1}} [y_jy_k] = \E_{\mu^{t,*}}[y_jy_k] ,\
\Var_{\mu^{t,*}}[x_i] = \Var_{\mu^{t+1}}[x_i],\
\Var_{\mu^{t,*}}[y_j] = \Var_{\mu^{t+1}}[y_j]
\enspace.
\]
\end{lemma}

We notice that the variance of the leaves remains the same at each iteration. This means that the determining quantity for the distribution in each iteration are the correlations $\rho^t$. In particular, during the execution of the algorithm, the variances of the internal nodes might be different than $1$, however the correlations always dictate the next iteration. Hence, for any fixpoint $\f \mu$ of the procedure that is given by some parameters $\f\sigma,\f J$ can be converted into a fixpoint with the same likelihood value but with all internal nodes $y$ having variance $1$. Therefore, in the sequel when we analyze the fixpoint of this rule, we assume w.l.o.g. that
all variances are equal to $1$. By scaling the variances of the internal nodes, we can obtain all equivalent fixpoints. 

To see how the fixpoint of this rule relates to the stationary points of the log-likelihood, let's first choose a parametrization in terms of the inverse covariance matrix $J$. This exponential family parametrization will enable us to compute the stationary points easily by setting the derivatives to $0$.

We first define the function $L:\mathbb{R}^{(n+m)\times(n+m)}_+ \mapsto \mathbb{R}$ as 
$$
L(J) := \E_{x\sim \mu^*} \log \Pr_{\mu(J)} (x)
$$
We have the following Lemma, which connects the stationary points of $L$ with the fixpoints of EM. The proof is standard and is omitted.

\begin{lemma}
Let $\mu^* \in P_G$ be such that $\rho^*_{ij} \in (0,1)$ for all
$(i,j) \in E$. Then, for any $\f J \in \mathbb{R}^{(n+m)\times(n+m)}_+$ we have that 
$\nabla L(\f J) = 0$ if and only if $\f J$ is
a stationary point of the update rule \eqref{eq:em}.

\end{lemma}

Hence, the two notions of stationarity are equivalent and we can focus on understanding when the update rule \eqref{eq:em} has a fixpoint.

\section{Uniqueness of stationary points of EM for general trees}\label{app:general}
We would like to prove that the only stationary point of the log-likelihood if $\f\rho_{ij} \in (0,1)$ is when $\f\rho = \rho^*$. This is the content of the following Theorem.
\stationary*

\begin{proof}
Let $\f\mu$ denote the distribution induced by the fixpoint $\f\rho$. Let $\f J$ be the information matrix and $\f \Sigma$ the covariance matrix corresponding to $\f\rho$. Also,
let $\mu^{*,f}$ denote the distribution with density
$$
\mu^{*,f}(x,y) = \mu^*(x) \f\mu(y|x)
$$
Using Lemma~\ref{l:update_rule}, we get that in the fixpoint we should have for all $z_i,z_j \in V$ 
\begin{equation}
\Var_{\mu^{*,f}}(z_i) = \Var_{\f\mu}(z_i)  \quad 
, \quad \Cov_{\mu^{*,f}}(z_i,z_j) = \Cov_{\f\mu}(z_i,z_j)
\end{equation}
where $z_i,z_j$ are either leaf or non-leaf nodes that are connected in the topology of $G$.

We will show that the only possible solution to this system of equations is $\f\rho = \rho^*$. 
First, let us analyze these equations for two non-leaf nodes $y_1,y_2$ that are connected in $G$. Since all the variables are zero mean, we have by the definition of $\mu^{*,f}$ that
\begin{align*}
\Cov_{\mu^{*,f}}(y_1,y_2) &= 
\E_{\mu^{*,f}}[y_1y_2] = 
\E_{x \sim \mu^*}[\E_{\f\mu}[y_1y_2|x]] = \E_{x \sim \mu^*}[\Cov_{\f\mu}(y_1,y_2|x) + \E_{\f\mu}[y_1|x]\E_{\f\mu}[y_2|x]]
\end{align*}
We get a similar equation for $\f\mu$, namely that
$$
\Cov_{\f\mu}(y_1,y_2) = \E_{x \sim \f\mu}[\Cov_{\f\mu}(y_1,y_2|x) + \E_{\f\mu}[y_1|x]\E_{\f\mu}[y_2|x]]
$$
At this point, we notice that the first of the two terms contains the quantity $\Cov_{\f\mu}(y_1,y_2|x)$, which does not
depend on the value $x$ that we condition upon (this can be seen by just applying the conditioning formula for gaussians). 
Hence, this term will appear in both sides of the equation and can be cancelled. 
So, we only need to compute the second term, which will depend on the values $x$. 
A similar argument for the variance shows that
\begin{align}\label{eq:fixpoint}
\Var_{\mu^{*,f}}(y_1) = \E_{\mu^*}\lp[\Var_{\f\mu}(y_1|x) + \E_{\f\mu}[y_1|x]^2\rp]\nonumber \\
\Var_{\f\mu}(y_1) = \E_{\f\mu}\lp[\Var_{\f\mu}(y_1|x) + \E_{\f\mu}[y_1|x]^2\rp]
\end{align}
Again, the conditional variance does not depend on the value of $x$, hence it will be the same for the two distributions.

After this preliminary observations, let's see how the we can reduce this problem to the one latent case. 
Suppose $y_1$ is a latent node and denote its set of neighbors in $G$ by $N$. Notice that some of the neighbors will be leaves and some will be non-leaves, which prompts us to partition $N$ into corresponding subsets $N_x,N_y$. 
Let's denote $s = |N|, s_x = |N_x|, s_y = |N_y|$. 
There is a natural partition of the leaves $R_1,\ldots,R_s$, which is induced by removing $y_1$ from the graph and taking the leaves in each connected component. 
We will focus on the fixed point equations that we get for the covariances of $y_1$ with it's neighbors. 
To do that, we should calculate the conditional expectations of $y_1$ and it's neighbors, given $x$. We will need to be careful when calculating them, since we would like certain quantities to appear. Therefore, we will start by marginalizing out all the non-leaf nodes except $y_1$ and $N_y$. Let $Y^c$ denote these nodes. 
We first compute what is the external field of $y_1$ and the nodes in $N_y$ when we do this marginalization. Indeed, if $h_{y_1}$ was the original external field of $y_1$ and the new external field is $h'_{y_1}$, then
we have that
$$
h'_{y_1} = h_{y_1} - \f J_{y_1 Y^c} \f J_{Y^cY^c}^{-1} h_{Y^c}
$$
However, notice that the vector $\f J_{y_1Y^c}$ is the $0$ vector, since $Y^c$ does not contain any neighbor of $y_1$. Hence, the external field of $y_1$ does not change. 
Now let $y_2$ be a neighbor in $N_y$. Then,
$$
h'_{y_2} = h_{y_2} - \f J_{y_2 Y^c} \f J_{Y^cY^c}^{-1} h_{Y^c}
$$
Notice that $h_{y_2}$ is a linear combination of the leaves that are neighbors of $y_2$. Also, we can show that the second term of the right hand side is a linear combination of the leaves that belong to the connected component corresponding to $y_2$ and are not neighbors of $y_2$. This is established in the following Lemma.

\linearcomb*
\begin{proof}
The leaves in $S_{y_2}$ can be partitioned into subsets $A,B$, where $A$ are the leaves that are neighbors of $y_2$ and $B$ the remaining leaves. Let $T_2$ be the topology of the connected component that $y_2$ belongs to, when we remove $y_1$. Also, let $N_2$ be the neighborhood of $y_2$ in $T_2$, with corresponding subsets $N_{2x}, N_{2y}$. Clearly, $h_{y_2}$ is a linear combination of the leaves in $A$. We will show that the second term is a linear combination of the leaves in $B$, thus concluding the claim. 
First of all, notice that $\f J_{Y^cY^c}$ is the information matrix of a gaussian model, whose graphical representation is the tree $G$ when $y_1$ and all nodes in $N_y$ have been removed. Let's call $T_3$ this new topology. In this topology, $T_2$ has been partitioned into $|N_2|$ subtrees, one for each
neighbor of $y_2$ (because $y_2$ is removed). Hence, the leaves $B$ of $T_2$ have been partitioned into $|N_{2y}|$ subsets $Q_1,\ldots,Q_{|N_{2y}|}$.

Since $\f J_{Y^cY^c}$ is an information matrix, it's inverse
$\f J_{Y^cY^c}^{-1}$ is a covariance matrix, where the nodes have standard deviations $w_i$ and normalized covariances (correlations) $w_{ij}$. Since $T_3$ is a forest, there is at most one path connecting each one of the nodes. Hence, it is well known that the covariances multiply across paths in this structure, namely:
$$
(\f J_{Y^cY^c})^{-1}_{ij} = 
\left\{
\begin{array}{ll}
      0 &\text{ , if $i,j$ are not connected in $T_3$}\\
      w_iw_j \prod_{e \in P_{ij}}w_{ij} &\text{ , if $P_{ij}$ is the unique path connecting $i,j$}
\end{array} 
\right. 
$$
Given this description, it is easy to see that for each $i \in Y^c$, we have that
$
(\f J_{Y^cY^c}^{-1} h_{Y^c})_i
$
is a linear combination of leaves that belong to the same component as $i$ in $T_3$. Hence, for each $i \in N_{2y}$, $
(\f J_{Y^cY^c}^{-1} h_{Y^c})_i
$
is a linear combination of the leaves of the $Q_i$ that $i$ is connected to. Hence, $\f J_{y_2 Y^c} \f J_{Y^cY^c}^{-1} h_{Y^c}$ is simply a linear combination of the leaves in all the $Q_i$'s, which means it is a linear combination of the leaves in $B$.
\end{proof}

Hence, overall the external field $h'_{y_2}$ will be a linear combination of the leaves in $S_{y_2}$. The same is true for all nodes in $N_y$. It will be convenient to 
define the vector $H \in \mathbb{R}^s$, which has one entry for each node in $N$. If the node is a $y_i \in N_y$, then define $H_i = h'_{y_i}$. If the node is an $x_i \in N_x$, then define $H_i = x_i$. 
Let's focus on some $y_2 \in N_y$ and see what relations we get in the fixed point. As we say in the earlier computation, the relation for the covariance becomes
\begin{equation}\label{eq:cov_condition}
\E_{\f\mu}[\E_{\f\mu}[y_1|x]\E_{\f\mu}[y_2|x]] = 
\E_{\mu^*}[\E_{\f\mu}[y_1|x]\E_{\f\mu}[y_2|x]]
\end{equation}
The inner expectation is common on both sides, so let's start by calculating that. 
Since we have already marginalized out all the nodes in $Y^c$, we only need to marginalize out the nodes in $N_y$ other than $y_2$. For convenience, denote $N_2 = N_y \setminus \{y_2\}$. Then, if we marginalize out nodes in $N_2$, the resulting extenal field $h''_{y_1}$ will be
$$
h''_{y_1} = h_{y_1} -\f J_{y_1 N_2} (\f J'_{N_2N_2})^{-1} h'_{N_2}
$$
The reason why we write $J'$ is that the information matrix has been altered when marginalizing $Y^c$. 
Now, notice that since the neighbors of $y_1$ are not connected with each other, the matrix $\f J'_{N_2N_2}$ is diagonal. This means that we have
$$
h''_{y_1} = \sum_{x_i \in N_x} \f J_{y_1 x_i} x_i + \sum_{y_j \in N_y, y_j \neq y_2} \frac{\f J_{y_1 y_j}}{\f J'_{y_jy_j}} h'_{y_j}
$$
To write this more compactly, we introduce the vector $r \in \mathbb{R}^s$, where $r_i = \f J_{y_1 x_i}$ if $i \in N_x$ and $r_i = \f J_{y_1 y_i}/ \f J'_{y_iy_i}$ if $i \in N_y$. Notice that $r_i \neq 0$ always. 
With this notation, the previous equation becomes 
\begin{equation}\label{eq:external}
h''_{y_1} = \sum_{i \neq y_2} r_i H_i
\end{equation}
We also need to compute $h''_{y_2}$, which is a much easier task, since 
$$
h''_{y_2} = h'_{y_2} - \f J_{y_2  N_2} (\f J'_{N_2N_2})^{-1} h'_{N_2} = h'_{y_2}
$$
since $J_{y_2 N_2}$ is the zero vector (no connections between neighbors). 
Hence, we have calculated the external field of $y_1,y_2$ in the marginal model that contains only these two nodes.
Now, to calculate the conditional expectations of these two nodes, we just have to use the relation
$$
\begin{pmatrix}
\E_{\f\mu}[y_1|x]\\
\E_{\f\mu}[y_2|x]
\end{pmatrix} = 
\f\Sigma_{y_1y_2|x} 
\begin{pmatrix}
h''_{y_1}\\
h''_{y_2}
\end{pmatrix}
$$
that connects the external field to the mean of a gaussian. Here, $\f\Sigma_{y_1y_2|x}$ is the
$2\times 2$ covariance matrix of the conditional distribution of $y_1,y_2$ given $x$. 
The reason we have used $\f\Sigma_{y_1y_2|x}$ is that the covariance matrix does not change when we marginalize some nodes.
Suppose
$$
\f\Sigma_{y_1y_2|x} = \begin{pmatrix}
c_1 &c_2\\
c_3 &c_4
\end{pmatrix}
$$
with $c_1c_4 - c_2c_3 \neq 0$. The reason the variances are not necessarily $1$ is that we are now in the conditional model. Then, condition\eqref{eq:cov_condition} translates to the following:

\begin{align*}
&\E_{x \sim \mu^*} [(c_1h''_{y_1} + c_2 h''_{y_2})(c_3 h''_{y_1} +  c_4h_{y_2})] = \E_{x \sim \f\mu} [(c_1h''_{y_1} + c_2 h''_{y_2})(c_3 h''_{y_1} +  c_4h_{y_2})] 
\end{align*}
This implies that
\begin{align}\label{eq:cov1}
c_1c_3 (\E_{x\sim \mu^*}[(h''_{y_1})^2]  - \E_{x \sim \f\mu}[(h''_{y_1})^2]) + c_2c_4 (\E_{x\sim \mu^*}[(h''_{y_2})^2]
- \E_{x \sim \f\mu}[(h''_{y_2})^2]) \nonumber \\
+ (c_1c_4 +c_2c_3)(\E_{x\sim \mu^*}[h''_{y_1}h''_{y_2}]  - \E_{x \sim \f\mu}[h''_{y_1}h''_{y_2}]) = 0
\end{align}
Similarly, from the variance condition on $y_1$ (Equations ~\ref{eq:fixpoint}) we obtain

\begin{align*}
\E_{x \sim \mu^*} [(c_1h''_{y_1} + c_2 h''_{y_2})^2] = \E_{x \sim \f\mu} [(c_1h''_{y_1} + c_2 h''_{y_2})^2]
\end{align*}
This implies
\begin{align}\label{eq:var1}
 c_1^2(\E_{x\sim \mu^*}[(h''_{y_1})^2]  - \E_{x \sim \f\mu}[(h''_{y_1})^2]) + c^2 (\E_{x\sim \mu^*}[(h''_{y_2})^2] \nonumber \\
 - \E_{x \sim \f\mu}[(h''_{y_2})^2]) + 2c_1c_2(\E_{x\sim \mu^*}[h''_{y_1}h''_{y_2}]  - \E_{x \sim \f\mu}[h''_{y_1}h''_{y_2}]) = 0
\end{align}
Similarly, for $y_2$ we get
\begin{align}\label{eq:var2}
c_3^2 \lp(\E_{x\sim \mu^*}[h_1^2]  - \E_{x \sim \f\mu}[h_1^2]\rp) +  c_4^2\lp(\E_{x\sim \mu^*}[h_2^2]
- \E_{x \sim \f\mu}[h_2^2]\rp)\nonumber \\
+2c_3c_4 \lp(\E_{x\sim \mu^*}[h_1h_2]  - \E_{x \sim \f\mu}[h_1h_2]\rp) = 0 
\end{align}

We can think of equations \eqref{eq:cov1}, \eqref{eq:var1} and \eqref{eq:var2} as a $3\times 3$ system with matrix
$$
\begin{pmatrix}
c_1c_3 &c_2c_4 &c_1c_4 + c_2 c_3\\
c_1^2 &c^2 &2c_1c_2\\
c_3^2 &c_4^2 &2c_3c_4
\end{pmatrix}
$$
where the three unknown variables are
$$
\E_{x\sim \mu^*}[h_1^2]  - \E_{x \sim \f\mu}[h_1^2],\ \E_{x\sim \mu^*}[h_2^2]
- \E_{x \sim \f\mu}[h_2^2], \ \E_{x\sim \mu^*}[h_1h_2]  - \E_{x \sim \f\mu}[h_1h_2]
$$

The determinant of this matrix is $-(c_1c_4 - c_2c_3)^3$, which is non-zero, since the matrix $\Sigma_{y_1y_2|x}$ is invertible. The reason is that we have assumed that in the fixpoint $\f\rho$ all correlations are strictly less than $1$. Hence, we conclude that
$$
\E_{x\sim \mu^*}[h''_{y_1}h''_{y_2}]  = \E_{x \sim \f\mu}[h''_{y_1}h''_{y_2}]
$$
Based on the calculations that were done earlier (equation \eqref{eq:external}), this equation can be written as
$$
\E_{x \sim \mu^*}\lp[H_{y_2} \lp(\sum_{i \neq y_2}r_i H_i\rp)\rp] = 
\E_{x \sim \f\mu}\lp[H_{y_2} \lp(\sum_{i \neq y_2}r_i H_i\rp)\rp]
$$
If $x_j$ is a non-leaf neighbor, the fixed point equations give us
\begin{align*}
\E_{x \sim \mu^*} [x_j (\sum_{i\neq j}r_i H_i) ] = 
\E_{x \sim \f\mu} [x_j (\sum_{i\neq j}r_i H_i) ] \implies \\
\E_{x \sim \mu^*} [H_j (\sum_{i\neq j}r_i H_i) ] = 
\E_{x \sim \f\mu} [H_j (\sum_{i\neq j}r_i H_i) ]
\end{align*}

The key observation here is that the coefficients $r_i$ that appear will be the same in all equations involving $y_1$. The last step of the argument involves actually computing the expectation and seeing what it implies for $\mu^{*,f}$. 
First, let's try to compute $\E[H_iH_j]$ for $y_i,y_j \in N_y$. Remember that we have established already that for each $i \in N$, $H_i$ is a linear combination of the leaves in the partition corresponding to $i$. Hence,  computing $\E[H_iH_j]$ amounts to computing the covariance between all pairs of leaves from the two different subsets. Suppose $H_i = (a^i)^\top X_i$, where $X_i$ is the vector of leaves in $S_i$. As we said, correlations multiply across paths, so in particular we have that
\begin{align*}
\E_{x\sim \mu^*} [H_iH_j] &= 
\E_{x\sim \mu^*} [(a^i)^\top X_i (a^j)^\top X_j] = 
\sum_{x_k \in S_i , x_l \in S_j}a^i_ka^j_l \E_{x\sim \mu^*} [x^i_kx^j_l] \\
&= \sum_{x_k \in S_i , x_l \in S_j}a^i_ka^j_l \prod_{e \in P_{x_k,x_l}}\theta^*_e
\end{align*}
where $P(x_k,x_l)$ denotes the path between leaves $x_k,x_l$.
Now, notice that all the paths from $X_i$ to $X_j$ will have to go through the edges connecting $y_1$ to $i$ and $y_1$ to $j$. Let $\rho^*_i, \rho^*_j$ be the correlations in these two edges. Then, there is a convenient factorization that can be written as follows
\begin{align*}
\E_{x\sim \mu^*} [H_iH_j] =
\rho^*_i \rho^*_j
\lp(\sum_{x_k \in S_i} a^i_k \prod_{e \in P_{x_k,y_i}} \rho^*_e\rp)\lp(\sum_{x_l\in S_j} a^j_l \prod_{e \in P_{x_l,y_j}} \rho^*_e\rp)
\end{align*}
The exact same relations hold for $\f\rho$, since the topology is the same. Hence, by writing out the condition 
\begin{align}\label{eq:paths}
\rho^*_i
\lp(\sum_{x_r\in S_i} a^i_r \prod_{e \in P_{x_ry_i}} \rho^*_e\rp)\sum_{z_j \in N,j \neq i}\rho^*_j\lp(\sum_{x_s\in S_j} a^j_s \prod_{e \in P_{x_sz_j}} \rho^*_e\rp) \nonumber \\
= \f\rho_i
\lp(\sum_{x_r\in S_i} a^i_r \prod_{e \in P_{x_ry_i}} \f\rho_e\rp)\sum_{z_j \in N,j \neq i}\f\rho_j\lp(\sum_{x_s\in S_j} a^j_s \prod_{e \in P_{x_sz_j}} \f\rho_e\rp)
\end{align}
The reason we used the notation $z_j$ is that this
neighbor could be either a leaf or an internal node.
Hence, for each $i \in N_y$ we set 
$$
w_i^* = \rho^*_i
\lp(\sum_{r\in X^i} a^i_r \prod_{e \in P_{ri}} \rho^*_e\rp)
$$
and similarly 
$$
\f w_i = \f\rho_i
\lp(\sum_{r\in X^i} a^i_r \prod_{e \in P_{ri}} \f\rho_e\rp).
$$
Notice that if $i$ or $j$ is a leaf neighbor, then we get the same expression, except that the parenthesis will be $1$ and we will simply have $\rho^*_i$ or $\rho^*_j$ as the variable. Hence, this definition can be extended to all $i \in N$. Now, given these parametrizations, the fixed point conditions are
\begin{equation}\label{eq:final}
r_j w_j^* \lp(\sum_{i\neq j} r_i w_i^*\rp) = 
r_j \f w_j \lp(\sum_{i\neq j} r_i \f w_i\rp)
\end{equation}
for each $i \in N$. 
But this is exactly the system that we got for one latent.. As we proved in Lemma~\ref{lem:equation-system}, the only solution for this system is $r_i \f w_i = r_i w_i^*$ for each $i$, which implies that $\f w_i = w_i^*$ for each $i$. Notice that if $N_x$ is nonempty, then this implies that $\rho^*_e = \f\rho_e$ for all edges $e$ between $y_1$ and it's neighboring leaves. 

We can use this result to show that $\f\rho = \rho^*$. Our argument is inductive. At each step, we pick an internal node $y$ that only has one non-leaf neighbor (it's parent). Then, for all edges of the form $e = (y,x_i)$ for some $x_i$, we have that $\rho_e^* = \f \rho_e$ by the previous argument.
Once we establish that, we remove all the leaf nodes that are connected to $y$ from the tree along with their corresponding edges. This means that in the next iteration, $y$ will be a leaf, so it will no longer be selected. This means that this process terminates after $m$ steps. 

The correctness of this procedure can be proven inductively as follows:
for the base case, we already saw that edges that are connected to leaves will be equal in the two models.
In each step, for the node $y$ that is selected at that step, we know from the induction hypothesis that all of it's leaves in the remaining tree are either true leaves, or internal nodes who have already been selected. 
This means that are descendants of $y$ have already been proven to be equal in the two models.
Then, for each leaf neighbor $z$ of the node $y$ (it might not be a true leaf in the original tree) we have parameters $w^*,\f w$.
These are proven to be equal by the previous arguments. Let $e$ be the edge connecting $z,y$. Then $w^*$ is a multiple of $\rho_e^*$ and $\f w $ is a multiple of $\f\rho_e$. The multiplier for both of these is the same in both quantities, since it only depends on descendant edges, which are proven to be equal for the two models. It follows that $\rho^*_e = \f\rho_e$ and the induction stpe is complete.

\end{proof}

% \else
% \bibliography{ref}
% \appendix
% \input{app-one-latent}
% \input{finite/bounded-iterate}
% \input{finite/finite-main}
% \input{general_tree_model}
% \fi

\end{document}